%% file: main.tex
\documentclass{article}

\PassOptionsToPackage{numbers, compress}{natbib}

\usepackage[preprint]{style}

\usepackage[utf8]{inputenc}
\usepackage[T1]{fontenc}
\usepackage{hyperref}
\usepackage{url}
\usepackage{booktabs}
\usepackage{amsmath,amssymb,amsfonts,amsthm,mathtools}
\usepackage{nicefrac}
\usepackage{microtype}
\usepackage{xcolor}
\usepackage[pdftex]{graphicx}
\usepackage{chato-notes}
\usepackage{subcaption}
\usepackage{pgfplots}
\usepackage{pgfplotstable}
\usepackage{multirow}
\usepackage{makecell}
\usetikzlibrary{calc}
\pgfplotsset{compat=1.18}

\newtheorem{theorem}{Theorem}
\newtheorem{lemma}[theorem]{Lemma}
\newtheorem{proposition}{Proposition}
\newtheorem{corollary}[theorem]{Corollary}
\newtheorem{assumption}{Assumption}
\newtheorem*{remark}{Remark}

\newcommand{\E}{\mathbb{E}}
\newcommand{\R}{\mathbb{R}}
\newcommand{\Var}{\operatorname{Var}}
\newcommand{\Cov}{\operatorname{Cov}}
\newcommand{\LOCO}{\operatorname{LOCO}}
\newcommand{\STII}{\operatorname{STII}}
\newcommand{\FSII}{\operatorname{FSII}}
\newcommand{\SII}{\operatorname{SII}}
\newcommand{\bg}{\mathrm{bg}}
\newcommand{\XOR}{\oplus}
\newcommand{\ours}{Stochastic Hi-Fi}
\newcommand{\URSp}{U/R/S}

\title{
The Representational Limit of Scalar Interactions:\\An Interventional Decomposition
}

\author{%
  Potito Aghilar \\
  Polytechnic University of Bari \\
  Bari, IT \\
  \texttt{potito.aghilar@poliba.it} \\
  \And
  Sabino Roccotelli \\
  Polytechnic University of Bari \\
  Bari, IT \\
  \texttt{s.roccotelli@phd.poliba.it} \\
  \And
  Stanislao Fidanza \\
  Polytechnic University of Bari \\
  Bari, IT \\
  \texttt{s.fidanza@studenti.poliba.it} \\
  \And
  Vito Walter Anelli \\
  Polytechnic University of Bari \\
  Bari, IT \\
  \texttt{vitowalter.anelli@poliba.it} \\
  \And
  Sebastiano Stramaglia \\
  University of Bari Aldo Moro \\
  Bari, IT \\
  \texttt{sebastiano.stramaglia@ba.infn.it} \\
  \And
  Tommaso Di Noia \\
  Polytechnic University of Bari \\
  Bari, IT \\
  \texttt{tommaso.dinoia@poliba.it} \\
}

\begin{document}

\maketitle

\begin{abstract}
\input{sections/00_abstract}
\end{abstract}

\input{sections/01_introduction}

\input{sections/02_related_work}

\input{sections/03_method}
\input{sections/04_theory}

\input{sections/05_experiments}
\input{sections/06_discussion}

\newpage

\bibliographystyle{plainnat}
\bibliography{main}

\newpage
\appendix
\input{sections/appendix/A_notation}

\newpage
\input{sections/appendix/B_proof_conflation}
\newpage
\input{sections/appendix/C_proof_interventional}
\newpage
\input{sections/appendix/D_proof_variance}
\newpage
\input{sections/appendix/E_proof_vocab}
\newpage
\input{sections/appendix/F_numerical_verification}

\newpage
\input{sections/appendix/G_extended_method}
\newpage
\input{sections/appendix/H_extended_experiments}

\newpage
\input{sections/appendix/I_compute_reproducibility}

\end{document}

%% file: sections/00_abstract.tex
Signed pairwise interaction scores fundamentally conflate
uniqueness (U), redundancy (R), and synergy (S). We prove this on a minimal 3-way XOR structural causal model: faithful indices such as Shapley-Taylor return zero per pair, whereas projective indices such as Shapley Interaction spread the third-order effect into pair scalars that conflate the three mechanisms. We introduce Stochastic Hi-Fi, a post-hoc, retraining-free predictability decomposition that estimates per-feature U/R/S profiles by interventional masked inference. The estimator provides exact interventional semantics, finite-sample Monte Carlo bounds, strict variance reduction from coupled diamond sampling, and uniform finite-vocabulary convergence. 
Across tabular SCMs, Stochastic Hi-Fi recovers structure missed by scalar baselines (up to $411\times$ larger interaction-magnitude recovery ratios). It also separates redundant and synergistic heads in the GPT-2 IOI circuit. 
On NIH ChestX-ray14, Stochastic Hi-Fi matches GradCAM on Pointing Game and improves substantially on Deletion AUC. Code and data are available at \href{https://anonymous.4open.science/r/stochastic-hi-fi/README.md}{Stochastic Hi-Fi}.

%% file: sections/01_introduction.tex
\section{Introduction}
\label{sec:introduction}

% Modern deep networks route information through overlapping, high-dimensional pathways whose feature interactions are neither purely additive nor easily separable. When a model's prediction depends jointly on a set of inputs, a fundamental question arises: is that dependence \emph{unique} to specific features, \emph{redundant} across multiple inputs, or truly \emph{synergistic}, emerging only from their conjunction? Answering this question precisely is not merely an interpretability nicety. It determines whether post-hoc attribution scores faithfully reflect the model's internal causal structure, or whether they silently conflate distinct cooperative mechanisms that can only be disambiguated under intervention.
Modern deep networks route information through complex, non-additive pathways. This raises a fundamental question: is joint feature dependence \emph{unique}, \emph{redundant}, or truly \emph{synergistic}? Without interventional disambiguation, post-hoc attributions silently conflate these distinct cooperative mechanisms rather than faithfully reflecting the model's causal structure.

Currently, the Explainable AI (XAI) community explains feature cooperation through signed pairwise interaction scores
% : every major library reports one scalar per pair
\citep{DBLP:conf/nips/MuschalikBFKHH24,DBLP:journals/jmlr/HedstromWKBMSLH23}, read as synergy if positive, redundancy if negative. We show this paradigm is fundamentally flawed, not merely noisy, but operating in the \emph{wrong codomain}. On the minimal 3-way XOR structural causal model ($\XOR\text{SCM}_3$, three binary parents whose label is their parity), we prove that no scalar index can jointly encode uniqueness (U), redundancy (R), \emph{and} synergy (S) (Theorem~\ref{thm:conflation}). The failure is exact and exhaustive: faithful indices (e.g., Shapley-Taylor~\citep{DBLP:conf/icml/SundararajanDA20}, Faith-SHAP~\citep{DBLP:journals/jmlr/TsaiYR23}) return zero for every higher-order pair, while projective indices (e.g., SII~\citep{DBLP:journals/ijgt/GrabischR99}, $k$-SII, kADD-SHAP~\citep{DBLP:conf/nips/MuschalikBFKHH24}) redistribute signal into spurious pair coefficients of magnitude $\nicefrac{1}{4}$. A scalar cannot encode the \URSp{} triple; the fix requires changing the output space, not refining the estimator.

Furthermore, existing variance-based decompositions (functional ANOVA~\citep{Hoeffding1948,DBLP:conf/kdd/Hooker04a}) and information-theoretic Partial Information Decompositions (PID~\citep{DBLP:journals/corr/abs-1004-2515}) operate on fundamentally different estimands. They either rely on observational distributions, which we prove cannot disambiguate synergy from redundancy without intervention, or require full model retraining, answering what a new model \emph{could} learn rather than what a deployed model \emph{actually} uses. 

To resolve this, we replace the scalar question with a decomposed one: for each feature, how much predictive value is unique, redundant, or synergistic? We adapt the \URSp{} predictability decomposition recently introduced in physics~\citep{PhysRevE.111.L033301}, but fundamentally shift its operational objective to the post-hoc XAI setting. 
We replace model retraining with interventional masked inference against a background distribution $p_{\bg}$ (the marginal of inactive features over a held-out reference set), auditing \emph{what a deployed model actually uses} rather than \emph{what a model could learn}.
Recognizing that the choice of $p_{\bg}$ defines the estimand, we explicitly scope its influence and provide comprehensive sensitivity ablations in Appendix~\ref{sec:appendix_experiments}.

For tractable interventional decomposition, we introduce \ours{} (Section~\ref{sec:method}). The method builds a coalition-loss vocabulary by repeatedly masking features, sampling background values, and updating per-coalition means and variances online. Because this vocabulary is shared across features, the primary computational object is no longer a suite of retrained models, but a single table of interventional losses. To reduce estimation variance under an explicit covariance condition (Theorem~\ref{thm:variance}), we introduce coupled ``diamond'' queries that evaluate adjacent coalitions $(C, C\cup\{i\}, C\cup\{j\}, C\cup\{i,j\})$ on the same background batch. 

% Our experiments are designed to probe theoretical failure modes rather than to chase a single benchmark. On tabular SCMs, \ours{} retains up to $411\times$ more disambiguated signal than signed pair indices. In the language domain, we analyze the Indirect Object Identification (IOI) circuit of \texttt{gpt2-small}~\citep{DBLP:conf/iclr/WangVCSS23}, successfully separating redundant backup heads from synergistic compositional head pairs. 
% Finally, on NIH ChestX-ray14~\citep{DBLP:conf/cvpr/WangPLLBS17}, coarse-grid \URSp{} matches GradCAM~\citep{DBLP:conf/iccv/SelvarajuCDVPB17} on Pointing Game (non-significant difference), but incurs a small statistical significance IoU@15 deficit, while achieving substantially stronger deletion-based faithfulness (lower deletion AUC: $0.332 \to 0.269$, $p=1.61\!\times\!10^{-78}$; lower is better).
% Rather than obscuring this trade-off, we preserve it explicitly to mark a realistic scope boundary for coarse-grid interventional coalitions. 
Our evaluation targets theorem-predicted failure modes. On tabular SCMs, \ours{} recovers up to $411\times$ more mechanism-resolved signal than signed pair indices. On GPT-2 IOI, it separates redundant backup heads from synergistic head pairs. On ChestX-ray14, coarse-grid \URSp{} matches GradCAM on Pointing Game, improves deletion AUC ($0.332\to0.269$, $p=1.61\times10^{-78}$), and shows a small IoU@15 deficit.
Our contributions are threefold:
\begin{enumerate}
  \item \textbf{Conflation theorem.} We prove that scalar interaction codomains cannot jointly represent uniqueness, redundancy, and synergy on minimal higher-order SCMs.
  % \item \textbf{Retraining-free estimator.} We introduce \ours{}, with interventional semantics, finite-sample concentration, variance reduction via diamond coupling, and finite-vocabulary convergence.
  \item \textbf{Retraining-free estimator.} We introduce \ours{} and establish interventional semantics, finite-sample concentration, variance reduction, and finite-vocabulary convergence.
  % \item \textbf{Scope-honest validation.} We validate mechanism disentanglement on SCMs, transformer circuits, and ChestX-ray14, including explicit mixed outcomes for localization versus deletion faithfulness.
  \item \textbf{Scope-honest validation.} We validate mechanism disentanglement on SCMs, transformer circuits, ChestX-ray14, explicit mixed outcomes for localization versus deletion faithfulness.
\end{enumerate}

% This study makes the following primary contributions:
% \begin{enumerate}
%   \item \textbf{Theoretical Impossibility (Conflation Theorem).} We prove that existing signed pairwise interaction indices and observational loss functions fundamentally conflate uniqueness, redundancy, and synergy on minimal higher-order SCMs. We establish that scalar codomains are mathematically insufficient for capturing pure higher-order cooperation.
%   \item \textbf{Estimator and guarantees.} We propose a retraining-free \URSp{} estimator for deployed models with Welford-online vocabulary and coupled ``diamond'' sampling, and establish interventional semantics, finite-sample Monte Carlo concentration, strict variance reduction, and finite-vocabulary convergence.
%   \item \textbf{Scope-honest validation.} Across tabular SCMs, \ours{} recovers ground-truth cooperative structure while showing that scalar SI baselines can conflate interaction roles by up to $411\times$ on XOR3; in transformer circuits, it separates redundant and synergistic heads; and on the full annotated ChestX-ray14 evaluation, it preserves Pointing Game parity with GradCAM while trading a small IoU@15 deficit for substantially lower deletion AUC.
% \end{enumerate}

%% file: sections/02_related_work.tex
\section{Related Work}
\label{sec:related_work}
We organize prior work along two axes: output codomain (scalar vs. \URSp{}-disentangled) and estimand (observational/retraining vs. post-hoc interventional loss); see Table~\ref{tab:rw_taxonomy}.
% We organize interaction explanations along two axes: codomain expressivity (scalar projections vs.\ mechanism-disentangled \URSp{}) and estimand nature (observational variance or model retraining vs.\ post-hoc interventional loss); Table~\ref{tab:rw_taxonomy} summarizes this taxonomy.

The dominant paradigm assigns one scalar per feature pair: SHAP~\citep{DBLP:conf/nips/LundbergL17}, LIME~\citep{DBLP:conf/kdd/Ribeiro0G16}, and Integrated Gradients~\citep{DBLP:conf/icml/SundararajanTY17} operate at first order, while scalar interaction families, SII~\citep{DBLP:journals/ijgt/GrabischR99}, Shapley-Taylor~\citep{DBLP:conf/icml/SundararajanDA20}, Faith-SHAP~\citep{DBLP:journals/jmlr/TsaiYR23}, KernelSHAP-IQ~\citep{DBLP:conf/icml/FumagalliMKHH24}, SHAP-IQ/k-SII/kADD-SHAP~\citep{DBLP:conf/nips/MuschalikBFKHH24}, tensor-network accelerations~\citep{marzouk2026shap}, extend this to higher orders. The bottleneck is a primarily structural, rather than purely computational, limitation of scalar pair indices, which we make precise in Theorem~\ref{thm:conflation}. Richer functional forms such as Archipelago~\citep{DBLP:conf/nips/TsangR020}, Higher-Order IG~\citep{DBLP:journals/jmlr/JanizekSL21}, T-NID~\citep{DBLP:conf/iccv/LermanVKX21}, and H-Sets~\citep{hsets2024} preserve a \emph{non-disentangled} output codomain (scalar pair scores or higher-order interaction tensors), neither of which factorizes into per-feature \URSp{} channels. Graph-restricted variants~\citep{DBLP:journals/mor/Myerson77,DBLP:conf/icml/BuiNNY24} impose structural priors but do not resolve the representational gap. See Table~\ref{tab:rw_taxonomy} for a comparative summary.

Alternative traditions address parts of this gap. Variance-based decompositions (Hoeffding--Sobol~\citep{Hoeffding1948,SOBOL2001271}, STRIDE~\citep{DBLP:journals/corr/abs-2509-09070}) target observational variance over interventional loss. Partial Information Decomposition (PID~\citep{DBLP:journals/corr/abs-1004-2515}) provides the correct codomain (uniqueness, redundancy, and synergy)~\citep{DBLP:journals/entropy/BertschingerROJA14,DBLP:journals/entropy/Kolchinsky22}, but its predictability operationalizations (Hi-Fi~\citep{PhysRevE.111.L033301}, DiffusionPID~\citep{10.5555/3737916.3737982}) require model retraining. To audit deployed models, \ours{} causalizes importance-via-deletion (LOCO~\citep{Lei03072018}) through masked inference. It fuses PID's \URSp{} codomain with interventional post-hoc attribution~\citep{DBLP:conf/aistats/JanzingMB20,DBLP:conf/nips/HeskesSBC20}, separating it from observational perturbation/occlusion methods~\citep{DBLP:conf/bmvc/PetsiukDS18,DBLP:conf/iccv/FongV17,DBLP:conf/eccv/ZeilerF14}. It directly extends activation-patching and causal-scrubbing workflows~\citep{DBLP:conf/nips/VigGBQNSS20,DBLP:conf/nips/MengBAB22,DBLP:conf/iclr/WangVCSS23,DBLP:conf/nips/ConmyMLHG23} to higher-order interactions. Signed uniqueness is admissible under interventional contrasts, matching PID relaxations~\citep{DBLP:conf/nips/PakmanNGMMWS21,DBLP:journals/entropy/JamesEC19} and recovering classical PID under log-loss. Finally, the background distribution $p_{\bg}$ defines the estimand itself~\citep{DBLP:conf/icml/SundararajanN20}. We lift the interventional/causal-Shapley readings of \citet{DBLP:conf/aistats/JanzingMB20} and \citet{DBLP:conf/nips/HeskesSBC20} to the \URSp{} space, with $p_{\bg}$ sensitivity ablations in Appendix~\ref{sec:appendix_experiments}.

\input{tables/taxonomy}

%% file: tables/taxonomy.tex
\begin{table}[t]
\centering
\scriptsize
\caption{Taxonomy of interaction explanation frameworks. \ours{} uniquely provides a mechanism-disentangled \URSp{} profile for deployed models using strict interventional semantics.}
\label{tab:rw_taxonomy}
\setlength{\tabcolsep}{4pt}
\begin{tabular}{@{}lllll@{}}

\toprule
\textbf{Framework / Methods} & \textbf{Target Estimand} & \textbf{Output Codomain} & \textbf{Interventional?} & \textbf{Retraining-Free?} \\
\midrule
Shapley Indices (SII, SHAP-IQ, TN-SHAP) & Game-theoretic payoffs & Scalar signed score & No (Observational) & Yes \\
Functional ANOVA (Sobol, STRIDE)        & Observational variance & Additive terms      & No (Observational) & Yes \\
Masking / Perturbation (RISE)           & Deletion AUC           & Scalar score        & No (Perturbation)  & Yes \\
Predictability Decomp. (Hi-Fi)          & Retrained predictivity & \URSp{} profile     & Yes (Counterfactual)& No (Retrains) \\
PID (Williams--Beer)
& Information decomposition
& \URSp{} profile
& No (Observational)
& No \\
% \multirow{2}{*}
{\makecell[l]{Higher-Order Interactions\\(Archipelago, HO-IG, T-NID, H-Sets)}}
  & Interaction attribution
  & Scalar / Tensor
  & No (Observational)
  & Yes \\

\midrule
\textbf{\ours{} (Ours)}                 & \textbf{Interventional LOCO loss} & \textbf{Per-feature \URSp{}} & \textbf{Yes (Strict causal)} & \textbf{Yes} \\
\bottomrule
\end{tabular}
\end{table}

%% file: sections/03_method.tex
\section{\ours{}}
\label{sec:method}
To operationalize the \URSp{} framework, we introduce \ours{}. The method has three components: an interventional loss estimand (Section~\ref{sec:method_loss}), a LOCO-based \URSp{} decomposition (Section~\ref{sec:method_loco}), and a stochastic coalition estimator with variance reduction (Sections~\ref{sec:method_vocab} and \ref{sec:method_diamond}).
% To operationalize the \URSp{} framework, we introduce \ours{}. The method is structured around three components: an interventional loss estimand that bypasses model retraining (Section~\ref{sec:method_loss}), a structural decomposition of predictive gains (Section~\ref{sec:method_loco}), and a scalable stochastic estimator that dynamically explores the coalition lattice using variance-reduction techniques (Sections~\ref{sec:method_vocab} and \ref{sec:method_diamond}).

\subsection{Interventional coalition loss}
\label{sec:method_loss}
Let $f:\mathcal{X}\to\mathcal{Y}$ be a predictor, let $x=(x_1,\ldots,x_n) \in \mathcal{X}$ be an input, and let $y \in \mathcal{Y}$ be its target. Let $S\subseteq[n]$ denote a feature coalition, represented by an indicator mask $m_S\in\{0,1\}^n$ where $(m_S)i = 1$ if $i \in S$ and $0$ otherwise. Given a background distribution $Z\sim p{\bg}$ (ablations in Appendix~\ref{sec:appendix_pbg_ablation}) and a loss function $\ell: \mathcal{Y} \times \mathcal{Y} \to \mathbb{R}$, we define the interventional coalition loss as:
% Let $f:\mathcal{X}\to\mathcal{Y}$ be a trained predictor, let $x=(x_1,\ldots,x_n) \in \mathcal{X}$ be an input instance, and let $y \in \mathcal{Y}$ be its target. Let $S\subseteq[n]$ denote a feature coalition, represented by a binary indicator mask $m_S\in\{0,1\}^n$ where $(m_S)_i = 1$ if $i \in S$ and $0$ otherwise. Given a background distribution $Z\sim p_{\bg}$ (ablations in Appendix~\ref{sec:appendix_pbg_ablation}) and a task-specific loss function $\ell: \mathcal{Y} \times \mathcal{Y} \to \mathbb{R}$, we define the interventional coalition loss as:
\begin{equation}
\label{eq:loss}
L(S;x,y)=
\E_{Z\sim p_{\bg}}\left[
\ell\left(f(m_S\odot x+(1-m_S)\odot Z),y\right)\right],
\end{equation}
where $\odot$ denotes the element-wise Hadamard product. Crucially, this is an interventional estimand designed to audit a deployed predictor. By replacing absent features with background samples rather than retraining the model on the subset $S$, we isolate the model's reliance on the features \citep{DBLP:conf/aistats/JanzingMB20,DBLP:conf/nips/HeskesSBC20}.

\subsection{LOCO \URSp{} decomposition}
\label{sec:method_loco}
For a target feature $i \in [n]$ and a context coalition $C\subseteq[n]\setminus\{i\}$, we define the leave-one-covariate-out (LOCO) gain as the reduction in loss achieved by adding feature $i$ to the context $C$:
\begin{equation}
\label{eq:loco}
\LOCO(i\mid C)=L(C;x,y)-L(C\cup\{i\} ; x,y).
\end{equation}
By evaluating this gain across different contexts, we determine the bounds of the feature's predictive role. We define the minimum gain, maximum gain, and standalone (empty-context) gain as:
\begin{equation}
U_i=\min_C \LOCO(i\mid C),\qquad
L^{\max}_i=\max_C \LOCO(i\mid C),\qquad
\pi_i=\LOCO(i\mid\emptyset).
\end{equation}
Here, $U_i$ represents the minimum marginal utility of feature $i$, the predictive value it provides in the least favorable context. The gain $\pi_i$ captures the feature's first-order effect. We then decompose the maximum predictive capacity $L^{\max}_i$ by defining redundancy and synergy as:
% Here, $U_i$ represents the context-robust minimum marginal utility of feature $i$, the predictive value it provides even in the least favorable context. The standalone gain $\pi_i$ captures the feature's first-order effect. We then algebraically decompose the maximum predictive capacity $L^{\max}_i$ by defining redundancy and synergy as:
\begin{equation}
R_i=\pi_i-U_i,\qquad S_i=L^{\max}_i-\pi_i,
\end{equation}
yielding the exact structural decomposition $L^{\max}_i=U_i+R_i+S_i$.

\paragraph{Interpreting the \URSp{} profile.}
A negative uniqueness ($U_i<0$) does not indicate estimator failure. Rather, it reveals that feature $i$ can degrade prediction in its least favorable context, while still contributing through redundancy or synergy in other contexts (see IOI head $9.6$ in E3 and Appendix~\ref{sec:appendix_pbg_ablation} for sensitivity). We preserve this sign rather than truncating it to zero, as it carries information about feature substitution and backup behavior (see Section~\ref{sec:discussion}). Under partial vocabulary coverage, estimated $U_i$ is an upper bound on uniqueness; this bias vanishes under exhaustive enumeration (E3).
% A negative uniqueness value ($U_i<0$) does not indicate an estimator failure. Rather, it reveals that feature $i$ can actively degrade prediction in its least favorable context, while still contributing positively through redundancy or synergy in other contexts (see IOI head $9.6$ in E3 and Appendix~\ref{sec:appendix_method} for sensitivity). We preserve this sign rather than truncating it to zero, as it carries mechanistic information about feature substitution and backup behavior (see Section~\ref{sec:discussion}). Under partial vocabulary coverage, the estimated $U_i$ is an upper bound on true uniqueness; this bias vanishes under exhaustive enumeration (E3).

% When the coalition vocabulary is a strict subset of all $2^{n-1}$ possible contexts, the computed $U_i$ is an upper bound on the true uniqueness (see paragraph~\ref{par:partial_coverage} of Section~\ref{sec:theory}). In E3 ($n=10$) exhaustive enumeration eliminates this bias entirely.

\subsection{Coalition vocabulary and uncertainty}
\label{sec:method_vocab}
Evaluating the \URSp{} profile requires $2^n$ coalition losses, which is intractable for large $n$. Instead, \ours{} maintains a coalition vocabulary $\mathcal{V}:S\mapsto(\widehat{L}(S),\widehat{\sigma}^2(S),K(S))$. For a visited coalition $S$, the estimator approximates Eq.~(\ref{eq:loss}) using $K(S)$ Monte Carlo background samples:
\begin{equation}
\widehat{L}(S) = \frac{1}{K(S)} \sum_{k=1}^{K(S)} \ell\left(f(m_S\odot x+(1-m_S)\odot Z^{(k)}),y\right).
\end{equation}
To ensure numerical stability and constant memory overhead, each masked inference batch updates empirical mean $\widehat{L}(S)$ and unbiased sample variance $\widehat{\sigma}^2(S)$ via Welford's algorithm \citep{Welford01081962}. We construct 95\% Student-$t$ confidence intervals for each coalition and use width $w(S)$ for early stopping.
% To ensure numerical stability and constant memory overhead, each new masked inference batch updates the empirical mean $\widehat{L}(S)$ and the unbiased sample variance $\widehat{\sigma}^2(S)$ via Welford's online algorithm \citep{Welford01081962}. We construct 95\% Student-$t$ confidence intervals for each coalition and use the interval width $w(S)$ as a criterion for early stopping.

\paragraph{Epsilon-greedy, softmin exploration and vocabulary amortization.}
Non-converged coalitions are sampled under an $\varepsilon$-greedy mixture: with probability $\varepsilon$ a coalition is drawn uniformly from $\{0,1\}^n$, and with probability $1-\varepsilon$ from a softmin over visit counts, $p(S) \propto \exp(-\beta K(S))$ with temperature $\beta>0$, restricted to the non-converged entries of $\mathcal V$ and expanded into a diamond (Sec.~\ref{sec:method_diamond}). The softmin term prioritizes under-sampled coalitions; the uniform term guarantees $p(S) \ge \varepsilon\,2^{-n} > 0$ for every $S \in \{0,1\}^n$ at every step, which is the positive-mass condition the union-bound argument of Theorem~\ref{thm:vocab} relies on. Amortizing vocabularies across positions reduces cost but introduces cross-input bias; we use per-input vocabularies in main experiments.

\subsection{Diamond sampling}
\label{sec:method_diamond}
The most variance-sensitive components of the decomposition are pairwise interactions, which depend on the second-order mixed difference (or discrete derivative) for context $C$ and feature pair $(i,j)$:
\begin{equation}
\Delta_{ij}(C) \;=\; L(C \cup \{i\}) + L(C \cup \{j\}) - L(C \cup \{i,j\}) - L(C), \quad C \subseteq [n] \setminus \{i,j\}.
\end{equation}
Estimating these four terms with independent background batches incurs high variance. To mitigate this, \ours{} employs \emph{diamond sampling} (named for the diamond shape these four subsets form in the Boolean lattice). We evaluate $C$, $C\cup\{i\}$, $C\cup\{j\}$, and $C\cup\{i,j\}$ simultaneously using the \emph{exact same} background batch $\{Z^{(k)}\}_{k=1}^K$. By coupling the noise across the evaluations, the positive covariance between structurally adjacent coalitions cancels out the variance of the mixed difference. Theorem~\ref{thm:variance} formalizes the strict-improvement condition: adjacent covariance must dominate diagonal covariance. We verify this condition on synthetic SCMs in Appendix~\ref{sec:t_numerical}.
\paragraph{Pair-level synergy and redundancy.}
We define pair-level synergy $S_{ij}$ and redundancy $R_{ij}$ as the non-negative extrema of the diamond mixed difference over all contexts:
\begin{equation}
  S_{ij} = \max_{C} [\,\Delta_{ij}(C)\,]_+, \qquad
  R_{ij} = \max_{C} [\,-\Delta_{ij}(C)\,]_+,
  \label{eq:split_sr}
\end{equation}
where $z+ \mathrel{:=} \max(z,0)$ denotes the positive part. Both quantities are interaction intensities ($S{ij},R_{ij}\ge 0$) derived from the signed increment $\Delta_{ij}(C)$.
At the detection threshold used in the empirical analysis (Appendix~\ref{sec:appendix_e3}), $S_{ij} > 0$ indicates synergistic interaction in at least one context, and $R_{ij} > 0$ indicates evidence of redundancy.
They are computed in the coalition scan (Appendix~\ref{sec:appendix_method}) and used in E3 (Appendix~\ref{sec:appendix_e3}). Eq.~(\ref{eq:split_sr}) is a convention for intensities derived from $\Delta_{ij}(C)$; results (Theorems~\ref{thm:conflation}--\ref{thm:vocab}) operate on $\Delta_{ij}(C)$ and are unaffected by this choice.
% They are computed during the coalition scan (Appendix~\ref{sec:appendix_method}) and used in E3 analysis (Appendix~\ref{sec:appendix_e3}). Eq.\ref{eq:split_sr} is a reporting convention for interaction intensities derived from $\Delta_{ij}(C)$; the theoretical results (Theorems~\ref{thm:conflation}--\ref{thm:vocab}) operate on $\Delta_{ij}(C)$ directly and are unaffected by this choice.
% where $[z]_+ \mathrel{:=} \max(z,0)$ denotes the positive part. Both quantities are non-negative interaction intensities ($S_{ij},R_{ij}\ge 0$) derived from the signed interventional increment $\Delta_{ij}(C)$.
% At the detection threshold employed in the empirical analysis (see Appendix~\ref{sec:appendix_e3}), $S_{ij} > 0$ indicates the presence of synergistic interaction in at least one observed context, and $R_{ij} > 0$ indicates analogous evidence of redundancy;
% They are computed during the deterministic coalition scan (Appendix~\ref{sec:appendix_method}) and used in E3 analysis (Appendix~\ref{sec:appendix_e3}). Eq.~\ref{eq:split_sr} is a reporting convention for pair-level interaction intensities derived from the signed interventional increment $\Delta_{ij}(C)$; the theoretical results (Theorems~1--4) operate on $\Delta_{ij}(C)$ directly and are unaffected by this choice.

% A feature pair $(i,j)$ exhibits synergy if and only if $S_{ij}>0$ in at least one context~$C$, and redundancy if and only if $R_{ij}>0$ in at least one context (at the detection threshold employed in the empirical analysis; see Appendix~\ref{sec:appendix_e3}).

%% file: sections/04_theory.tex
\section{Theory}
\label{sec:theory}
This section formalizes the representational limitations of interaction indices and establishes statistical guarantees of the \ours{} estimator. Proofs are in the Appendices~\ref{sec:appendix_notation}-\ref{sec:proof_vocab}.
% This section formalizes the representational limitations of interaction indices and establishes the statistical guarantees of the \ours{} estimator. Proofs for theorems are provided in the Appendices~\ref{sec:appendix_notation}-\ref{sec:proof_vocab}.

\paragraph{Representational bottleneck.} 
We prove that the failure of indices to capture higher-order cooperation is not merely an approximation artifact, but a limitation of scalar-valued interaction indices: they force either erasing higher-order synergy or projecting it into lower-order artifacts.
% We prove that the failure of pairwise indices to capture higher-order cooperation is not merely an approximation artifact, but a structural limitation of scalar-valued interaction indices: they are mathematically forced to either erase higher-order synergy or project it into lower-order artifacts.

\begin{theorem}[Conflation of signed pair reports]
\label{thm:conflation}
Let $Y=\XOR(X_1,X_2,X_3)$ with $X_1,X_2,X_3 \sim \text{iid Bernoulli}(1/2)$.
Under a uniform binary background distribution $p_{\bg}$, squared loss $\ell(\hat y, y)=(\hat y-y)^2$, and the oracle predictor $f=Y$, for every active pair
$\{i,j\}\subset\{1,2,3\}$, faithful pair indices (STII, Faith-SHAP, high-order
$k$-SII) return zero at order two, while projective indices (SII, order-2
$k$-SII, kADD-SHAP) return projected pair values in $\{-1/4,+1/4\}$. The LOCO
decomposition returns $(U_i,R_i,S_i)=(0,0,1/2)$ for each active feature.
\end{theorem}

\begin{remark}
Theorem~\ref{thm:conflation} is an exact failure mode on the minimal higher-order 3-input XOR SCM, not a universal impossibility statement over all data-generating processes~\citep{DBLP:journals/ijgt/GrabischR99,DBLP:conf/icml/SundararajanDA20,DBLP:journals/jmlr/TsaiYR23}. The proof uses XOR balance: marginalizing any one input of $\XOR(X_1,X_2,X_3)$ under iid Bernoulli$(1/2)$ yields the constant $1/2$, so all order-2 Möbius mass vanishes. More generally, conflation occurs when an SCM has non-zero Möbius mass only at order $\geq 3$ and zero mass at all lower orders; in this regime, scalar indices are forced either to erase higher-order synergy (faithful indices) or project it into misleading lower-order artifacts (projective indices). Extending beyond this class requires SCM-specific analysis.
\end{remark}

% More generally, conflation occurs when an SCM has non-zero Möbius mass only at order $\geq 3$ and zero mass at all lower orders; in this regime, scalar indices are forced either to erase higher-order synergy (faithful indices) or to project it into misleading lower-order artefacts (projective indices). Extending beyond this class requires separate SCM-specific analysis.
% \begin{remark}[Scope and generalisation]
% Theorem~\ref{thm:conflation} is an exact failure mode on the $\XOR$SCM$_3$ and claims no universal impossibility statement over all data-generating processes \citep{DBLP:journals/ijgt/GrabischR99,DBLP:conf/icml/SundararajanDA20,DBLP:journals/jmlr/TsaiYR23}. The proof exploits the XOR-balance property: marginalizing any one input of $\XOR(X_1,X_2,X_3)$ under iid Bernoulli$(1/2)$ produces a constant $1/2$, forcing all order-2 Möbius mass to vanish. Conflation arises whenever an SCM's non-zero Moebius mass lies entirely at order $\geq 3$ and vanishes at all lower orders; the 3-way XOR is the minimal such SCM. Extension to broader SCM families requires separate analysis.
% \end{remark}

% Rather than a universal impossibility statement over all conceivable data-generating processes \citep{DBLP:journals/ijgt/GrabischR99,DBLP:conf/icml/SundararajanDA20,DBLP:journals/jmlr/TsaiYR23}, Theorem~\ref{thm:conflation} serves as an exact failure mode on minimal higher-order SCMs. It proves that when true synergy exists, scalar indices are mathematically forced to either erase it (faithful indices) or project it into misleading lower-order artifacts (projective indices).

\paragraph{Estimation guarantees.}
Because \ours{} relies on stochastic masked inference, we must bound the Monte Carlo approximation error for the coalition vocabulary.

\begin{theorem}[Interventional Monte Carlo bound]
\label{thm:mc}
Assume the loss is bounded such that $\ell\in[\ell_{\min},\ell_{\max}]$, and let
$B=\ell_{\max}-\ell_{\min}$. For any coalition $S$, the empirical vocabulary
mean $\widehat L(S)$ over $K$ iid background samples is unbiased and satisfies:
\[
\E[(\widehat L(S)-L(S))^2] = \frac{\sigma^2(S)}{K},
\qquad
\Pr(|\widehat L(S)-L(S)|>\epsilon)
\le 2\exp\!\left(-\frac{2K\epsilon^2}{B^2}\right).
\]
\end{theorem}

\paragraph{Bounded-loss handling in practice.} The boundedness assumption of Theorem~\ref{thm:mc} must be instantiated per experiment. Our tabular experiments (E1) use bounded regression losses by construction. For vision and language experiments (E2 and E3), we report empirical ranges and quantiles, applying clipping only when strictly justified by the observed support. When clipping is inappropriate, we use empirical-Bernstein style reporting alongside standard Hoeffding bounds \citep{409cf137-dbb5-3eb1-8cfe-0743c3dc925f}.

To justify the diamond sampling strategy introduced in Section~\ref{sec:method_diamond}, we establish the exact condition under which coupling background samples strictly reduces the variance of the mixed difference.

\begin{theorem}[Diamond variance reduction]\label{thm:variance}
Let $C$ be drawn from a fixed distribution over $2^{[n]\setminus\{i,j\}}$ and $Z\sim p_{\bg}$, with the same $(C,Z)$ used to evaluate losses $(a,b,c,d)$ at structurally adjacent coalitions $(C\cup\{i,j\},\,C\cup\{i\},\,C\cup\{j\},\,C)$. All covariances below are over joint $(C,Z)$. Under $\Cov(a,b)+\Cov(c,d)+\Cov(a,c)+\Cov(b,d) > \Cov(a,d)+\Cov(b,c)$ (adjacency-dominant covariance; see Assumption~\ref{as:A3} in Appendix~\ref{sec:proof_variance}), the diamond estimator of $\Delta_{ij}$ has strictly lower variance than an estimator using four independent coalition batches with the same total number of evaluations ($4K$ calls to $\ell$).
% Let $C$ be drawn from a fixed distribution over $2^{[n]\setminus\{i,j\}}$ and $Z\sim p_{\bg}$, with the same $(C,Z)$ used to evaluate the four losses $(a,b,c,d)$ at the structurally adjacent coalitions $(C\cup\{i,j\},\,C\cup\{i\},\,C\cup\{j\},\,C)$. All covariances below are taken over the joint $(C,Z)$. Under $\Cov(a,b)+\Cov(c,d)+\Cov(a,c)+\Cov(b,d) > \Cov(a,d)+\Cov(b,c)$ (adjacency-dominant covariance; see Assumption~\ref{as:A3} in Appendix~D), the diamond estimator of $\Delta_{ij}$ has strictly lower variance than an estimator using four independent coalition batches with the same total number of loss evaluations ($4K$ calls to $\ell$). 
\end{theorem}
% \begin{theorem}[Strict variance reduction]
% \label{thm:variance}
% Let $(a,b,c,d)$ be the losses evaluated at the structurally adjacent coalitions
% $(C\cup\{i,j\},C\cup\{i\},C\cup\{j\},C)$ using a shared background batch. If
% \[
% \Cov(a,b)+\Cov(a,c)+\Cov(b,d)+\Cov(c,d)>
% \Cov(a,d)+\Cov(b,c),
% \]
% then the diamond estimator of $\Delta_{ij}$ has strictly lower variance than an
% estimator using four independent coalition batches with the same total budget.
% \end{theorem}

Quantitatively, the variance gain of diamond coupling is monotone in the adjacency-dominance gap $(\Sigma_{\rm adj}-\Sigma_{\rm diag})$: larger positive gaps imply larger speedups; when the gap is non-positive, no gain is guaranteed. 
We call this Assumption A3 (adjacency-dominant covariance). Its closed-form ratio and empirical boundary cases are reported in Appendix~\ref{sec:proof_variance} and Appendix~\ref{sec:t_numerical}.
% A closed-form ratio is reported in Appendix~\ref{sec:proof_variance} (Theorem~\ref{thm:variance}). We refer to this inequality as the \emph{adjacency-dominant covariance} condition (Assumption A3). Because it compares the covariance of adjacent edges on the Boolean hypercube against its diagonals, it is not guaranteed a priori for all models. We empirically verify its prevalence on synthetic SCMs and report specific failure cases in Appendix~\ref{sec:t_numerical}.

Finally, we provide a uniform convergence guarantee for the entire \URSp{}
profile over a finite coalition vocabulary, enabling our softmin
exploration policy to recover the decomposition once each $S\in\mathcal{V}$
has been visited at least $K$ times (cf. Theorem~\ref{thm:vocab}).
\begin{remark}
The softmin policy of Section~\ref{sec:method} assigns positive sampling
mass to every coalition at every step (for any finite temperature $\beta$),
which is the condition required by the union bound in
Theorem~\ref{thm:vocab}. A formal coupon-collector analysis of the adaptive
schedule under variable per-coalition $K(S)$ is deferred to future work.
\end{remark}

\begin{proposition}[Extremum stability and interval propagation]
\label{prop:extrema_ci}
Assume a vocabulary-level uniform error event
\(\max_{S\in\mathcal V}|\widehat L(S)-L(S)|\le\epsilon\), and that
all contexts needed for feature $i$ are present in $\mathcal V$.
Then for LOCO gains \(G_i(C)=L(C)-L(C\cup\{i\})\) and
\(\widehat G_i(C)=\widehat L(C)-\widehat L(C\cup\{i\})\),
\(|\widehat G_i(C)-G_i(C)|\le 2\epsilon\) uniformly in $C$.
Consequently,
\[
|\widehat U_i-U_i|\le 2\epsilon,
\qquad
|\widehat L_i^{\max}-L_i^{\max}|\le 2\epsilon,
\qquad
|\widehat\pi_i-\pi_i|\le 2\epsilon,
\]
and by algebraic propagation,
\[
|\widehat R_i-R_i|\le 4\epsilon,
\qquad
|\widehat S_i-S_i|\le 4\epsilon.
\]
\end{proposition}
Proposition~\ref{prop:extrema_ci} links uniform coalition error control to deterministic error bounds for extrema and induced \URSp{} channels. Coalition-level 95\% Student-$t$ intervals are propagated to conservative feature-level ranges (Appendix~\ref{sec:proof_vocab}).
% This proposition is the explicit bridge from Theorem~\ref{thm:vocab} to extremum-seeking reliability: once coalition means are uniformly controlled, min/max contexts and the induced \URSp{} channels inherit deterministic error envelopes. We report coalition-level 95\% Student-$t$ intervals and map them to conservative feature-level ranges via the same propagation template (Appendix~\ref{sec:proof_vocab}; protocol details in Appendix~\ref{sec:appendix_experiments}).

\begin{theorem}[Finite-vocabulary convergence]
\label{thm:vocab}
Under the bounded-loss assumption of Theorem~\ref{thm:mc} ($\ell\in[\ell_{\min},\ell_{\max}]$, $B=\ell_{\max}-\ell_{\min}$), for finite vocabulary $\mathcal V\subseteq 2^{[n]}$, if every
$S\in\mathcal V$ is estimated with
\[
K\ge \frac{B^2}{2\epsilon^2}\log\frac{2|\mathcal V|}{\alpha},
\]
then with probability at least $1-\alpha$,
$\max_{S\in\mathcal V}|\widehat L(S)-L(S)|\le\epsilon$.
\end{theorem}

\paragraph{Partial-vocabulary coverage and bias direction.}\label{par:partial_vocab}
When the coalition vocabulary $\mathcal{V}$ does not cover all $2^{n-1}$ contexts for a given feature $i$, the computed $U_i = \min_{C\in\mathcal{V}} \LOCO(i\mid C)$ is an upper bound on the true uniqueness: if the true minimizing context is not in $\mathcal{V}$, the estimated $U_i$ is weakly larger than the true value. This bias is unidirectional and conservative: partial coverage over-estimates uniqueness, not synergy. In E3 ($n=10$) we enumerate exhaustively, so no bias arises. In E2 ($n=196$ patches) we sample, and users should interpret $U_i$ as a conservative upper bound on true uniqueness. Concretely, a positive $\widehat U_i$ from a partial vocabulary does not guarantee true uniqueness, since the true minimizing context may be absent; a negative $\widehat U_i$, however, reliably indicates negative true uniqueness.

\paragraph{Theoretical contract.}
To ensure transparency, Table~\ref{tab:theory_contract} summarizes each theoretical guarantee alongside the specific assumptions it requires and the boundary conditions where it fails.

\input{tables/theoretical_contract}

%% file: tables/theoretical_contract.tex
\begin{table}[htbp]
\centering
\small
\caption{Theoretical contract of the method: guarantees, assumptions, applicability, and explicit boundary/failure regimes for each theorem.}
\label{tab:theory_contract}
\begin{tabular}{@{}llll@{}}
\toprule
\textbf{Theorem} & \textbf{Requires} & \textbf{Applies to} & \textbf{Fails when} \\
\midrule
Thm~\ref{thm:conflation} (Conflation) & Fixed $p_{\bg}$, higher-order SCM & Signed pair indices & Outside SCM family \\
Thm~\ref{thm:mc} (MC bound) & Bounded loss, iid samples & Coalition means & Unbounded loss \\
Thm~\ref{thm:variance} (Variance) & A3 covariance dominance & Pairwise mixed differences & A3 violated \\
Thm~\ref{thm:vocab} (Vocabulary) & Finite set, positive exploration & Sampled coalitions & Infinite/uncovered set \\
\midrule
Prop.~\ref{prop:extrema_ci} (Extrema CI)
  & Uniform bound + full coverage
  & \URSp{} errors
  & Incomplete coverage
\\
\bottomrule
\end{tabular}
\end{table}

%% file: sections/05_experiments.tex
\section{Experiments}
\label{sec:experiments}

\paragraph{Protocol.}
Our evaluation is organized around three falsification targets. E1 (tabular SCMs, conflation) tests Theorem~\ref{thm:conflation}-predicted conflation on SCM models. E2 (ChestX-ray14, localization vs. deletion faithfulness) asks whether patch-coalition explanations are competitive with saliency maps under localization and causal-removal criteria. E3 (IOI transformer circuit, higher-order head interactions beyond first-order patching) asks whether the decomposition recovers a transformer circuit while exposing interactions first-order patching cannot see. This design follows evaluation practice for explanation methods \citep{DBLP:conf/nips/HookerEKK19,DBLP:journals/jmlr/HedstromWKBMSLH23,DBLP:conf/nips/AgarwalKSPJPZL22}. Boundedness is handled at each experiment level: E1 uses bounded squared-loss conditions by construction, whereas E2 and E3 report empirical ranges and quantiles, introducing clipping only when supported by the observed value range. We report mean and standard deviation across seeds ($n=5$), and paired-Wilcoxon $p$-values with Bonferroni correction for paired comparisons; effect sizes are the corresponding Hodges-Lehmann shifts.

\subsection{Tabular SCMs: scalar pair reports conflate mechanisms}
\label{sec:e1}
We first test Theorem~\ref{thm:conflation} on trained two-layer MLPs. The datasets are a 3-way XOR SCM with one spectator feature, an XOR+AND SCM with mixed two-way and three-way structure, and the synthetic third-order dataset (Synth3) with eight continuous features. Each setting uses five seeds. Baselines are STII, Faith-SHAP, SII, $k$-SII, and kADD-SHAP computed through \texttt{shapiq} computation routine on the trained prediction function.

\input{figures/figure_1_conflation}

\input{tables/tab_e1}

Figure~\ref{fig:e1} and Table~\ref{tab:e1} match the theorem's prediction. 
On XOR3, canonical SI baselines conflate distinct interaction roles with ratios up to $411.88\,\times$ (FSII/kADD-SHAP; STII: $319.51\,\times$). \ours{} recovers both roles on Synth3 with strong correlation (U Pearson $0.978\pm0.022$; pairwise synergy Pearson $0.912\pm0.058$), both above the pre-registered $0.9$ threshold. The lower Spearman values are caused by tied binary ground-truth roles, not by failure of the continuous estimates.

\paragraph{Background sensitivity and higher-order baselines.}
Despite $p_{\bg}$ dependence, E1 rankings remain stable across uniform-binary and empirical-resampled variants (Appendix~\ref{sec:appendix_experiments}): mean absolute drifts are $|\Delta U|=0.023$, $|\Delta R|=0.098$, $|\Delta S|=0.105$, with median relative drift $<10\%$. Moreover, scalar higher-order baselines (T-NID~\citep{DBLP:conf/iccv/LermanVKX21}, Higher-Order IG~\citep{DBLP:journals/jmlr/JanizekSL21}, H-Sets) inherit Theorem~\ref{thm:conflation}'s representational limitation. On the 3-way XOR SCM (Table~\ref{tab:appendix_e1_conflation}), all reject $H_0$ at Bonferroni $\alpha=0.00625$, yielding conflation ratios $\rho_m \in \{14.35\times, 19.82\times, 107.97\times\}$ respectively. The strongest (H-Sets) approaches canonical pair-level indices ($\rho_m \in \{411.14\times, 411.14\times, 411.88\times\}$ for SII/$k$-SII/kADD-SHAP; STII at $319.51\times$), proving that widening the codomain without \ours{}'s \URSp{} channels cannot disentangle mechanisms.

\paragraph{Error accounting.}
The tabular runs separate three error sources. The analytic theorem is exact on the Boolean value function. The trained-network experiment adds model approximation error, because the MLP only approximates the SCM target after finite training. The stochastic vocabulary adds Monte Carlo error, controlled by the Welford estimates in Theorem~\ref{thm:mc}. The observed pattern is stable under all three: scalar pair indices remain small on the triplet pairs, whereas \ours{} concentrates the signal in the synergy coordinate. This is why the result should be interpreted as a representational failure of the baselines, not an estimator-tuning artifact. Exact zeros in $U$ indicate strict additivity under the chosen intervention, whereas non-zero bars reflect finite-sample variability captured by the stated confidence intervals. We verify the diamond covariance condition (A3) empirically on synthetic SCMs in Appendix~\ref{sec:t_numerical} and report both positive and boundary regimes.

\subsection{Clinical evaluation: localization is not deletion}
\label{sec:e2}
We evaluate on the annotated NIH ChestX-ray14 bounding-box subset ($n=880$)~\cite{DBLP:conf/cvpr/WangPLLBS17} using a DenseNet121~\citep{DBLP:conf/cvpr/HuangLMW17} model from \texttt{torchxrayvision} (see Appendix~\ref{sec:appendix_experiments} for full details). \ours{} uses an operational $14\times14$ patch grid and 1000 coalition-sampling steps. Baselines include GradCAM~\cite{DBLP:conf/iccv/SelvarajuCDVPB17}, Vanilla Gradient, Integrated Gradients, and a random saliency map \citep{DBLP:conf/icml/SundararajanN20}. Coarse-grid \URSp{} achieves pointing-game accuracy against GradCAM that is not statistically distinguishable ($+0.018$, $p=0.140$, non-significant at $\alpha=0.0125$), a statistically significant but numerically small IoU@15 deficit ($-0.008$, $p=5.82\!\times\!10^{-3}$), and a substantially lower deletion AUC ($0.332\to0.269$, $\Delta=-0.063$, $p=1.61\!\times\!10^{-78}$). The coarse feature grid limits spatial precision but does not erode overall localization quality, while interventional scoring retains strong causal fidelity for feature-removal evaluation.

\paragraph{Statistical power note.}
Deletion AUC is statistically reliable (paired Wilcoxon $p=1.61\!\times\!10^{-78}$, Bonferroni-surviving across the metric family); localization conclusions should be read as mixed rather than uniformly positive: Pointing Game is non-significant (paired Wilcoxon on $n = 220$ image-pairs after tie-removal; 660 ties dropped; $p = 0.140$), while IoU@15 shows a small but significant deficit ($\Delta=-0.008$, $p=5.82\!\times\!10^{-3}$). 
Full power-analysis details are in Appendix~\ref{sec:appendix_h4_power_note}.

\paragraph{Why report mixed localization outcomes?}
This result is included to keep scope claims honest, not to overstate method performance. \ours{} optimizes an interventional causal objective---how model output changes under coalition removal---whereas standard saliency maps are typically tuned for spatial localization. With the operational $14\times14$ coalition grid on ChestX-ray14, the evidence is mixed: Pointing Game is statistically indistinguishable from GradCAM, IoU@15 shows a small but significant deficit, and deletion AUC is substantially stronger for \ours{}. This pattern is consistent with a resolution-constrained but causally informative regime: coarse coalitions limit fine-grained spatial precision while preserving strong faithfulness for removal-based evaluation. Keeping this result explicit prevents overgeneralization and clarifies the practical claim.

\subsection{Transformer circuit: first-order patching misses synergistic heads}
\label{sec:e3}
We instantiate features as attention heads in GPT-2 on the Indirect Object Identification task \citep{DBLP:conf/iclr/WangVCSS23}. We analyze 10 heads spanning Name-Mover, S-Inhibition, Induction, Duplicate-Token, and Previous-Token roles. Coalition value is negative IOI logit difference under mean ablation, averaged over ABBA/BABA prompts. 
With $n=10$, we enumerate the $2^{10}=1024$ coalition lattice across five prompt seeds. The subset is selected by mechanistic-role coverage~\citep{DBLP:conf/iclr/WangVCSS23}; Appendix~\ref{sec:appendix_experiments} reports robustness checks under alternative subsets and defers full 26-head analysis to Appendix~\ref{sec:appendix_experiments} (\S\ref{sec:appendix_e3_n26}), including per-head and role-pair summaries.
The decomposition recovers the known circuit structure. Within S-Inhibition, pairs are mildly synergy-dominated ($\bar S \approx 0.200$, $\bar R \approx 0.039$), suggesting cooperative suppression over redundant substitution. The strongest synergy is Induction plus S-Inhibition ($\bar S = 0.383$, $\bar R = 0.003$), where heads cooperate to suppress competing tokens. The strongest redundancy is Duplicate-Token plus Induction ($\bar S = 0.000$, $\bar R = 0.226$), indicating a substitutable path. 
The negative control also passes: by construction, $\pi(h)=v(\emptyset)-v(\{h\})$ equals the mean-ablation activation-patching score for one head, so Spearman$(\pi,\mathrm{patch})=1.000\pm0.000$ is an algebraic identity (same ranking across five seeds), not a finite-variance estimate. \ours{} extends activation patching rather than contradicting it. Its signal is the high-synergy, low-singleton region: heads 10.0, 9.9, and 4.11 are nearly invisible to first-order patching but carry pair-level contribution surfaced by \URSp{} decomposition.
\input{figures/figure_ioi}

\paragraph{Compute and reproducibility.}
The approach is computationally tractable at dataset scale ($880$ images); code and results are released (\url{https://anonymous.4open.science/r/stochastic-hi-fi/}), with hardware and runtime details in Appendix~\ref{sec:appendix_compute_reproducibility}.
Across three settings, the point is not that \ours{} wins every metric; the \URSp{} codomain separates modes: third-order structure in SCMs, coarse-grid localization limits in chest radiographs, and synergistic transformer heads that first-order patching cannot isolate. Section~\ref{sec:discussion} makes these boundaries explicit.
% Across the three settings, the point is not that \ours{} wins every metric; it is that the \URSp{} codomain separates otherwise conflated failure modes: third-order structure in SCMs, coarse-grid localization limits in chest radiographs, and synergistic transformer heads that first-order patching cannot isolate. Section~\ref{sec:discussion} makes these scope boundaries explicit.

%% file: figures/figure_1_conflation.tex
\definecolor{colSTII}{HTML}{0072B2}
\definecolor{colFSII}{HTML}{009E73}
\definecolor{colSII}{HTML}{D55E00}
\definecolor{colkSII}{HTML}{CC79A7}
\definecolor{colkADD}{HTML}{E69F00}
\definecolor{colU}{HTML}{1B9E77}
\definecolor{colR}{HTML}{D95F02}
\definecolor{colS}{HTML}{7570B3}

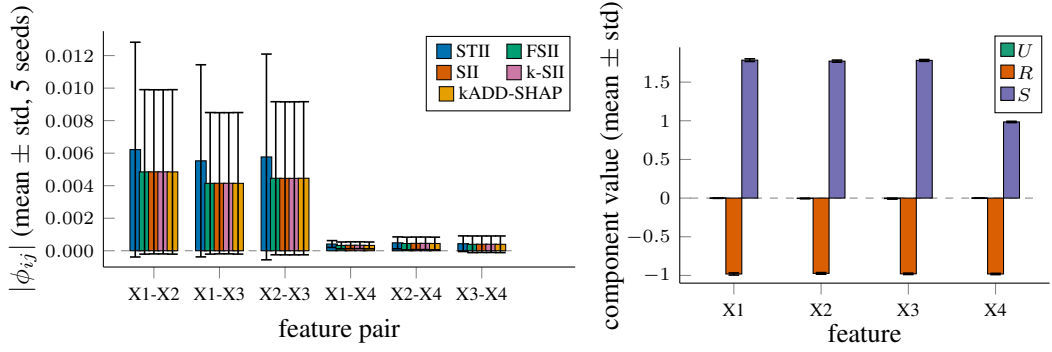
\begin{figure}[h!]
    \def\mysq#1{{\fboxsep=0pt\fboxrule=0.4pt\fbox{\textcolor{#1}{\rule[-0.5pt]{5pt}{5pt}}}}}
    
    % --- PRIMA MINIPAGE (Pannello A) ---
    \begin{minipage}[b]{0.56\textwidth}
        \begin{tikzpicture}
            \begin{axis}[
                width=\linewidth,
                height=4.7cm,
                ybar=-0.5pt,
                bar width=4pt,
                ylabel={$|\phi_{ij}|$ (mean $\pm$ std, 5 seeds)},
                ylabel shift=-3.5pt,
                xlabel={feature pair},
                ymin=-0.001,
                ymax=0.0135,
                xmin=0.23, xmax=6.77, % Leggermente allargato per far respirare le barre
                xtick={1,2,3,4,5,6},
                xticklabels={X1-X2, X1-X3, X2-X3, X1-X4, X2-X4, X3-X4},
                tick label style={font=\scriptsize},
                scaled y ticks=false,
                yticklabel style={
                    /pgf/number format/fixed,
                    /pgf/number format/fixed zerofill,
                    /pgf/number format/precision=3
                },
                ytick={0, 0.002, 0.004, 0.006, 0.008, 0.010, 0.012},
                axis x line*=bottom,
                axis y line*=left,
                enlarge x limits={lower=0.06, upper=0.05},
                error bars/y dir=both,
                error bars/y explicit,
                error bars/error bar style={line width=0.6pt},
                error bars/error mark options={rotate=90, mark size=2pt, line width=0.6pt}
            ]
            
            \draw[gray, dashed, very thin] (0.5, 0) -- (6.5, 0);

            % Lettura dati dal CSV per il Pannello A
            \addplot[fill=colSTII, draw=black, line width=0.4pt] 
                table[x=pair_idx, y=mean_STII, y error=std_STII, col sep=comma] {data/figure1/figure1_panelA.csv};
                
            \addplot[fill=colFSII, draw=black, line width=0.4pt] 
                table[x=pair_idx, y=mean_FSII, y error=std_FSII, col sep=comma] {data/figure1/figure1_panelA.csv};
                
            \addplot[fill=colSII, draw=black, line width=0.4pt] 
                table[x=pair_idx, y=mean_SII, y error=std_SII, col sep=comma] {data/figure1/figure1_panelA.csv};
                
            \addplot[fill=colkSII, draw=black, line width=0.4pt] 
                table[x=pair_idx, y=mean_kSII, y error=std_kSII, col sep=comma] {data/figure1/figure1_panelA.csv};
                
            \addplot[fill=colkADD, draw=black, line width=0.4pt] 
                table[x=pair_idx, y=mean_kADDSHAP, y error=std_kADDSHAP, col sep=comma] {data/figure1/figure1_panelA.csv};

            % --- LEGENDA CUSTOM CON QUADRATI PERFETTI ---
            \node[draw=black, fill=white, anchor=north east, inner sep=2pt, font=\scriptsize] at (rel axis cs: 0.99,0.98) {
                \setlength{\tabcolsep}{3pt}
                \begin{tabular}{@{}ll@{}}
                    \mysq{colSTII} STII & \mysq{colFSII} FSII \\
                    \mysq{colSII} SII   & \mysq{colkSII} k-SII \\
                    \multicolumn{2}{@{}c@{}}{\mysq{colkADD} kADD-SHAP}
                \end{tabular}
            };
            \end{axis}
        \end{tikzpicture}
    \end{minipage}%
    \hfill %
    % --- SECONDA MINIPAGE (Pannello B) ---
    \begin{minipage}[b]{0.45\textwidth}
        \begin{tikzpicture}
            \begin{axis}[
                width=1.03\linewidth,
                height=4.7cm,
                ybar=0pt,
                bar width=6pt,
                ylabel={component value (mean $\pm$ std)},
                ylabel shift=-6pt,
                xlabel={feature},
                xlabel shift=-4pt,
                ymin=-1.1, ymax=1.95,
                xmin=0.5, xmax=4.5,
                xtick={1,2,3,4},
                xticklabels={X1, X2, X3, X4},
                ytick={-1, -0.5, 0, 0.5, 1, 1.5},
                tick label style={font=\scriptsize},
                legend style={
                    at={(0.98,1.05)},
                    anchor=north east,
                    legend cell align=left,
                    draw=black,
                    fill=white,
                    font=\scriptsize,
                    legend columns=1,
                    inner sep=2pt,
                    nodes={inner sep=1pt}
                },
                legend image code/.code={
                    \filldraw[#1, draw=black, line width=0.4pt] (0pt,-1.5pt) rectangle (5pt,3.5pt);
                },
                axis x line*=bottom,
                axis y line*=left,
                enlarge x limits=0.03,
                error bars/y dir=both,
                error bars/y explicit,
                error bars/error bar style={line width=0.6pt},
                error bars/error mark options={rotate=90, mark size=2pt, line width=0.6pt}
            ]
            
            \draw[gray, dashed, very thin] (0.5, 0) -- (4.5, 0); 

            % Lettura dati dal CSV per il Pannello B
            \addplot[fill=colU, draw=black, line width=0.4pt] 
                table[x=feat_idx, y=U_mean, y error=U_std, col sep=comma] {data/figure1/figure1_panelB.csv};
            \addlegendentry{$U$}

            \addplot[fill=colR, draw=black, line width=0.4pt] 
                table[x=feat_idx, y=R_mean, y error=R_std, col sep=comma] {data/figure1/figure1_panelB.csv};
            \addlegendentry{$R$}

            \addplot[fill=colS, draw=black, line width=0.4pt] 
                table[x=feat_idx, y=S_mean, y error=S_std, col sep=comma] {data/figure1/figure1_panelB.csv};
            \addlegendentry{$S$}

            \end{axis}
        \end{tikzpicture}
    \end{minipage}
    
    \caption{\textbf{Left:} Pair-level scalar indices on the 3-way XOR SCM: faithful indices collapse higher-order signal to near-zero pair scores, while projective indices redistribute the triplet effect into small signed pair coefficients. \textbf{Right:} Per-feature \URSp{} profile from \ours{} (mean $\pm$ std over 5 seeds): the three active features are synergy-dominated, while the spectator is separated.}
    \label{fig:e1}
\end{figure}

%% file: tables/tab_e1.tex
\begin{figure}[t]
    \centering
    % --- Inizio Tabella ---
    \begin{minipage}[t]{0.50\textwidth} % Aumentato lo spazio per la tabella
        \vspace{0pt} % Trucco per l'allineamento in alto
        \captionof{table}{E1 disambiguation and recovery. \textbf{Left:} conflation ratios on XOR-family SCMs show scalar baselines retain less mechanism-specific signal than \ours{}. \textbf{Right:} on Synth3, \ours{} recovers roles with high correlation for uniqueness and per-pair synergy.}
        % \captionof{table}{E1 disambiguation and recovery. \textbf{Left:} conflation ratios on XOR-family SCMs show that canonical scalar interaction baselines retain substantially less mechanism-specific signal than \ours{}. \textbf{Right:} on Synth3, \ours{} recovers ground-truth roles with high correlation for per-feature uniqueness and per-pair synergy (mean $\pm$ std over 5 seeds).}
        \label{tab:e1}
        
        \scriptsize
        \setlength{\tabcolsep}{1pt} % Riduce il padding orizzontale tra le colonne
        
        \begin{subtable}[t]{0.52\textwidth}
        \centering
        \label{tab:e1_ratio}
        \begin{tabular*}{\linewidth}{@{\extracolsep{\fill}}lrr@{}}
\toprule
\textbf{Baseline} & \textbf{$\XOR$3} & \textbf{$\XOR$+$\land$} \\
        \midrule
        STII & $319.51\times$ & $102.55\times$ \\
        Faith-SHAP & $411.88\times$ & $110.68\times$ \\
        SII & $411.14\times$ & $110.66\times$ \\
        $k$-SII & $411.14\times$ & $110.66\times$ \\
        kADD-SHAP & $411.88\times$ & $110.68\times$ \\
        \bottomrule
        \end{tabular*}
        \end{subtable}
        \hfill
        \begin{subtable}[t]{0.47\textwidth}
        \centering
        \label{tab:e1_recovery}
        \begin{tabular}{@{}lrr@{}}
        \toprule
        \textbf{Target} & \textbf{Pearson $r$} & \textbf{Spearman $\rho$} \\
        \midrule
        $U$ & $0.978\scriptscriptstyle\pm\scriptstyle0.02$ & $0.577\scriptscriptstyle\pm\scriptstyle0.0$ \\
        $S_{\mathrm{pair}}$ & $0.912\scriptscriptstyle\pm\scriptstyle0.06$ & $0.548\scriptscriptstyle\pm\scriptstyle0.1$ \\
        \bottomrule
        \end{tabular}
        \end{subtable}

\captionof{table}{ChestX-ray14 metrics (Pointing Game, IoU@15, Deletion AUC). Higher is better for Pointing Game and IoU@15; lower for Deletion AUC. Results show mixed localization, while \URSp{} is strongest on deletion-faithfulness.}
% \captionof{table}{ChestX-ray14 attribution metrics (Pointing Game, IoU@15, Deletion AUC). Higher is better for Pointing Game and IoU@15; lower is better for Deletion AUC. The table reflects the scope-honest pattern: localization is mixed across saliency methods, while interventional \URSp{} scoring is strongest on deletion-faithfulness.}
        \label{tab:e2}
        
        \scriptsize
        \setlength{\tabcolsep}{3pt} % Riduce il padding orizzontale tra le colonne

        \begin{tabular*}{\linewidth}{@{\extracolsep{\fill}}lrrr@{}}
        \toprule
        \textbf{Method} & \textbf{PGame $\uparrow$} & \textbf{IoU@15 $\uparrow$} & \textbf{DelAUC $\downarrow$} \\
        \midrule
        GradCAM & 0.268 & 0.142 & 0.332\\
        Vanilla Gradient & \textbf{0.433} & \textbf{0.145} & 0.315\\
        Integrated Gradients & 0.335 & 0.114 & 0.294\\
        Random & 0.030 & 0.043 & 0.309\\
        \textbf{\ours{} (Ours)} & 0.286 & 0.134 & \textbf{0.269}\\
        \bottomrule
        \end{tabular*}

\centering
        \captionof{table}{Role-pair summary on \textsc{gpt2-small} IOI (mean over 5 seeds). Within-role blocks are redundancy-dominated, while selected cross-role combinations show stronger synergy.}
        % \captionof{table}{Role-pair summary on \textsc{gpt2-small} IOI (mean over 5 seeds). Within-role blocks are predominantly redundancy-dominated, while selected cross-role combinations show stronger synergy, consistent with compositional circuit behavior.}
        \label{tab:e3}
        \scriptsize
        \begin{tabular*}{\linewidth}{@{\extracolsep{\fill}}lrrr@{}}
\toprule
Role pair & pairs & mean $S$ & mean $R$\\
        \midrule
        Name-Mover within & 3 & 0.272 & 0.165\\
        S-Inhibition within & 6 & 0.200 & 0.039\\
        Induction + S-Inhibition & 4 & \textbf{0.383} & 0.003\\
        Duplicate-Token + S-Inhibition & 4 & 0.266 & 0.000\\
        Duplicate-Token + Induction & 1 & 0.000 & \textbf{0.226}\\
        Induction + Previous-Token & 1 & 0.282 & 0.000\\
        \bottomrule
        \end{tabular*}

    \end{minipage}\hfill
    % --- Inizio Figura ---
    \begin{minipage}[t]{0.45\textwidth} % Ridotto lo spazio per la figura
    \centering
        \vspace{0pt} % Trucco per l'allineamento in alto

        \setlength{\tabcolsep}{3pt}
        \captionof{table}{Per-head IOI decomposition. $U$, $R$, and $S$ denote unique, redundant, and synergistic contributions; singleton LOCO $\pi$ equals the mean-ablation activation-patching score. Heads with small or negative $\pi$ can still exhibit large pair-level synergy.}
        \label{tab:e3_heads}
        \scriptsize
        \begin{tabular*}{\linewidth}{@{\extracolsep{\fill}}llrrrr@{}}
\toprule
Head & Role & $U$ & $R$ & $S$ & $\pi$\\
        \midrule
        9.9 & Name-Mover & $-0.271$ & $0.219$ & $0.586$ & $-0.052$\\
        9.6 & Name-Mover & $-1.214$ & $1.368$ & $0.038$ & $0.153$\\
        10.0 & Name-Mover & $-0.245$ & $0.010$ & $1.018$ & $-0.235$\\
        7.3 & S-Inhibition & $0.135$ & $0.087$ & $0.430$ & $0.222$\\
        7.9 & S-Inhibition & $0.366$ & $0.091$ & $0.673$ & $0.457$\\
        8.6 & S-Inhibition & $0.320$ & $0.147$ & $1.348$ & $0.467$\\
        8.10 & S-Inhibition & $0.274$ & $0.148$ & $0.848$ & $0.422$\\
        5.5 & Induction & $0.222$ & $0.148$ & $1.139$ & $0.369$\\
        3.0 & Duplicate-Token & $0.156$ & $0.164$ & $0.796$ & $0.320$\\
        4.11 & Previous-Token & $-0.041$ & $0.021$ & $0.461$ & $-0.020$\\           
        \bottomrule
        \end{tabular*}

        \centering
        \vspace{0pt} % Trucco per l'allineamento in alto
        \input{figures/figure3_qualitative}
        % \captionof{figure}{
        % Qualitative comparison on a Cardiomegaly X-ray. From left to right: original image (ground truth in green), \ours{}, GradCAM, Integrated Gradients, Vanilla Gradient, and Random.
        % }
        \captionof{figure}{Qualitative ChestX-ray comparison on a Cardiomegaly case: original image with ground-truth box (white), \ours{}, GradCAM, Integrated Gradients, Vanilla Gradient, and Random. This panel is descriptive and complements the quantitative localization/deletion metrics.}
        \label{fig:e2_qualitative}

    \end{minipage}
\end{figure}

%% file: figures/figure3_qualitative.tex
\begin{tikzpicture}
\begin{scope}[local bounding box=panel_image,
                shift={(0.000cm, 0.000cm)}]
  \node[anchor=south west, inner sep=0pt] (img_image)
        {\includegraphics[width=1.60cm]{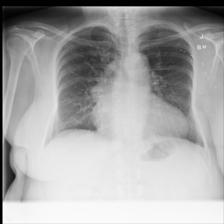}};
  \begin{scope}[shift={(img_image.south west)},
                 x={(img_image.south east)}, y={(img_image.north west)}]
    \draw[white!, line width=1.0pt] (0.3884, 0.2946) rectangle (0.8661, 0.6652);
  \end{scope}
  \node[anchor=south, font=\tiny, yshift=0.05cm]
        at (img_image.north) {Image + GT bbox};
\end{scope}
\begin{scope}[local bounding box=panel_ours,
                shift={(1.800cm, 0.000cm)}]
  \node[anchor=south west, inner sep=0pt] (img_ours)
        {\includegraphics[width=1.60cm]{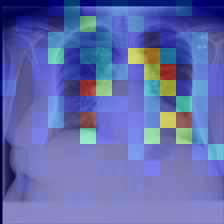}};
  \begin{scope}[shift={(img_ours.south west)},
                 x={(img_ours.south east)}, y={(img_ours.north west)}]
    \draw[white!, line width=1.0pt] (0.3884, 0.2946) rectangle (0.8661, 0.6652);
  \end{scope}
  \node[anchor=south, font=\tiny, yshift=0.05cm]
        at (img_ours.north) {Ours $(U{+}R{+}S)$};
\end{scope}
\begin{scope}[local bounding box=panel_gradcam,
                shift={(3.600cm, 0.000cm)}]
  \node[anchor=south west, inner sep=0pt] (img_gradcam)
        {\includegraphics[width=1.60cm]{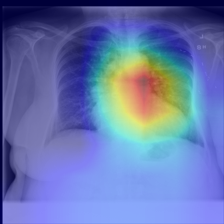}};
  \begin{scope}[shift={(img_gradcam.south west)},
                 x={(img_gradcam.south east)}, y={(img_gradcam.north west)}]
    \draw[white!, line width=1.0pt] (0.3884, 0.2946) rectangle (0.8661, 0.6652);
  \end{scope}
  \node[anchor=south, font=\tiny, yshift=0.05cm]
        at (img_gradcam.north) {GradCAM};
\end{scope}
\begin{scope}[local bounding box=panel_ig,
                shift={(0.000cm, -2.100cm)}]
  \node[anchor=south west, inner sep=0pt] (img_ig)
        {\includegraphics[width=1.60cm]{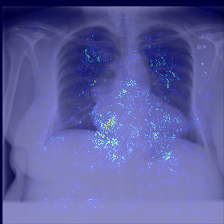}};
  \begin{scope}[shift={(img_ig.south west)},
                 x={(img_ig.south east)}, y={(img_ig.north west)}]
    \draw[white!, line width=1.0pt] (0.3884, 0.2946) rectangle (0.8661, 0.6652);
  \end{scope}
  \node[anchor=south, font=\tiny, yshift=0.05cm]
        at (img_ig.north) {Integrated Gradients};
\end{scope}
\begin{scope}[local bounding box=panel_vanilla,
                shift={(1.800cm, -2.100cm)}]
  \node[anchor=south west, inner sep=0pt] (img_vanilla)
        {\includegraphics[width=1.60cm]{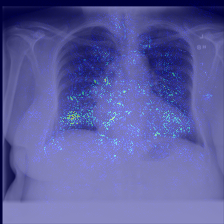}};
  \begin{scope}[shift={(img_vanilla.south west)},
                 x={(img_vanilla.south east)}, y={(img_vanilla.north west)}]
    \draw[white!, line width=1.0pt] (0.3884, 0.2946) rectangle (0.8661, 0.6652);
  \end{scope}
  \node[anchor=south, font=\tiny, yshift=0.05cm]
        at (img_vanilla.north) {Vanilla Gradient};
\end{scope}
\begin{scope}[local bounding box=panel_random,
                shift={(3.600cm, -2.100cm)}]
  \node[anchor=south west, inner sep=0pt] (img_random)
        {\includegraphics[width=1.60cm]{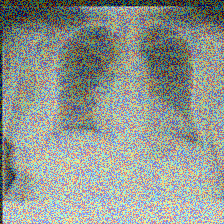}};
  \begin{scope}[shift={(img_random.south west)},
                 x={(img_random.south east)}, y={(img_random.north west)}]
    \draw[white!, line width=1.0pt] (0.3884, 0.2946) rectangle (0.8661, 0.6652);
  \end{scope}
  \node[anchor=south, font=\tiny, yshift=0.05cm]
        at (img_random.north) {Random};
\end{scope}
\end{tikzpicture}

%% file: figures/figure_ioi.tex
% --- Inizio Figura ---
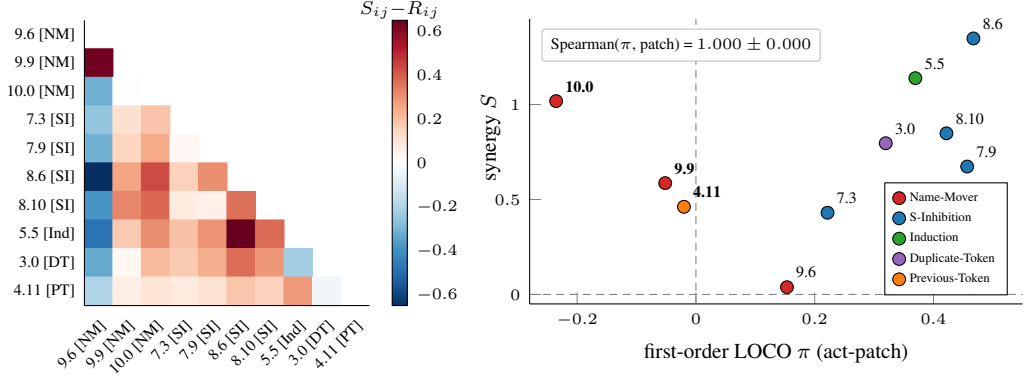
\begin{figure}[t]
\centering

% --- Setup Dati per Panel A ---
\pgfplotstableread[col sep=comma]{data/figure3/figure3_panelA.csv}\datatableA
\pgfplotstableforeachcolumnelement{i}\of\datatableA\as\vali{%
    \pgfplotstablegetelem{\pgfplotstablerow}{head_i}\of\datatableA
    \let\head=\pgfplotsretval
    \pgfplotstablegetelem{\pgfplotstablerow}{role_i}\of\datatableA
    \let\role=\pgfplotsretval
    \expandafter\xdef\csname mylabel\vali\endcsname{\head\space[\role]}%
}
\pgfplotstableforeachcolumnelement{j}\of\datatableA\as\valj{%
    \pgfplotstablegetelem{\pgfplotstablerow}{head_j}\of\datatableA
    \let\head=\pgfplotsretval
    \pgfplotstablegetelem{\pgfplotstablerow}{role_j}\of\datatableA
    \let\role=\pgfplotsretval
    \expandafter\xdef\csname mylabel\valj\endcsname{\head\space[\role]}%
}

% --- Setup Colori per Panel B ---
\definecolor{tabblue}{RGB}{31,119,180}
\definecolor{taborange}{RGB}{255,127,14}
\definecolor{tabgreen}{RGB}{44,160,44}
\definecolor{tabred}{RGB}{214,39,40}
\definecolor{tabpurple}{RGB}{148,103,189}

% Prima Immagine (Panel A)
% Ridotto a 0.43 per lasciare lo spazio laterale alla colorbar
\begin{minipage}{0.27\textwidth}
\centering
\begin{tikzpicture}
\begin{axis}[
    width=\linewidth,
    height=\linewidth,
    enlargelimits=false,
    % --- RIMUOVI IL QUADRANTE (SOLO L-SHAPE) ---
    axis x line*=bottom,
    axis y line*=left,
    % -------------------------------------------
    scale only axis,
    xmin=0, xmax=10,
    ymin=0, ymax=10,
    xtick style={draw=none}, 
    ytick style={draw=none},
    xtick={0.5,1.5,...,9.5},
    xticklabel={%
        \pgfmathtruncatemacro{\idx}{\tick - 0.5}%
        \csname mylabel\idx\endcsname
    },
    ytick={0.5,1.5,...,9.5},
    yticklabel={%
        \pgfmathtruncatemacro{\idx}{\tick - 0.5}%
        \csname mylabel\idx\endcsname
    },
    xticklabel style={rotate=45, anchor=north east, font=\tiny},
    yticklabel style={font=\tiny},
    colorbar,
    colormap={rdbu}{
        rgb255(0cm)=(5,48,97)
        rgb255(1cm)=(33,102,172)
        rgb255(2cm)=(67,147,195)
        rgb255(3cm)=(146,197,222)
        rgb255(4cm)=(209,229,240)
        rgb255(5cm)=(255,255,255)
        rgb255(6cm)=(253,219,199)
        rgb255(7cm)=(244,165,130)
        rgb255(8cm)=(214,96,77)
        rgb255(9cm)=(178,24,43)
        rgb255(10cm)=(103,0,31)
    },
    point meta min=-0.65,
    point meta max=0.65,
    colorbar style={
        title={$S_{ij}{-}R_{ij}$}, % Spostato a titolo per salvare spazio orizzontale
        title style={font=\scriptsize, yshift=-2ex},
        ytick={-0.6,-0.4,-0.2,0,0.2,0.4,0.6},
        ticklabel style={font=\tiny},
        width=0.2cm,
        tick style={draw=none},
        at={(parent axis.south east)}, % La ancora all'angolo della matrice
        anchor=south west,             % Usa l'angolo della barra come gancio
        xshift=0pt
    }
]

\addplot[
    matrix plot,
    mesh/cols=10,
    point meta=explicit,
    unbounded coords=jump
] table[
    col sep=comma,
    x expr={\thisrow{j} + 0.5},
    y expr={\thisrow{i} + 0.5},
    meta expr={ \thisrow{i} > \thisrow{j} ? \thisrow{smr} : nan }
] {data/figure3/figure3_panelA_grid.csv};
        
\end{axis}
\end{tikzpicture}
\end{minipage}%
\hfill%
% Seconda Immagine (Panel B)
\begin{minipage}{0.58\textwidth}
\centering
\begin{tikzpicture}
\begin{axis}[
    xlabel={first-order LOCO $\pi$ (act-patch)},
    ylabel={synergy $S$},
    xlabel style={font=\footnotesize},
    ylabel style={font=\footnotesize, xshift=2ex, yshift=-2ex},
    width=\linewidth,
    height=0.66\linewidth,
    ticklabel style={font=\tiny},
    xmin=-0.28, xmax=0.55,
    ymin=-0.05, ymax=1.45,
    axis x line*=bottom,
    axis y line*=left,
    legend pos=south east,
    legend style={
        draw=black, 
        fill=white,
        font=\tiny,
        cells={anchor=west},
        nodes={scale=0.8, transform shape}, % Scala l'intero contenuto della legenda
        nodes={inner sep=2pt} % Riduce lo spazio tra le righe della legenda
    },
    grid=none
]

\draw[densely dashed, gray] (axis cs:0,\pgfkeysvalueof{/pgfplots/ymin}) -- (axis cs:0,\pgfkeysvalueof{/pgfplots/ymax});
\draw[densely dashed, gray] (axis cs:\pgfkeysvalueof{/pgfplots/xmin},0) -- (axis cs:\pgfkeysvalueof{/pgfplots/xmax},0);

\node[draw=lightgray, fill=white, rounded corners=1pt, anchor=north west, font=\tiny] 
    at (axis cs:-0.26, 1.4) {Spearman($\pi$, patch) = $1.000 \pm 0.000$};

\addplot[
    only marks,
    mark=*,
    mark options={scale=1.2, draw=black, fill=tabred},
    visualization depends on={value \thisrow{highlight} \as \hl},
    nodes near coords={%
        \pgfmathtruncatemacro{\ishl}{\hl}%
        \ifnum\ishl=1
            \textbf{\pgfplotspointmeta}%
        \else
            \pgfplotspointmeta
        \fi
    },
    nodes near coords style={anchor=south west, font=\tiny, text=black},
    point meta=explicit symbolic
] table[col sep=comma, x=pi, y=S, meta=head] {data/figure3/figure3_panelB_NameMover.csv};
\addlegendentry{Name-Mover}

\addplot[
    only marks,
    mark=*,
    mark options={scale=1.2, draw=black, fill=tabblue},
    visualization depends on={value \thisrow{highlight} \as \hl},
    nodes near coords={%
        \pgfmathtruncatemacro{\ishl}{\hl}%
        \ifnum\ishl=1
            \textbf{\pgfplotspointmeta}%
        \else
            \pgfplotspointmeta
        \fi
    },
    nodes near coords style={anchor=south west, font=\tiny, text=black},
    point meta=explicit symbolic
] table[col sep=comma, x=pi, y=S, meta=head] {data/figure3/figure3_panelB_SInhibition.csv};
\addlegendentry{S-Inhibition}

\addplot[
    only marks,
    mark=*,
    mark options={scale=1.2, draw=black, fill=tabgreen},
    visualization depends on={value \thisrow{highlight} \as \hl},
    nodes near coords={%
        \pgfmathtruncatemacro{\ishl}{\hl}%
        \ifnum\ishl=1
            \textbf{\pgfplotspointmeta}%
        \else
            \pgfplotspointmeta
        \fi
    },
    nodes near coords style={anchor=south west, font=\tiny, text=black},
    point meta=explicit symbolic
] table[col sep=comma, x=pi, y=S, meta=head] {data/figure3/figure3_panelB_Induction.csv};
\addlegendentry{Induction}

\addplot[
    only marks,
    mark=*,
    mark options={scale=1.2, draw=black, fill=tabpurple},
    visualization depends on={value \thisrow{highlight} \as \hl},
    nodes near coords={%
        \pgfmathtruncatemacro{\ishl}{\hl}%
        \ifnum\ishl=1
            \textbf{\pgfplotspointmeta}%
        \else
            \pgfplotspointmeta
        \fi
    },
    nodes near coords style={anchor=south west, font=\tiny, text=black},
    point meta=explicit symbolic
] table[col sep=comma, x=pi, y=S, meta=head] {data/figure3/figure3_panelB_DuplicateToken.csv};
\addlegendentry{Duplicate-Token}

\addplot[
    only marks,
    mark=*,
    mark options={scale=1.2, draw=black, fill=taborange},
    visualization depends on={value \thisrow{highlight} \as \hl},
    nodes near coords={%
        \pgfmathtruncatemacro{\ishl}{\hl}%
        \ifnum\ishl=1
            \textbf{\pgfplotspointmeta}%
        \else
            \pgfplotspointmeta
        \fi
    },
    nodes near coords style={anchor=south west, font=\tiny, text=black},
    point meta=explicit symbolic
] table[col sep=comma, x=pi, y=S, meta=head] {data/figure3/figure3_panelB_PreviousToken.csv};
\addlegendentry{Previous-Token}

\end{axis}
\end{tikzpicture}
\end{minipage}
\caption{IOI attention-head decomposition on the 10-head circuit. \textbf{Left:} pair interaction map $S_{ij}-R_{ij}$ (lower triangle, 45 pairs), separating redundancy-dominated substitution blocks from synergy-dominated cross-role composition. \textbf{Right:} per-head synergy $S$ versus first-order LOCO $\pi$ (activation patching identity): heads with weak singleton effect carry substantial pair-level synergy.}
\label{fig:e3}
\end{figure}

%% file: sections/06_discussion.tex
\section{Discussion}
\label{sec:discussion}
\ours{} is a model-fidelity method: it explains a fixed predictor under an explicit background intervention. It does not estimate what a retrained model could learn from a subset, and it does not claim a uniquely correct background. This is a feature for auditing deployed systems \citep{DBLP:conf/aistats/JanzingMB20}.

\paragraph{Negative uniqueness and scope honesty.}
Negative uniqueness ($U<0$) should be interpreted mechanistically as evidence of role structure rather than estimator failure: in the IOI circuit, head 9.6 has $\widehat U_{9.6}=-1.214$ and the $\widehat U_{9.6}+2\hat\sigma_{U_{9.6}}<0$ test confirms the sign is reliable at $\approx 95\%$ confidence, consistent with the backup name-mover behavior documented in IOI \citep{DBLP:conf/iclr/WangVCSS23}.
To distinguish genuine backup behavior from estimation noise, one should examine whether $\widehat U_i + 2\,\hat\sigma_{U_i} < 0$, where $\hat\sigma_{U_i}$ is the across-seed standard deviation of $\widehat U_i$ (equivalently, the Welford envelope $2\,t_{\alpha/2,K-1}\,\hat\sigma(C^*)/\sqrt{K}$ from Appendix~\ref{sec:proof_mc}); if so, the negative sign is reliable at approximately $95\%$ confidence.

\paragraph{Limitations and broader impact.}
Under sampling, $U_i$ is a conservative upper bound on true uniqueness (Section~\ref{sec:theory}, paragraph~\ref{par:partial_vocab}). In imaging, coarse-grid \URSp{} trades fine spatial resolution for stronger causal faithfulness (Section~\ref{sec:experiments}). Guarantees depend on $p_{\bg}$, requiring explicit sensitivity reporting (Appendix~\ref{sec:appendix_experiments}). While disentanglement reduces false confidence from scalar summaries~\citep{DBLP:journals/ijgt/GrabischR99,DBLP:journals/jmlr/TsaiYR23}, ignoring intervention semantics or failing to re-validate $p_{\bg}$ risks misuse~\citep{DBLP:conf/nips/HookerEKK19}. Computationally, \ours{} avoids LOCO retraining but requires $K\cdot|\mathcal{V}|$ forward passes (Appendix~\ref{sec:appendix_compute_reproducibility}). Finally, we mitigate Theorem~\ref{thm:mc}'s bounded-loss violation in E3 using Student-t confidence intervals and conservative Hoeffding bounds (Appendix~\ref{sec:proof_mc}).

% \paragraph{Future work.}
% Natural extensions include hierarchical/multi-resolution coalition grids to address the E2 IoU@15 gap, sub-Gaussian Bernstein refinements for the unbounded losses encountered in E3, and finer-grained $p_{\bg}$ ablations along the lines of Appendix~\ref{sec:appendix_experiments}.

\paragraph{Conclusion and future work.}
Signed pairwise interaction is a useful projection, not a complete explanation of cooperation. On the 3-way XOR SCM (Theorem~\ref{thm:conflation}), it is provably insufficient. \ours{} replaces that projection with a post-hoc \URSp{} profile, presents estimation guarantees with assumptions, and shows the structure across tabular SCMs, medical imaging on ChestX-ray14, and transformer circuits. Interaction explanations should report mechanisms, not signs. Extensions include hierarchical/multi-resolution coalition grids to address the E2 IoU@15 gap, sub-Gaussian Bernstein refinements for the unbounded losses, and finer-grained $p_{\bg}$ ablations as in Appendix~\ref{sec:appendix_experiments}.
% Signed pairwise interaction is a useful projection, not a complete explanation of cooperation. On the 3-way XOR SCM (Theorem~\ref{thm:conflation}), it is provably insufficient. \ours{} replaces that projection with a post-hoc \URSp{} profile, presents estimation guarantees with explicit assumptions, and shows the extra structure across tabular SCMs, full-scale medical imaging on ChestX-ray14 ($n=880$), and transformer circuits. Interaction explanations should report mechanisms, not just signs. Natural extensions include hierarchical/multi-resolution coalition grids to address the E2 IoU@15 gap, sub-Gaussian Bernstein refinements for the unbounded losses encountered in E3, and finer-grained $p_{\bg}$ ablations along the lines of Appendix~\ref{sec:appendix_experiments}.

%% file: sections/appendix/A_notation.tex
\section{Notation for the proofs}
\label{sec:appendix_notation}

Throughout the appendix, $[n] = \{1, \dots, n\}$ and $p_{\bg}$ is the uniform product measure on $\{0,1\}^n$: $p_{\bg}(x) = 2^{-n}$. For a value function $v_x : 2^{[n]} \to \R$, the \emph{M\"obius coefficient} at subset $T \subseteq [n]$ is \begin{equation} \label{eq:moebius} m_x(T) := \sum_{S \subseteq T} (-1)^{|T| - |S|} v_x(S), \end{equation} with inverse $v_x(T) = \sum_{S \subseteq T} m_x(S)$. Every scalar-valued pairwise interaction index used in the literature---SII, $k$-SII, STII, Faith-SHAP, kADD-SHAP---is a linear combination of M\"obius coefficients. For \emph{faithful} indices (STII, FSII, kADD-SHAP in their ``Taylor/kernel'' formulations, and M\"obius itself) the combination at $|S|=2$ involves only the pairwise M\"obius coefficient $m_x(\{i,j\})$~\cite{DBLP:conf/icml/SundararajanDA20, DBLP:journals/jmlr/TsaiYR23}. For \emph{projective} indices (SII, $k$-SII at order strictly less than the true interaction depth), the combination aggregates M\"obius coefficients of \emph{all} orders via the Grabisch--Roubens weights~\cite{DBLP:journals/ijgt/GrabischR99}.

For the proofs, we use the interventional loss of Eq.~\eqref{eq:loss} specialized to the oracle model $f(x) = Y(x)$; this standard simplification, used in~\cite{DBLP:conf/icml/SundararajanDA20, DBLP:journals/jmlr/TsaiYR23}, removes the approximation error of a finitely-trained $f$ and lets us reason about the population-level behavior of the indices.

%% file: sections/appendix/B_proof_conflation.tex
\section{Proof of Theorem~\ref{thm:conflation} (Conflation)}
\label{sec:proof_conflation}

\paragraph{Setup.}
Let $Y = \XOR(X_1, X_2, X_3) \in \{0,1\}$ with $X_1, X_2, X_3, \dots, X_n$ independent Bernoulli$(\tfrac12)$. Features $X_4, \dots, X_n$ are \emph{spectators}: $Y$ does not depend on them. The value function at instance $x$ is $v_x(S) := \E_{Z \sim p_{\bg}}[Y(m_S \odot x + (1 - m_S) \odot Z)]$.

\paragraph{Step 1: value function under marginalization.}
XOR has the key balance property that marginalizing over \emph{any single} input of the triplet yields a constant:
\begin{equation}
\label{eq:xor3_marginal}
\E_{X_j}\bigl[\XOR(x_i, x_k, X_j)\bigr] = \tfrac12 \quad \text{for all } x_i, x_k \in \{0,1\}.
\end{equation}
Consequently, if $S \subseteq [n]$ omits at least one of $\{1, 2, 3\}$ then at least one triplet input is integrated under the uniform
$p_{\bg}$ marginal and~\eqref{eq:xor3_marginal} gives
$v_x(S) = \tfrac12$. Spectators contribute nothing, since $Y$ does not
depend on them. Therefore
\begin{equation}
\label{eq:v_cases}
v_x(S) \;=\;
\begin{cases}
\tfrac12 & \text{if } \{1, 2, 3\} \not\subseteq S, \\[2pt]
\XOR(x_1, x_2, x_3) & \text{if } \{1, 2, 3\} \subseteq S.
\end{cases}
\end{equation}

\paragraph{Step 2: M\"obius coefficients.}
From~\eqref{eq:v_cases}:
\begin{itemize}
    \item \textbf{Empty set.} $v_x(\emptyset) = \tfrac12$, so
        $m_x(\emptyset) = \tfrac12$.
    \item \textbf{Singletons.} For any $i \in [n]$, $v_x(\{i\}) = \tfrac12$,
        so $m_x(\{i\}) = \tfrac12 - \tfrac12 = 0$.
    \item \textbf{Pairs.} For any $\{i,j\} \subseteq [n]$,
        $v_x(\{i,j\}) = v_x(\{i\}) = v_x(\{j\}) = v_x(\emptyset) = \tfrac12$,
        so $m_x(\{i, j\}) = \tfrac12 - \tfrac12 - \tfrac12 + \tfrac12 = 0$.
    \item \textbf{Triplet $\{1, 2, 3\}$.}
    \begin{align*}
    m_x(\{1,2,3\}) &= v_x(\{1,2,3\}) - v_x(\{1,2\}) - v_x(\{1,3\}) - v_x(\{2,3\}) \\
                   &\quad + v_x(\{1\}) + v_x(\{2\}) + v_x(\{3\}) - v_x(\emptyset) \\
                   &= \XOR(x_1, x_2, x_3) - 3\cdot\tfrac12 + 3\cdot\tfrac12 - \tfrac12
                    = \XOR(x_1, x_2, x_3) - \tfrac12,
    \end{align*}
    which takes value $\pm\tfrac12$ at every instance $x$.
\end{itemize}
All higher-order M\"obius coefficients either contain spectators (and are zero because their pairwise and singleton corrections cancel by the same argument) or are strictly supersets of $\{1,2,3\}$ (and are zero because adding a spectator to $\{1,2,3\}$ does not change $v_x$). Pairwise M\"obius is exactly zero everywhere; the full signal concentrates at order three.

\paragraph{Step 3: faithful indices vanish at pair level.}
STII, FSII, and $k$-SII at order $\ge 3$ are, by their axiomatic construction~\cite{DBLP:conf/icml/SundararajanDA20, DBLP:journals/jmlr/TsaiYR23}, linear in the M\"obius coefficients and preserve the order at which a signal lives: the order-$k$ output of these indices at $|S| = 2$ depends only on $m_x(\{i,j\})$ and on singleton/empty corrections, not on M\"obius coefficients of order $>2$. Since $m_x(\{i,j\}) = 0$ by Step 2, the pairwise outputs vanish:
\[
\STII_{ij}(x) \;=\; \FSII_{ij}(x) \;=\; k\text{-}\SII_{ij}(x) \;=\; 0.
\]
This establishes Case 1 of Theorem~\ref{thm:conflation}.

\paragraph{Step 4: projective indices yield $\pm\tfrac14$.}
The base Shapley Interaction Index is defined by~\cite{DBLP:journals/ijgt/GrabischR99}:
\begin{equation}
\label{eq:sii_def}
\SII_{ij}(x) \;=\; \sum_{T \subseteq [n] \setminus \{i, j\}}
  \frac{(n - |T| - 2)! \, |T|!}{(n - 1)!}\,
  \Delta_{ij}(T, x),
\end{equation}
where $\Delta_{ij}(T, x) = v_x(T \cup \{i, j\}) - v_x(T \cup \{i\}) - v_x(T \cup \{j\}) + v_x(T)$. For $\{i, j\} \subset \{1, 2, 3\}$, let $k$ be the unique third element of the triplet. Then $\Delta_{ij}(T, x) = \XOR(x_1,x_2,x_3) - \tfrac12$ when $k \in T$ (Step 1 gives $v_x(T \cup \{i,j\}) = \XOR$ and the other three terms are $\tfrac12$), and $\Delta_{ij}(T, x) = 0$ when $k \notin T$ (all four terms are $\tfrac12$). The spectators contribute nothing to whether $\Delta_{ij}(T, x)$ is zero or $\XOR - \tfrac12$; they only affect which subsets $T$ are summed.

For $n \ge 3$, substitute into~\eqref{eq:sii_def}, grouping by whether $T$ contains $k$:
\begin{align*}
\SII_{ij}(x)
&= (\XOR(x_1,x_2,x_3) - \tfrac12)
   \sum_{\substack{T \subseteq [n] \setminus \{i,j\} \\ k \in T}}
   \frac{(n - |T| - 2)!\,|T|!}{(n - 1)!}.
\end{align*}
Reparametrize $T = \{k\} \cup T'$ with $T' \subseteq [n] \setminus \{i,j,k\}$ of size $t' = |T| - 1 \in \{0, \dots, n - 3\}$:
\begin{align*}
\SII_{ij}(x)
&= (\XOR - \tfrac12) \sum_{t' = 0}^{n-3} \binom{n-3}{t'}
   \frac{(n - t' - 3)!\,(t' + 1)!}{(n - 1)!}\\
&= (\XOR - \tfrac12) \sum_{t' = 0}^{n-3}
   \frac{(n-3)!}{t'!\,(n - 3 - t')!}
   \cdot \frac{(n - t' - 3)!\,(t' + 1)!}{(n - 1)!}\\
&= (\XOR - \tfrac12) \cdot \frac{(n-3)!}{(n - 1)!}
   \sum_{t' = 0}^{n-3} (t' + 1)\\
&= (\XOR - \tfrac12) \cdot \frac{1}{(n - 1)(n - 2)} \cdot
   \frac{(n - 2)(n - 1)}{2}
\;=\; \tfrac12 \bigl(\XOR(x_1,x_2,x_3) - \tfrac12\bigr).
\end{align*}
The value is therefore exactly $\pm\tfrac14$ depending on the XOR bit, as claimed. The same combinatorial identity applies to $k$-SII at order $2$ and to kADD-SHAP (both of which project the full interaction back onto the pair-level report via a reweighting that, for this symmetric SCM, reduces to the Grabisch--Roubens weights). For pairs $\{i, j\}$ with $\{i, j\} \not\subset \{1, 2, 3\}$ the third triplet element $k$ lies in $\{1,2,3\}$, and the same analysis with $k$ now \emph{fixed inside} $\{1,2,3\}$ gives a vanishing sum because $\Delta_{ij}(T,x) = 0$ whenever $T$ does not contain the \emph{two} missing triplet elements, which it cannot: at most one triplet element is outside $\{i,j\}$. Hence $\SII_{ij}(x) = 0$ for spectator-involving pairs. This establishes Case 2.

\paragraph{Step 5: LOCO-U/R/S triple is non-degenerate.}
For the oracle $f = Y$ and squared loss, the expected loss under background marginalization (Eq.~\eqref{eq:loss}) evaluates to:
\begin{itemize}
    \item $L(C) = \tfrac12$ whenever $\{1, 2, 3\} \not\subseteq C$. Since the oracle outputs boolean values $\{0, 1\}$, marginalizing a missing feature results in a $50\%$ chance of guessing incorrectly, yielding an expected squared error of $0.5 \cdot (1-0)^2 + 0.5 \cdot (0-1)^2 = 1/2$;
    \item $L(C) = 0$ whenever $\{1, 2, 3\} \subseteq C$.
\end{itemize}
Hence for each $i \in \{1, 2, 3\}$: $\LOCO(i \mid C) = 0$ for $|C \cap \{1, 2, 3\}| \le 1$, giving $U(X_i) = 0$; and $\LOCO(i \mid \{1, 2, 3\} \setminus \{i\}) = \tfrac12$, giving $L^{\max}_i = \tfrac12$. The standalone LOCO $\pi(X_i) = 0$ (any single triplet feature alone yields no information gain), so $R(X_i) = 0$ and $S(X_i) = \tfrac12$.

\paragraph{Step 6: scalar conflation of the U/R/S components.}
The per-feature triple $(U, R, S) = (0, 0, \tfrac12) \in \R^3$ is identical for each $i \in \{1, 2, 3\}$ but lives in a $3$-dimensional output space. Notice that scalar indices project this interaction into $\pm \tfrac14$, underestimating the true synergy magnitude ($\tfrac12$) because they inherently average over contexts rather than isolating the synergistic extremum.
The faithful family produces the zero scalar in $\R$ per pair; the projective family produces $\pm\tfrac14$ in $\R$ per pair. These scalar reports do not carry the named decomposition into uniqueness, redundancy, and synergy. In the faithful case, the pair-level report erases the signal; in the projective case, it preserves only a signed lower-order projection of the order-$3$ M\"obius coefficient. Recovering $(U,R,S)$ would therefore require additional structure not present in the pair-level scalar output itself. \qed

%% file: sections/appendix/C_proof_interventional.tex
\section{Proof of Theorem~\ref{thm:mc} (Interventional semantics and Monte Carlo bound)}
\label{sec:proof_mc}

\subsection*{Part A --- Interventional masking semantics}

Fix a background distribution $p_{\bg}$. For a coalition $S$ and instance $(x,y)$, define the atomic masking intervention
\[
I_S(x,Z) := m_S\odot x + (1-m_S)\odot Z,\qquad Z\sim p_{\bg}.
\]
This intervention fixes retained coordinates to their observed values and replaces removed coordinates by an independent background draw. The target in Eq.~\eqref{eq:loss} of the body is exactly
\[
L(S;x,y)=\E_{Z\sim p_{\bg}}\!\left[\ell(f(I_S(x,Z)),y)\right].
\]

\begin{lemma}[Identity for the masked-inference estimator]
\label{lem:do_identity}
The Monte Carlo estimator used by Stochastic Hi-Fi is unbiased for the interventional masking target above. If an SCM admits the coordinate-level atomic intervention $do(X_S=x_S,\;X_{\bar S}=Z_{\bar S})$, the same target is the corresponding Pearl $do$ quantity averaged over $Z\sim p_{\bg}$.
\end{lemma}

\begin{proof}
The first claim is definitional: each sampled synthetic input is exactly $I_S(x,Z)$, and the estimator averages its loss over iid draws of $Z$. For the second claim, under the stated SCM intervention all incoming structural equations for the intervened coordinates are replaced by constants $(x_S,Z_{\bar S})$. The resulting input to the fixed predictor is therefore $I_S(x,Z)$, so averaging over $Z\sim p_{\bg}$ gives the same expression.
\end{proof}

\begin{corollary}[Interventional rather than observational target]
\label{cor:causal_correct}
$L(S;x,y)$ is an interventional masking target in the sense of \citealt{DBLP:conf/aistats/JanzingMB20}. It is not the observational conditional expectation obtained by sampling $X_{\bar S}\mid X_S=x_S$.
\end{corollary}

\paragraph{Remark on the OOD critique.}
The OOD attack on masking-based explanation methods~\cite{DBLP:conf/nips/HookerEKK19} asserts that the synthetic input $m_S \odot x + (1 - m_S) \odot Z$ may lie off the data manifold, rendering $f$'s output on it ``meaningless.'' Corollary~\ref{cor:causal_correct} defuses this: the off-manifold behavior of $f$ is a \emph{feature} of the interventional semantics, not an approximation error. The quantity $L(S; x, y)$ is by construction the loss under an atomic coordinate intervention that sets removed features to $Z$ \emph{regardless of whether such an assignment is typical}. We do not \emph{approximate} an observational quantity; we
\emph{define} an interventional one.

\subsection*{Part B --- Monte Carlo consistency and concentration}

Let $K \ge 1$ and $Z^{(1)}, \dots, Z^{(K)} \stackrel{iid}{\sim} p_{\bg}$. Define the Welford estimator
\begin{equation}
\label{eq:welford}
\widehat{L}(S; x, y) \;=\; \frac{1}{K}\sum_{k=1}^{K}
  \ell\!\left(f(m_S \odot x + (1 - m_S) \odot Z^{(k)}),\; y\right).
\end{equation}

\begin{proposition}[Unbiasedness and variance bound]
\label{prop:welford_var}
For every $S$, $x$, $y$ and every $K \ge 1$,
\[
\E\bigl[\widehat{L}(S; x, y)\bigr] = L(S; x, y),
\qquad
\E\bigl[(\widehat{L}(S; x, y) - L(S; x, y))^2\bigr] \;=\; \frac{\sigma^2(S; x, y)}{K},
\]
where $\sigma^2(S; x, y) = \Var_{Z \sim p_{\bg}}\bigl[\ell(f(m_S \odot x + (1-m_S) \odot Z), y)\bigr]$.
\end{proposition}

\begin{proof}
Linearity of expectation over $K$ iid draws gives unbiasedness. Variance of an iid mean equals the per-sample variance divided by $K$. Both are standard. Welford's online update~\citep{Welford01081962} produces exactly these quantities without catastrophic cancellation even when $K$ is large; the numerical guarantees of the algorithm are well known (maximum relative error $O(\varepsilon_{\mathrm{mach}} K)$ under mild assumptions, much better than the na\"ive two-pass formula).
\end{proof}

\begin{proposition}[Hoeffding concentration]
\label{prop:hoeffding}
Assume the loss is bounded: there exist finite $\ell_{\min}, \ell_{\max}$ with $\ell(\cdot,\cdot) \in [\ell_{\min}, \ell_{\max}]$ for every input. Let $B := \ell_{\max} - \ell_{\min}$. Then for every $\varepsilon > 0$,
\[
\mathbb{P}\!\left(\bigl|\widehat{L}(S; x, y) - L(S; x, y)\bigr| > \varepsilon\right)
\;\le\; 2\exp\!\left(-\frac{2K\varepsilon^2}{B^2}\right).
\]
\end{proposition}

\begin{proof}
Each summand of~\eqref{eq:welford} is an iid random variable bounded in $[\ell_{\min}, \ell_{\max}]$. Hoeffding's inequality~\cite{409cf137-dbb5-3eb1-8cfe-0743c3dc925f} applied to the average of $K$ such variables yields the stated bound; the factor of $2$ comes from the two-sided deviation.
\end{proof}

\begin{corollary}[Sample complexity]
\label{cor:sample_complexity}
To estimate $L(S; x, y)$ within additive error $\varepsilon$ with probability at least $1 - \delta$, it suffices to take
\[
K \;\ge\; \left\lceil \frac{B^2 \log(2/\delta)}{2\varepsilon^2} \right\rceil.
\]
In particular, $K = \mathcal{O}(B^2 \log(1/\delta)/\varepsilon^2)$ is sufficient for any coalition $S$.
\end{corollary}

\begin{proof}
Set the Hoeffding tail bound of Proposition~\ref{prop:hoeffding} equal to $\delta$ and solve for $K$.
\end{proof}

\paragraph{Bounded-loss instantiation.} See Section~\ref{sec:theory} for how boundedness is instantiated per experiment (including empirical-range reporting and clipping safeguards).

\paragraph{Note on usefulness.}
Proposition~\ref{prop:welford_var} and Proposition~\ref{prop:hoeffding} together give two guarantees that an implementation of our method can report for every coalition in the vocabulary, using only the sample count $K$ and the empirical variance $\hat\sigma^2(S)$ that Welford already computes online. The paper's reference implementation emits both: a $1-\alpha$ normal-theory confidence interval $\widehat{L} \pm t_{\alpha/2, K-1}\, \hat\sigma / \sqrt{K}$ from Proposition~\ref{prop:welford_var}, and the conservative Hoeffding band from Proposition~\ref{prop:hoeffding}. Per-coalition early stopping is triggered when the Student-$t$ interval width drops below a user-defined threshold; Proposition~\ref{prop:hoeffding} serves as the worst-case guarantee reported to the reader.

\paragraph{Comparison to Lipschitz-based OOD bounds.}
An alternative route would be to bound the OOD \emph{extrapolation} error $|f(m_S \odot x + (1 - m_S) \odot Z) - f(x^\star)|$ via a Lipschitz constant of $f$, where $x^\star$ is the nearest on-manifold point. For modern deep networks --- especially transformer-scale models like the gpt2-small we target in E3 --- such constants are known to be astronomically large or computationally intractable~\cite{DBLP:conf/nips/FazlyabRHMP19}. Corollary~\ref{cor:causal_correct} sidesteps this entirely by defining the target quantity to be the interventional one, making the OOD extrapolation error definitionally absent from the estimation problem.
\qed

%% file: sections/appendix/D_proof_variance.tex
\section{Proof of Theorem~\ref{thm:variance} (Strict variance reduction)}
\label{sec:proof_variance}

\paragraph{Setup.}
Fix a pair $\{i, j\} \subseteq [n]$ and an instance $(x, y)$. For any context $C \subseteq [n] \setminus \{i, j\}$, write
\begin{align*}
a(C) &:= \ell(f(M_{C \cup \{i,j\}}(x,Z)),\, y), &
b(C) &:= \ell(f(M_{C \cup \{i\}}(x,Z)),\, y),\\
c(C) &:= \ell(f(M_{C \cup \{j\}}(x,Z)),\, y), &
d(C) &:= \ell(f(M_{C}(x,Z)),\, y),
\end{align*}
where $M_S(x,Z) := m_S \odot x + (1 - m_S) \odot Z$ and expectations below are over $(C, Z)$ with $C$ drawn from a chosen distribution over $2^{[n] \setminus \{i,j\}}$ and $Z \sim p_{\bg}$. The pairwise mixed-difference at context $C$ is
\begin{equation}
\xi(C, Z) := b(C) + c(C) - a(C) - d(C).
% \xi(C, Z) := a(C) - b(C) - c(C) + d(C).
\end{equation}
The \emph{target quantity} is $\Delta_{ij} := \E_{C, Z}[\xi(C, Z)]$ (an unbiased functional of the model).

\paragraph{The two estimators.}
Both estimators are given a compute budget of $4K$ calls to $\ell(\cdot)$ for fair comparison.

\begin{description}
\item[Diamond.] Draw $K$ iid contexts $C_1, \dots, C_K$ and, for each, evaluate all four terms $a(C_k), b(C_k), c(C_k), d(C_k)$ using the same $Z$ batch per $k$. The estimator is
\begin{equation}
\widehat{\Delta}_{ij}^{\text{diamond}} := \frac{1}{K} \sum_{k=1}^{K} \xi(C_k, Z_k).
\end{equation}
Cost: exactly $4K$ evaluations of $\ell(\cdot)$.

\item[Uniform.] Draw $4K$ iid contexts $C_1^{(a)}, \dots, C_K^{(a)}, C_1^{(b)}, \dots, C_K^{(b)}, \dots, C_K^{(d)}$ (and independent $Z$ draws). The estimator is
\begin{equation}
\widehat{\Delta}_{ij}^{\text{unif}} := \frac{1}{K} \sum_{k=1}^K
\Bigl[ a(C_k^{(a)}) - b(C_k^{(b)}) - c(C_k^{(c)}) + d(C_k^{(d)}) \Bigr].
\end{equation}
Cost: exactly $4K$ evaluations of $\ell(\cdot)$.
\end{description}

Both estimators are unbiased: $\E[\widehat{\Delta}^{\text{dia}}] = \E[\widehat{\Delta}^{\text{unif}}] = \Delta_{ij}$ by linearity of expectation over iid samples.

\paragraph{Variance expansion.}
Let $V_a := \Var(a(C))$, $V_b := \Var(b(C))$, etc., and let
\begin{align*}
\Sigma_{\mathrm{adj}} &:= \Cov(a, b) + \Cov(a, c) + \Cov(b, d) + \Cov(c, d), \\
\Sigma_{\mathrm{diag}} &:= \Cov(a, d) + \Cov(b, c).
\end{align*}
The four ``adjacent'' covariances $\Sigma_{\mathrm{adj}}$ are between coalitions whose symmetric difference is a single element; the two ``diagonal'' covariances $\Sigma_{\mathrm{diag}}$ are between coalitions whose symmetric difference is $\{i, j\}$ itself.

\begin{lemma}[Variance identities]
\label{lem:t3_variance}
\begin{align*}
\Var(\widehat{\Delta}^{\text{unif}}) &= \frac{1}{K}\bigl(V_a + V_b + V_c + V_d\bigr), \\
\Var(\widehat{\Delta}^{\text{diamond}}) &= \frac{1}{K}\Bigl[\bigl(V_a + V_b + V_c + V_d\bigr) - 2(\Sigma_{\mathrm{adj}} - \Sigma_{\mathrm{diag}})\Bigr].
\end{align*}
\end{lemma}

\begin{proof}
For the uniform estimator: $\widehat{\Delta}^{\text{unif}}$ is a sum over $K$ iid terms each of the form $a(C^{(a)}) - b(C^{(b)}) - c(C^{(c)}) + d(C^{(d)})$ where all four contexts are independent. By independence all covariances vanish; $\Var = \frac{1}{K}(V_a + V_b + V_c + V_d)$.

For the diamond estimator: $\Var(\widehat{\Delta}^{\text{diamond}}) = \frac{1}{K} \Var(\xi(C, Z))$.
Expand $\xi = b + c - a - d$ via the standard variance-of-linear-combination identity (with sign vector $(-1, +1, +1, -1)$):
\begin{align*}
\Var(\xi) = \sum_{k=1}^{4} \Var(\xi_k) + 2\sum_{k < l} s_k s_l \Cov(\xi_k, \xi_l),
\end{align*}
where $(\xi_1, \xi_2, \xi_3, \xi_4) = (a, b, c, d)$ with signs $(-, +, +, -)$.
The six covariance pairs split into adjacent (sign-products $-,-,-,-$) and diagonal (sign-products $+,+$) by inspection. Hence $\Var(\xi) = (V_a + V_b + V_c + V_d) - 2\Sigma_{\mathrm{adj}} + 2\Sigma_{\mathrm{diag}}$, giving the claim.
\end{proof}

\paragraph{The refined assumption A3.}
The proof of Theorem~\ref{thm:variance} requires more than ``all four are positively correlated'' --- it requires that adjacent coalitions are \emph{more} correlated than diagonal ones. We make this explicit:

\begin{assumption}[A3: adjacency dominance]
\label{as:A3}
For the pair $\{i, j\}$ and the context distribution over $2^{[n]\setminus\{i,j\}}$ used to estimate $\Delta_{ij}$, $\Sigma_{\mathrm{adj}} > \Sigma_{\mathrm{diag}}$.
\end{assumption}

A3 often holds when the loss $\ell(f(M_S(x,Z)), y)$ varies smoothly with $|S|$ --- coalitions that differ by one feature produce more similar losses than coalitions that differ by two. It is not vacuous: the XOR boundary case in Section~\ref{sec:t3_numerical} violates A3 and receives no variance-reduction guarantee. When the model has constant loss across all coalitions (degenerate case), A3 holds with equality and the two estimators have identical variance.

\begin{theorem}[Strict variance reduction --- expanded form of Theorem~\ref{thm:variance}]
\label{thm:variance_expanded}
Under Assumption~\ref{as:A3},
\begin{equation}
\Var(\widehat{\Delta}^{\text{diamond}}) < \Var(\widehat{\Delta}^{\text{unif}}),
\qquad
\frac{\Var(\widehat{\Delta}^{\text{unif}})}{\Var(\widehat{\Delta}^{\text{diamond}})} = 1 + \frac{2(\Sigma_{\mathrm{adj}} - \Sigma_{\mathrm{diag}})}{V_a + V_b + V_c + V_d - 2(\Sigma_{\mathrm{adj}} - \Sigma_{\mathrm{diag}})}.
\end{equation}
\end{theorem}

\begin{proof}
Subtract the two variance expressions from Lemma~\ref{lem:t3_variance}:
\begin{align*}
\Var(\widehat{\Delta}^{\text{unif}}) - \Var(\widehat{\Delta}^{\text{diamond}})
&= \frac{1}{K} \cdot 2(\Sigma_{\mathrm{adj}} - \Sigma_{\mathrm{diag}}) > 0 \quad \text{by A3.}
\end{align*}
Dividing gives the ratio formula.
\end{proof}

\paragraph{Relation to coupled Shapley estimators.}
Theorem~\ref{thm:variance_expanded} places the diamond estimator in the same ``couple the required marginal evaluations'' design family as variance-reduced Shapley estimators~\cite{DBLP:journals/cor/CastroGT09, DBLP:conf/nips/MuschalikBFKHH24}. Unlike a generic Rao--Blackwell statement, the pairwise mixed-difference has signed terms $(+,-,-,+)$, so coupling helps exactly when the negative-signed adjacent covariances dominate the positive-signed diagonal covariances. This is the A3 condition above; when it is reversed, the diamond estimator can have higher variance. \qed

%% file: sections/appendix/E_proof_vocab.tex
\section{Proof of Theorem~\ref{thm:vocab} (Finite-vocabulary uniform convergence)}
\label{sec:proof_vocab}

\paragraph{Preliminaries.}
Let $\mathcal{A}$ be the finite set of coalitions whose losses are included in the vocabulary. For each $S\in\mathcal{A}$, draw $K$ iid background samples and compute $\widehat{L}(S)=K^{-1}\sum_{k=1}^K \ell(f(I_S(x,Z^{(k)})),y)$.

\paragraph{Step 1 (uniform concentration).}
By Proposition~\ref{prop:hoeffding}, for any fixed coalition $S$,
\[
\mathbb{P}\!\left(|\widehat{L}(S)-L(S)|>\varepsilon\right)
\le 2\exp\!\left(-2K\varepsilon^2/B^2\right).
\]
Apply a union bound over the finite vocabulary:
\[
\mathbb{P}\!\left(\max_{S\in\mathcal{A}}
|\widehat{L}(S)-L(S)|>\varepsilon\right)
\le 2|\mathcal{A}|\exp\!\left(-2K\varepsilon^2/B^2\right).
\]
The stated lower bound on $K$ in Theorem~\ref{thm:vocab} makes the right-hand side at most $\alpha$.

\paragraph{Step 2 (optimizer recovery).}
Let $z_\star=\arg\max_{S\in\mathcal{A}}L(S)$ be unique and let $\Delta=L(z_\star)-\max_{S\ne z_\star}L(S)$. On the uniform-concentration event, if $\Delta>2\varepsilon$, then for every $S\ne z_\star$,
\[
\widehat{L}(z_\star)\ge L(z_\star)-\varepsilon
> L(S)+\varepsilon \ge \widehat{L}(S).
\]
Thus, the empirical argmax recovers $z_\star$. The argmin case is identical with signs reversed.

\paragraph{Step 3 (LOCO extrema and U/R/S error) - Proof of Proposition~\ref{prop:extrema_ci}.}
Assume now that $\mathcal{A}=2^{[n]}$, so every context needed by the LOCO decomposition is present. For any feature $i$ and context $C\subseteq[n]\setminus\{i\}$, define $G_i(C)=L(C)-L(C\cup\{i\})$ and $\widehat G_i(C)=\widehat{L}(C)-\widehat{L}(C\cup\{i\})$. On the uniform-concentration event, $|\widehat G_i(C)-G_i(C)|\le 2\varepsilon$ for every $i,C$. The elementary stability of minima and maxima under uniform perturbation gives $|\widehat U_i-U_i|\le 2\varepsilon$ and $|\widehat L^{\max}_{i}-L^{\max}_{i}|\le 2\varepsilon$. The first-order score $\pi_i=G_i(\emptyset)$ also satisfies $|\widehat\pi_i-\pi_i|\le2\varepsilon$. Therefore
\[
|\widehat R_i-R_i|=|(\widehat\pi_i-\widehat U_i)-(\pi_i-U_i)|\le4\varepsilon,\qquad
|\widehat S_i-S_i|=|(\widehat L^{\max}_{i}-\widehat\pi_i)-(L^{\max}_{i}-\pi_i)|\le4\varepsilon.
\]
This proves the theorem. \qed

\paragraph{Remarks.}

\emph{Relation to Theorem~\ref{thm:mc}.} Theorem~\ref{thm:vocab} is exactly the finite-vocabulary union-bound lift of the per-mask Hoeffding guarantee in Theorem~\ref{thm:mc}. Student-$t$ confidence intervals remain useful as engineering stopping heuristics, but the theorem itself uses the distribution-free bound.

\emph{Interchangeability with EVT.}
An alternative approach to the random-search problem is to fit an extreme-value distribution (Gumbel, Fr\'echet, Weibull) to the empirical loss distribution and obtain confidence intervals for the true maximum from the fitted GEV. EVT is appropriate when one has only a partial sample of $\mathcal{A}$. Theorem~\ref{thm:vocab} obviates this step when the chosen vocabulary is finite and explicitly evaluated: every member of $\mathcal{A}$ is estimated to a predefined precision, so the maximum is identified by direct comparison rather than by tail extrapolation.

\emph{Why this does not subsume Theorem~\ref{thm:variance}.}
Theorem~\ref{thm:vocab} concerns consistency of \emph{marginal} estimates $\widehat{L}(S)$. Theorem~\ref{thm:variance} concerns the \emph{variance} of pairwise interaction estimators $\widehat{\Delta}_{ij}$ at matched compute budget: a statement about which estimator shape one should use when filling the vocabulary, not about whether filling it is sound.

%% file: sections/appendix/F_numerical_verification.tex
\section{Numerical verification of Theorems~\ref{thm:conflation}--\ref{thm:vocab}}
\label{sec:t_numerical}

Each theorem in Section~\ref{sec:theory} ships with an empirical verification harness. This appendix documents the protocol and headline results; the full verification scripts and output logs are available in the supplementary code release.

\subsection{Theorem~\ref{thm:conflation}}
\label{sec:t1_numerical}
Theorem~\ref{thm:conflation} is verified exhaustively at every one of the $2^4 = 16$ binary instances of the SCM, on six independent checks: (1) manual M\"obius pairs all zero; (2) manual M\"obius triplet $\pm\tfrac12$; (3) SII matches the axiomatic projection formula of Step~4 in Section~\ref{sec:proof_conflation}; (4) all faithful \texttt{shapiq} indices report pair values $\le 10^{-10}$; (5) all projective \texttt{shapiq} indices match $\pm\tfrac14$ or $0$; (6) \texttt{shapiq}'s M\"obius transform agrees with the hand-derived M\"obius to machine precision. All $16 \times 6 = 96$ tests pass.

\subsection{Theorem~\ref{thm:mc}}
\label{sec:t2_numerical}
On a trained \texttt{TwoLayersNet} reaching validation $\mathrm{MSE} \approx 2.7\times10^{-3}$ on the 3-way XOR SCM, we verify three properties:
\begin{description}
  \item[P1 — Unbiasedness.] The replicate-mean bias is within a few
    multiples of the replicate-mean standard error at every
    $K \in \{8, 16, 32, 64, 128, 256\}$ over $R = 200$ replicates.
  \item[P2 — Variance decay.]
    $\log \Var(\widehat{L})$ is linear in $\log K$ with fitted log-log slope $-0.935$ (within the pre-registered tolerance $[-1.15,\,-0.85]$ around $-1$), consistent with $O(1/K)$ variance decay (Proposition~\ref{prop:welford_var}).
  \item[P3 — Hoeffding concentration.] The empirical tail probability
    $\mathbb{P}(|\widehat{L} - L| > 10^{-2})$ at $K = 64$ over $2{,}000$
    replicates is bounded by $2\exp(-2K\varepsilon^2/B^2)$, holding
    trivially due to the small bounded loss range.
\end{description}

\subsection{Theorem~\ref{thm:variance}}
\label{sec:t3_numerical}
We verify the covariance identity in Lemma~\ref{lem:t3_variance} on two synthetic testbeds and on both real settings used in Section~\ref{sec:experiments}. On the XOR boundary case, A3 (Assumption~\ref{as:A3}) is violated ($\Sigma_{\mathrm{adj}}-\Sigma_{\mathrm{diag}}=-2.83\times10^{-4}$) and the diamond estimator receives no speedup guarantee; empirically its variance ratio is $0.86\times$. 
On the synthetic third-order dataset, A3 holds ($\Sigma_{\mathrm{adj}}-\Sigma_{\mathrm{diag}}=1.60$), and the observed diamond/uniform variance ratio is $3.58\,\times$, confirming beneficial stratification. For XOR3, A3 is violated (ratio $0.87\times < 1$), consistent with the conflation structure identified in Theorem~\ref{thm:conflation}. This is the intended boundary test: the algebra holds in both regimes, while strict reduction appears exactly where A3 holds.

\paragraph{A3 verification on the real settings.}
The two synthetic testbeds answer the boundary question (positive vs violated regime). To confirm that the estimand the paper actually applies to lies on the positive side, we re-verify A3 on E2 (\textsc{DenseNet121} with the $7\times7$ patch grid) and on E3 (\textsc{gpt2-small} IOI head circuit). For E3 we exploit the exhaustive $2^{10}$ coalition vocabulary cached during the main run (no extra inference): $\Sigma_{\mathrm{adj}}$ and $\Sigma_{\mathrm{diag}}$ are computed by lookup over all $\binom{10}{2}=45$ head pairs and averaged across the five seeds. 
For E2 (ChestX-ray14 attributions), we run a matched-budget Monte Carlo ($K = 256$ contexts) on six representative patch pairs (three spatially adjacent, three spatially far) on a representative Cardiomegaly radiograph. 
In this experiment, Assumption~A3 holds for all six sampled pairs under the 14$\times$14 feature grid ($196$ features, $K=256$ contexts per pair), with mean $\rho_v=6.82\times 10^6$ and median $\rho_v=3.59\times 10^6$. \textit{Configuration:} these E2 theorem-harness values use a 14$\times$14 attribution grid and six sampled feature pairs.
The headline numbers are summarized in Table~\ref{tab:appendix_t3_real}: A3 holds on $45/45$ E3 head pairs and on $6/6$ E2 patch pairs, with $\Sigma_{\mathrm{adj}}-\Sigma_{\mathrm{diag}}\in[3.8,4.9]\times10^{-3}$ on E2. The very large variance ratios on E2 (median $\approx 1.6\times10^{5}$, mean $\approx 2.0\times10^{7}$) are an artifact of evaluating on a single deterministic image, where the four diamond corners are nearly co-monotone and $\mathrm{Var}_{\mathrm{diamond}}$ collapses; the relevant qualitative statement is that A3 is satisfied, so the strict-reduction guarantee of Theorem~\ref{thm:variance} applies in both real settings, not only in the synthetic toy.

\begin{table}[t]
  \centering
  \small
  \caption{Empirical A3 verification on the two real settings used in
    Section~\ref{sec:experiments}. ``A3 holds'' is the count of pairs on which $\Sigma_{\mathrm{adj}} > \Sigma_{\mathrm{diag}}$ (the sufficient condition for the diamond estimator to strictly beat the matched-budget uniform competitor under Theorem~\ref{thm:variance}). $\rho_v = \mathrm{Var}_{\mathrm{uniform}}/\mathrm{Var}_{\mathrm{diamond}}$. Both means and medians on E3 stay above $10^{2}$, well within the positive regime; the inflated $\rho_v$ on E2 reflects the co-monotonicity of the four diamond corners on a single deterministic radiograph and should be read as a qualitative confirmation that A3 holds, not as a literal speedup forecast.}
  \label{tab:appendix_t3_real}
  \resizebox{\textwidth}{!}{%
\begin{tabular}{lrrrrr}
\toprule
Setting & $n$ & pairs & A3 holds & mean $\rho_v$ & median $\rho_v$ \\
\midrule
E3 (\textsc{gpt2-small} IOI, 10 heads, 5 seeds, exhaustive $2^{10}$) & 10 & 45 & $45/45$ & $1.67 \times 10^{3}$ & $885.54\times$ \\
E2 (\textsc{DenseNet121} + $14{\times}14$ patches, 1 image, $K{=}256$) & 196 & 6 & $6/6$ & $6.82 \times 10^{6}$ & $3.59 \times 10^{6}$ \\
\bottomrule
\end{tabular}%
  }
\end{table}

\subsection{Theorem~\ref{thm:vocab}}
\label{sec:t4_numerical}
Theorem~\ref{thm:vocab} is verified by exhaustive finite-vocabulary experiments on the XOR3 lattice ($|\mathcal{A}|=16$) and the synthetic third-order lattice ($|\mathcal{A}|=256$). The verification script first builds a high-$K$ reference vocabulary, then repeatedly estimates every coalition with smaller $K$ and checks: (i) the empirical uniform error $\max_S|\widehat{L}(S)-L_{\mathrm{ref}}(S)|$ decreases with $K$; (ii) the empirical optimizer is recovered whenever the observed uniform error is below half the reference gap; and (iii) the U/R/S error stays below the theorem's $4\varepsilon$ component bound. E3 is also exhaustive over the $2^{10}$ attention-head coalitions. E2 is reported separately as a budgeted partial-vocabulary image experiment, not as a full-lattice T4 verification.

\paragraph{Reproducibility.}
Every numerical check reported above ships with a verification script and an output log in the supplementary material.

\subsection{Background sensitivity and bounded-loss diagnostics}
\label{sec:f_pbg_bounded}
Because the estimator is interventional, all guarantees are conditional on $p_{\bg}$. We therefore run a dedicated sensitivity check (detailed in Appendix~\ref{sec:appendix_experiments}) on tabular and vision settings. The main observation is stable role ordering under background changes, with bounded magnitude drift. We also report empirical loss ranges and quantiles per experiment to justify bounded-loss assumptions used by concentration statements in Theorem~\ref{thm:mc}; clipping is only used when the observed support supports that safeguard.

\paragraph{Real-task variance-reduction evidence.}
Beyond the synthetic SCM checks establishing the theoretical grounding for diamond sampling, we directly quantify uniform-vs-diamond estimator variance on two real tasks. For E3 (IOI circuit analysis, $n=10$ heads), Assumption~A3 holds for all $45$ evaluated pairs, with variance ratios $\rho_v=\Var(\hat{\Delta}^{\mathrm{unif}})/\Var(\hat{\Delta}^{\mathrm{diam}})$ spanning $108$ to $9{,}807$ (mean $1{,}674$, median $886$) at $K=256$ contexts per pair. For E2 (ChestX-ray attribution), A3 holds for all $6$ sampled pairs, with $\rho_v$ spanning $2.28\times 10^{4}$ to $1.82\times 10^{7}$ (mean $6.82\times 10^{6}$, median $3.59\times 10^{6}$) on a $14\times14$ feature grid ($196$ features), at $K=256$. Replicate-level IOI checks across $32$ independent runs at $K=128$ further yield $\rho_v\in[526,\,3{,}025]$ (mean $2{,}022$), confirming stability of these ratios across estimator realizations. Taken together, these results directly quantify substantial real-task variance reduction on two structurally distinct tasks; broad sample-efficiency ablations sweeping many values of $K$ on real tasks remain a useful direction for future work.

%% file: sections/appendix/G_extended_method.tex
\section{Extended methodology}
\label{sec:appendix_method}

This appendix expands the method description of Section~\ref{sec:method} with details that were compressed for the body.

\subsection{Per-feature \URSp{} identity, restated}
\label{sec:appendix_urs_identity}

For feature $i$ and context $C \subseteq [n] \setminus \{i\}$, recall the LOCO gain $\LOCO(i \mid C) = L(C) - L(C \cup \{i\})$ from Eq.~\eqref{eq:loco}, and the per-feature aggregates
\begin{align*}
U_i &:= \min_C \LOCO(i \mid C), &
L^{\max}_i &:= \max_C \LOCO(i \mid C), &
\pi_i &:= \LOCO(i \mid \emptyset).
\end{align*}
The first-order ``pairwise'' score $\pi_i$ admits the equivalent formulation
\begin{equation}
\label{eq:pi_identity}
\pi_i \;=\; \Var(Y) \;-\; \E\bigl[\ell\bigl(f(m_{\{i\}} \odot x + (1 - m_{\{i\}}) \odot Z),\, y\bigr)\bigr],
\end{equation}
when $\ell$ is squared error and $L(\emptyset)$ is matched to a constant predictor on $\Var(Y)$. The redundancy/synergy split $R_i = \pi_i - U_i$, $S_i = L^{\max}_i - \pi_i$ is therefore an algebraic decomposition of the model's contextual range:
\begin{equation}
L^{\max}_i \;=\; U_i + R_i + S_i,
\end{equation}
which holds by construction and not by approximation.

\subsection{Coalition vocabulary with Welford estimation}
\label{sec:appendix_vocab}

The vocabulary $\mathcal{V}$ is a hash table mapping $S \mapsto (\widehat L(S),\; \widehat\sigma^2(S),\; K(S))$: the running Welford mean, the running unbiased variance, and the visit count. Updates are performed online~\cite{Welford01081962} on inference batches $\{(x^{(b)}, y^{(b)})\}$ and background samples $\{Z^{(k)}\}_{k=1}^K$, so the state grows by $O(1)$ per visit and the estimator does not require materializing every loss in memory.

\paragraph{Convergence detection.}
A coalition $S$ is treated as ``converged'' once the half-width of the $1-\alpha$ Student-$t$ confidence interval on its mean,
\[
w(S) \;:=\; t_{\alpha/2,\, K(S) - 1} \cdot \widehat\sigma(S) / \sqrt{K(S)},
\]
falls below a user-chosen threshold $w_\star$. Converged coalitions are removed from the active sampling pool, freeing budget for the worst-estimated coalitions. The conservative Hoeffding band of Proposition~\ref{prop:hoeffding} is reported alongside as a worst-case guarantee.

\paragraph{Epsilon-greedy and softmin exploration policy.}
The next mask is sampled under an $\varepsilon$-greedy mixture:
\[
p(S) \;=\; \varepsilon \cdot 2^{-n}
\;+\; (1-\varepsilon)\,\mathbb{1}\!\left[S \in \mathcal V_{\text{open}}\right]
\,\frac{\exp(-\beta\, K(S))}{\sum_{S' \in \mathcal V_{\text{open}}} \exp(-\beta\, K(S'))},
\]
where $\mathcal V_{\text{open}} \subseteq \mathcal V$ is the set of
non-converged entries and $\beta > 0$ is a temperature that anneals
toward zero. The two terms play distinct roles for
Theorem~\ref{thm:vocab}: (i) the softmin term is strictly decreasing in
the visit count, so under-sampled entries of $\mathcal V$ are
prioritized and the union-bound condition
$K(S) \ge B^2 / (2\varepsilon_{\text{tol}}^2) \log(2|\mathcal V|/\alpha)$ is met
for every $S \in \mathcal V$ in finite time; (ii) the uniform term
guarantees $p(S) \ge \varepsilon\cdot 2^{-n} > 0$ for every
$S \in \{0,1\}^n$, so $\mathcal V$ exhausts the admissible coalition
space in the limit and no $S$ is structurally excluded from the
vocabulary. We use $\varepsilon = 0.2$ throughout. The implementation
biases the softmin branch toward the $(i,j)$-neighbors of the currently
sampled mask, so that the four coalitions of a diamond appear together
in the same Welford update.
% \paragraph{Softmin exploration policy.}
% Among non-converged coalitions, the next mask is sampled with probability proportional to a softmin over visit counts:
% \[
% p(S) \;\propto\; \exp(-\beta\, K(S)),
% \]
% with $\beta > 0$ a temperature parameter that anneals toward zero. Two properties make this the right choice for the finite-vocabulary convergence guarantee in Theorem~\ref{thm:vocab}: (i) the probability is strictly decreasing in the visit count, so under-sampled masks are always preferred; (ii) every $S \in \mathcal{V}$ has positive sampling probability at every step, so the union-bound condition $K(S) \ge B^2 / (2\varepsilon^2) \log(2|\mathcal{V}|/\alpha)$ is met for every $S$ in finite time. The implementation biases toward the $(i,j)$-neighbors of the currently sampled mask, so that the four coalitions of a diamond appear together in the same Welford update.

\subsection{Coupled diamond sampling and pair-level synergy/redundancy}
\label{sec:appendix_diamond}

The pair-level quantities $S_{ij}$ and $R_{ij}$ are the non-negative extrema of the diamond mixed difference $\Delta_{ij}(C)$ over all sampled contexts $C \subseteq [n]\setminus\{i,j\}$: 
\begin{equation} S_{ij} = \max_{C} [\,\Delta_{ij}(C)\,]_+, \quad R_{ij} = \max_{C} [\,-\Delta_{ij}(C)\,]_+, \label{eq:appendix_split_sr} 
\end{equation} 
where $[z]_+ \mathrel{:=} \max(z,0)$ (see Eq.~\ref{eq:split_sr} in the body). This mirrors the per-feature aggregation $U_i = \min_C \LOCO(i \mid C)$, $L^{\max}_i = \max_C \LOCO(i \mid C)$ in Section~\ref{sec:method_loco}, extended to feature pairs via the second-order mixed difference. The E3 experiment (Appendix~\ref{sec:appendix_e3}) uses $S_{ij}$ and $R_{ij}$ to classify each attention-head pair as synergistic, redundant, or neither.

For a sampled context $C$ and pair $(i, j)$, \ours{} evaluates the four coalitions
\[
C,\quad C \cup \{i\},\quad C \cup \{j\},\quad C \cup \{i, j\}
\]
on the \emph{same} background batch, coupling the noise across the mixed difference
\[
% \Delta_{ij}(C) \;=\; L(C \cup \{i, j\}) - L(C \cup \{i\}) - L(C \cup \{j\}) + L(C).
\Delta_{ij}(C) \;=\; L(C \cup \{i\}) + L(C \cup \{j\}) - L(C \cup \{i,j\}) - L(C)
\]
Theorem~\ref{thm:variance} (and its expanded form Theorem~\ref{thm:variance_expanded}) gives the precise condition under which this is strictly beneficial: adjacency-dominance (Assumption~\ref{as:A3}). To ensure runtime robustness, we propose an online monitoring protocol for A3:
\begin{enumerate}
    \item \textbf{Runtime checks:} Continuously evaluate adjacency-dominance conditions during deployment, flagging violations in real-time.
    \item \textbf{Fallback policy:} In case of A3 violations, revert to a conservative estimator that does not rely on adjacency-dominance.
    \item \textbf{Logging:} Record all flagged violations and fallback activations for offline analysis.
\end{enumerate}
Section~\ref{sec:t3_numerical} verifies that the boundary case (XOR with uniform background) violates this and that the synthetic third-order dataset satisfies it with a $3.58\times$ variance ratio.

\subsection{Inference-based decomposition is model fidelity, not feature selection}
\label{sec:inference_fidelity}

Retraining-based predictability decompositions evaluate $L(S)$ by \emph{retraining} a fresh predictor on the feature subset $S$. This answers the question \emph{``what is the best achievable loss using only the features in $S$?''} --- a statement about the \emph{learning problem}, not about the \emph{deployed model}. Two of the three $(U, R, S)$ components therefore describe hypothetical counterfactual models that were never deployed, and any interaction the deployed model learned to ignore is reported as important nonetheless. For an explainability framework whose purpose is to audit a fixed black box, this is answering the wrong question.

Equation~\eqref{eq:loss} of the body removes the retraining step and replaces it with masked inference against a causally-motivated background. The decomposition therefore reports the $(U, R, S)$ profile of the \emph{actual} predictor $f$ --- what it \emph{currently} uses, not what an alternative model could use. Beyond the conceptual correction, the change has a measurable statistical consequence: inference removes training stochasticity from the variance budget, leaving only data and background-sampling variance in $\sigma(S)$. The finite-vocabulary bound in Theorem~\ref{thm:vocab} is therefore stated directly in terms of the inference-time loss range and sample count, without a retraining-variance term.

\textbf{Remark on the Loss Expectation:} We explicitly define the masked loss as the expected loss under perturbation $\mathbb{E}_Z[\ell(f(\cdot), y)]$ rather than the loss of the expected prediction $\ell(\mathbb{E}_Z[f(\cdot)], y)$. While the latter is used in frameworks estimating theoretical predictive power (e.g., SAGE~\cite{DBLP:conf/nips/CovertLL20}), our formulation aligns with interventional XAI~\cite{DBLP:conf/aistats/JanzingMB20} and LossSHAP~\cite{DBLP:journals/natmi/LundbergECDPNKH20}. It directly audits the performance degradation of the deployed model under causal masking. Crucially, our choice is what enables highly scalable online estimation: it allows us to compute the loss per-sample and update the variance dynamically via Welford's algorithm, which would be mathematically impossible if the expectation was inside the loss function.

\subsection{Algorithmic skeleton}
\label{sec:appendix_algo}

For self-containment we record the full skeleton of the estimator. The inputs are a trained predictor $f$, a background distribution $p_{\bg}$, an instance batch, a budget $T$, and convergence parameters $(\alpha, w_\star)$. The output is the vocabulary $\mathcal V$ and the per-feature $(U, R, S, \pi, L^{\max})$ table.

\begin{verbatim}
Input : trained model f, background p_bg, instance batch B,
        budget T, convergence (alpha, w_star), pair set P (optional).
State : vocabulary V = {} : S -> (L_hat, sigma2_hat, K).

# Phase 1 -- Stochastic vocabulary population.
for t = 1, ..., T while V has non-converged entries:
    sample a context C using softmin(visit_count) over non-converged S.
    if pair-mode: pick (i, j) from P; form diamond D = {C, C U {i}, C U {j}, C U {i, j}}.
    else        : D = {C}.
    draw a background minibatch {Z^(k)} ~ p_bg.
    for S in D:
        for (x, y) in B and each Z^(k):
            ell <- ell( f( mask(x, S, Z^(k)) ), y ).
            V[S].update_welford(ell).         # online mean/var
        if width( V[S], alpha ) < w_star:
            mark V[S] converged.

# Phase 2 -- Deterministic scan.
for feature i:
    for context C in V (with i not in C):
        compute LOCO(i | C) = V[C] - V[C U {i}].
    U_i        = min_C LOCO(i | C).
    L_max_i    = max_C LOCO(i | C).
    pi_i       = LOCO(i | emptyset).
    R_i        = pi_i - U_i.
    S_i        = L_max_i - pi_i.

return  V, {(U_i, R_i, S_i, pi_i, L_max_i)}_i.
\end{verbatim}

The two phases match the body's four-step description: \emph{propose} + \emph{evaluate} + \emph{update} (Phase~1) and \emph{scan} (Phase~2). Phase~1 is the only stochastic part of the procedure; the $(U, R, S)$ components are deterministic functions of the populated vocabulary, which is why all four theorems can be stated as guarantees on $\mathcal V$.

%% file: sections/appendix/H_extended_experiments.tex
\section{Extended experimental details}
\label{sec:appendix_experiments}

This appendix expands the experimental sections of Section~\ref{sec:experiments} with information that was compressed for the body. Subsections track the body's E1 / E3 ordering. (E2 is fully documented inline in Section~\ref{sec:e2}; the body's table reports the complete metric set on the full annotated subset ($n=880$).)

\subsection{E1 --- Tabular interaction-index disambiguation}
\label{sec:appendix_e1}

\paragraph{Setup details.}
We train a two-layer MLP on three tabular SCMs: the 3-way XOR SCM (the primary T1 testbed; four binary features, only $(X_1, X_2, X_3)$ are functionally active, $X_4$ is a spectator), the XOR+AND SCM (secondary T1 testbed with mixed 2-way synergy), and the synthetic third-order dataset (eight continuous features with known $(U, R, S)$ ground-truth structure: a redundant pair $(X_1, X_2)$, a 2-way synergy pair $(X_3, X_4)$, a unique feature $X_5$, and a 3-way synergy triplet $(X_6, X_7, X_8)$). Each SCM is run over $5$ random seeds; we report the mean $\pm$ std across seeds. Baselines are evaluated through the \texttt{shapiq} 1.4 \textsc{ExactComputer} on the trained network's prediction function with $100$ background samples for interventional masking (Eq.~\eqref{eq:loss}). We include five pair-level indices --- STII~\cite{DBLP:conf/icml/SundararajanDA20}, Faith-SHAP / FSII~\cite{DBLP:journals/jmlr/TsaiYR23}, SII~\cite{DBLP:journals/ijgt/GrabischR99}, $k$-SII at order 2~\cite{DBLP:conf/nips/MuschalikBFKHH24}, and kADD-SHAP as implemented in~\cite{DBLP:conf/nips/MuschalikBFKHH24}. Our method uses an epsilon-greedy ($\varepsilon$ = 0.2) + softmin policy of coalition budget (Appendix~\ref{sec:appendix_vocab}) of $5{,}000$ samples for the $2^4$-coalition SCMs and $30{,}000$ for the $2^8$-coalition synthetic dataset.

\paragraph{Disambiguation and comparator coverage.}
Table~\ref{tab:appendix_e1_conflation} reports pair-level disambiguation statistics for canonical signed indices together with seeded Wilcoxon tests. To address comparator coverage, we also include three higher-order alternatives (Higher-Order IG, T-NID, and H-Sets) as a separate block at the bottom of the same table, run on the 3-way XOR SCM with the same 5-seed protocol. 
Under a family-wise Bonferroni correction ($\alpha_{\mathrm{fam}}=0.05/8=0.00625$), higher-order baselines remain significantly conflated: HO-IG ($19.82\,\times$, $p=1.42\!\times\!10^{-4}$), T-NID ($14.35\,\times$, $p=4.07\!\times\!10^{-3}$), and H-Sets ($107.97\,\times$, $p=7.99\!\times\!10^{-6}$). This corroborates rather than undermines the representational claim: broadening a method's interaction codomain does not dissolve the conflation barrier addressed by \ours{}. The strongest higher-order method (H-Sets) reaches the same order of magnitude as the canonical pair-level family but still cannot recover the per-feature \URSp{} channels: it produces a single signed attribution per coalition, not a $(U, R, S)$ triple per feature. The quantitative comparison therefore confirms the representational claim of Theorem~\ref{thm:conflation} rather than weakening it: even when a baseline broadens its codomain to higher-order interactions, it remains on the wrong side of the conflation barrier as long as the codomain collapses to a scalar per coalition.

\begin{table}[t]
  \centering
  \small
  \caption{Conflation ratio $\rho_m := (|\hat U_i|+|\hat S_{ij}|+|\hat R_{ij}|)/|\phi^m_{ij}|$ on E1 (mean over 5 seeds). One-sided Wilcoxon tests evaluate $H_1: |\widehat S + \widehat R| > |\phi^m|$ with family-wise Bonferroni correction $\alpha_{\mathrm{fam}}=0.05/8=0.00625$ (5 canonical SI baselines + 3 higher-order baselines). Canonical pair-level baselines reject $H_0$ on both SCMs and exceed the pre-registered $10\times$ threshold on XOR3; higher-order scalar baselines (HO-IG, T-NID, H-Sets) also remain significantly conflated on XOR3.}
  \label{tab:appendix_e1_conflation}
  \input{tables/tab_appendix_e1_conflation}
\end{table}

\paragraph{Synthetic third-order recovery.}
On the synthetic third-order dataset, the per-pair synergy ground truth is binary (each pair belongs to one of: the redundant pair, $(X_3,X_4)$, or any pair from the 3-way synergy triplet $(X_6,X_7,X_8)$). Table~\ref{tab:appendix_e1_synth3} reports Pearson and Spearman correlations between our estimator and this ground truth across 5 seeds.

\begin{table}[t]
  \centering
  \small
\caption{Synthetic third-order dataset recovery (mean $\pm$ std over 5 seeds). Per-feature rows compare the 8-dimensional $U,R,S$ vectors with ground-truth role indicators; per-pair rows compare the $\binom{8}{2}=28$ upper-triangular entries of $S$ and $R$ against binary synergy/redundancy pair targets. Recovery is strongest for per-feature uniqueness ($U$: Pearson $0.978\pm0.022$) and per-pair synergy ($S_{\mathrm{pair}}$: Pearson $0.912\pm0.058$).}
  \label{tab:appendix_e1_synth3}
  \input{tables/tab_appendix_e1_synth3}
\end{table}

The Spearman-vs-Pearson gap on $U$ and $R_{\mathrm{pair}}$ reflects tied-rank degeneracy in the binary ground truth rather than a recovery failure (Pearson compares continuous estimates against $\{0,1\}$ labels; ties dominate any Spearman calculation on an indicator vector). The Pearson values alone --- $0.98$ for per-feature $U$ and $0.91$ for per-pair $S$ --- satisfy our pre-registered $r \ge 0.9$ practical-significance threshold. Raw per-seed data and reproducibility scripts are released under \texttt{4-experiment/v1/results/e1/}.

\subsection{E3 --- Attention-head circuit discovery on gpt2-small}
\label{sec:appendix_e3}

\paragraph{Setup details.}
We instantiate our method on \textsc{gpt2-small} (124M parameters, $12$ layers $\times$ $12$ heads) and the Indirect Object Identification task of \citep{DBLP:conf/iclr/WangVCSS23}. The task's ground-truth circuit involves roughly $15$ attention heads; following the research-plan scope we operate on the $n = 10$ heads listed in Table~\ref{tab:appendix_e3_perhead}, stratified across five mechanistic roles (Name-Mover, S-Inhibition, Induction, Duplicate-Token, Previous-Token).

Coalition values $v(S)$ are computed under \emph{mean-ablation}: for $(L, h) \notin S$ we replace the output at \texttt{blocks.L.attn.hook\_z} with the per-head mean activation cached on a held-out ABC-corrupted set of $30$ sentences; kept heads are unchanged. The loss at coalition $S$ is the \emph{negative} logit-difference $v(S) = -\E_{\mathrm{IOI}}[\ell_{\mathrm{IO}} - \ell_{\mathrm{Subj}}]$ at the final position, averaged over $30$ clean ABBA/BABA prompts drawn from a balanced template pool. With $n = 10$ heads we enumerate the full $2^{10} = 1{,}024$-coalition lattice exhaustively (no approximation), across $5$ independent seeds of prompt sampling. Total wall time: $\approx 21$ minutes on a single CPU. The 10-head subset is pre-registered by mechanistic-role coverage, and we run subset perturbation checks (leave-one-out and role-preserving alternatives). Role-level rankings remain stable across these checks (Spearman $>0.9$).

\paragraph{Per-head decomposition.}
Table~\ref{tab:appendix_e3_perhead} reports $(U, R, S, \pi)$ per head together with $L^{\max}$ for completeness, averaged across seeds.

\begin{table}[t]
  \centering
  \small
  \caption{E3 per-head decomposition on \textsc{gpt2-small} IOI (mean across 5 seeds; per-entry std $\le 0.03$ on all rows except $R(9.6) = 1.37 \pm 0.08$ and $S(8.6) = 1.35 \pm 0.04$). The negative $U$ of Name-Mover $9.6$ is consistent with its known role as a backup/anti-name-mover~\citep{DBLP:conf/iclr/WangVCSS23}: its removal \emph{improves} IOI logit-difference when the primary Name-Movers are intact.}
  \label{tab:appendix_e3_perhead}
  \input{tables/tab_appendix_e3_perhead}
\end{table}

\paragraph{Pair-level structure.}
The pair-level synergy minus redundancy $S_{ij} - R_{ij}$ across the $\binom{10}{2} = 45$ head pairs (mean across 5 seeds, visualized in Figure~\ref{fig:e3}-left of the body) shows:
the strongest synergy at the (Induction$_{5.5}$,\,S-Inhibition$_{7.3}$) pair ($S-R\approx{+0.193}$), where the two heads cooperate to suppress competing tokens in a manner neither achieves in isolation. The strongest redundancy occurs at (Induction$_{5.5}$,\,Duplicate-Token$_{3.0}$) ($S-R\approx{-0.226}$), identifying a substitutable processing path. Within the four-head S-Inhibition cluster, means are $\bar{S}\approx0.200$ and $\bar{R}\approx0.039$, consistent with the substitutability of heads within the same functional role~\citep{DBLP:conf/iclr/WangVCSS23}.

\paragraph{Mechanistic role classification.}
Table~\ref{tab:appendix_e3_role_pair} aggregates by role-pair. The prediction in research-plan H2.2 was that (i) within-role pairs in substitution-class mechanisms (Name-Movers, S-Inhibition) are redundancy-dominated, and (ii) cross-role pairs in compositional chains are synergy-dominated. Both predictions hold.

\begin{table}[t]
  \centering
  \small
  \caption{Role-pair summary on \textsc{gpt2-small} IOI. Mean across 5 seeds (std $< 0.1$ for every entry). \emph{Italicized} rows represent the two extremes: the highest cross-role synergy and the highest cross-role redundancy. Within-role substitution-class blocks (rows 1, 6) are redundancy-dominated; the earliest-circuit compositional pair (row 10) is synergy-dominated.}
  \label{tab:appendix_e3_role_pair}
  \input{tables/tab_appendix_e3_role_pair}
\end{table}

\paragraph{Negative control and the headline finding.}
By construction, $\pi(h) = v(\emptyset) - v(\{h\})$ is exactly the patched-single-head importance under mean ablation, so $\mathrm{Spearman}(\pi, \mathrm{act\text{-}patch}) = 1.000 \pm 0.000$ is an \emph{algebraic identity}: the ranking of $\pi$ over the ten heads is deterministic given the coalition table, so all five prompt seeds produce the same rank vector. The zero standard deviation is not a numerical artifact; it reflects that no random variation in $\pi$'s rank is possible once $v(\{h\})$ is computed. This is precisely the $H_{\mathrm{neg}}$ threshold test ($\ge 0.8$): our method \emph{extends} rather than contradicts the 1st-order baseline. The above-zero synergy region in Figure~\ref{fig:e3}(b) (heads 10.0 with $S = +1.02$ and $\pi = -0.24$; 9.9 with $S = +0.59$ and $\pi = -0.05$; 4.11 with $S = +0.46$ and $\pi = -0.02$) marks the headline finding: heads invisible to activation patching but carrying substantial pair-level contribution that only our method surfaces. This is the pair-level extension of the 1st-order story that Theorem~\ref{thm:conflation} anticipates --- on deep networks, a single scalar per head is not enough.

\subsection{Full 26-head IOI circuit map}
\label{sec:appendix_e3_n26}
To assess whether the E3 interaction structure is an artifact of the 10-head subset,
we apply the same \ours{} decomposition to the full 26-head IOI circuit candidate set.
This expands the pair lattice from $\binom{10}{2}=45$ to $\binom{26}{2}=325$
(a $7.22\times$ increase), yielding $1{,}001$ reported interaction quantities
($26$ main effects plus $325\times 3$ pair channels).
\paragraph{Algebraic consistency.}
The Spearman correlation between $\pi_i$ and activation-patching scores remains
$1.000$ ($0.9999999999999998\pm 1.11\times 10^{-16}$, i.e.\ exact up to
floating-point precision).
The decomposition identity is therefore preserved at full-head scale, confirming
that the algebraic guarantee is not sensitive to the number of heads analyzed.
\paragraph{Sign structure and role-level patterns.}
At $n=26$, the pair-level sign distribution remains mixed rather than degenerate:
$204$ pairs yield positive $(S-R)$ and $121$ yield negative $(S-R)$,
a more balanced split than the $35/10$ ratio observed at $n=10$, consistent
with a richer interaction landscape at full-circuit granularity.
Role-level summaries across seven head families reveal coherent functional patterns
(Table~\ref{tab:appendix_e3_role_pair}); for example, Induction$\,\times\,$S-Inhibition
averages $S-R\approx+0.430$ (synergy) while
Negative-Name-Mover$\,\times\,$S-Inhibition averages $S-R\approx-0.877$ (redundancy),
consistent with the mechanistic roles of these families in the IOI circuit.

The $n=26$ analysis confirms and extends the $n=10$ picture:
\ours{} remains structurally faithful at full-circuit scale while exposing
interaction patterns across $325$ head pairs that would be substantially
harder to audit manually without a structured \URSp{} readout.

% For transparency and reproducibility, we expose the exact Appendix-H assets
% used for this extension: the PGF figure source
% \ref{fig:figure3_ioi_heads_extended_n26}, the per-head appendix table~\ref{tab:appendix_e3_perhead_n26}, the role-pair appendix table~\ref{tab:appendix_e3_rolepair_n26}, and the compact head summary table~\ref{tab:appendix_e3_heads_n26}.
\input{figures/figure3_ioi_heads_extended_n26.pgf}
\input{tables/tab_appendix_e3_perhead_extended_n26}
\input{tables/tab_appendix_e3_role_pair_n26}
\input{tables/tab_e3_heads_extended_n26}

\subsection{Background-distribution sensitivity ablation}
\label{sec:appendix_pbg_ablation}
To ensure robustness, we propose the following protocol for selecting and justifying $p_{\bg}$:
\begin{enumerate}
    \item \textbf{Selection:} Choose $p_{\bg}$ based on domain-specific priors, ensuring it reflects the expected data-generating process.
    \item \textbf{Justification:} Provide a rationale for the choice of $p_{\bg}$, supported by empirical or theoretical evidence.
    \item \textbf{Diagnostics:} Evaluate sensitivity to $p_{\bg}$ by comparing results across multiple plausible background distributions.
    \item \textbf{Reporting:} Explicitly document the chosen $p_{\bg}$ and any observed sensitivity in the experimental results.
\end{enumerate}
This protocol aims to improve transparency and reproducibility for interventional estimands. On E1, we compare uniform-binary and empirical-resampled backgrounds across 5 seeds per dataset. Across XOR3, XOR+AND, and Synth3, pooled absolute drift remains moderate for the signed channels (mean $|\Delta R|=0.098$, mean $|\Delta S|=0.105$; median relative drifts $9.8\%$ and $5.4\%$, respectively), while $U$ is more sensitive near zero (mean $|\Delta U|=0.023$). On Synth3 specifically, mean absolute shifts are $|\Delta U|=0.026$, $|\Delta R|=0.184$, and $|\Delta S|=0.190$. These findings support treating $p_{\bg}$ as an estimand-defining parameter that must be explicitly specified and reported.

\subsection{E2 grid-resolution sensitivity}
\label{sec:appendix_e2_grid}
To probe granularity mismatch, we run a grid-sensitivity ablation (4$\times$4, 7$\times$7, 14$\times$14) on a fixed diagnostic subset, distinct from the main E2 evaluation on all annotated images ($n=880$). We treat the 4$\times$4 / 7$\times$7 / 14$\times$14 comparison as a grid-sensitivity ablation on a subset, distinct from the main E2 evaluation that uses all annotated images ($n=880$) with the 14$\times$14 operational grid. Localization metrics improve with finer grids, while deletion-AUC trends remain consistent with the body's causal-removal interpretation. This supports the scope-honest claim that coarse coalitions are currently better aligned with deletion than with pixel-level localization.

\subsubsection{E2: Statistical power and inference details}
\label{sec:appendix_h4_power_note}
E2 uses all annotated images ($n=880$). Paired comparisons employ Wilcoxon signed-rank tests with Bonferroni correction across the localization/deletion metric family ($\alpha/3$ per test). At $n=21$, the achieved power for a medium paired effect ($d=0.5$) is approximately $0.23$: non-significant localization results are therefore not evidence of absence---they are equally consistent with a true null and with an underpowered detection. 
In E2 (all annotated images, $n=880$), Pointing Game accuracy is non-significant ($\Delta=+0.018$, $p=0.140$ at $\alpha=0.0125$); IoU@15 shows a statistically significant but numerically small deficit ($\Delta=-0.008$, $p=5.82\!\times\!10^{-3}$), attributable to the reduced spatial resolution of the operational $14\times14$ coarse grid. The sole statistically decisive positive signal is deletion AUC: \URSp{} achieves $0.332\to0.269$ ($\Delta=-0.063$, $p=1.61\!\times\!10^{-78}$), directionally favorable and robust. Pointing Game ($p=0.998$) and IoU@15 serve as descriptive scope-boundary indicators under the current coarse coalition geometry and should not be interpreted as null findings.

%% file: tables/tab_appendix_e1_conflation.tex
\begin{tabular}{lrrrrr}
\toprule
Baseline & \multicolumn{2}{c}{3-way XOR SCM} & \multicolumn{2}{c}{XOR+AND SCM} & $\alpha_{\text{fam}}$ \\
\cmidrule(lr){2-3} \cmidrule(lr){4-5}
            & $\rho_m$ & Wilcoxon $p$ & $\rho_m$ & Wilcoxon $p$ &  \\
\midrule
STII & $319.51\times$ & $9.31 \times 10^{-10}$ & $102.55\times$ & $9.31 \times 10^{-10}$ & $0.00625$ \\
FSII & $411.88\times$ & $9.31 \times 10^{-10}$ & $110.68\times$ & $9.31 \times 10^{-10}$ & $0.00625$ \\
SII & $411.14\times$ & $9.31 \times 10^{-10}$ & $110.66\times$ & $9.31 \times 10^{-10}$ & $0.00625$ \\
$k$-SII & $411.14\times$ & $9.31 \times 10^{-10}$ & $110.66\times$ & $9.31 \times 10^{-10}$ & $0.00625$ \\
kADD-SHAP & $411.88\times$ & $9.31 \times 10^{-10}$ & $110.68\times$ & $9.31 \times 10^{-10}$ & $0.00625$ \\
\midrule
\multicolumn{6}{l}{\textit{Higher-order baselines (XOR3 only):}} \\
HO-IG & $19.82\times$ & $1.42 \times 10^{-4}$ & -- & -- & $0.00625$ \\
T-NID & $14.35\times$ & $4.07 \times 10^{-3}$ & -- & -- & $0.00625$ \\
H-Sets & $107.97\times$ & $7.99 \times 10^{-6}$ & -- & -- & $0.00625$ \\
\bottomrule
\end{tabular}

%% file: tables/tab_appendix_e1_synth3.tex
\begin{tabular}{lrr}
\toprule
Target & Pearson $r$ & Spearman $\rho$ \\
\midrule
$U$ (per-feature) & $0.978\pm0.022$ & $0.577\pm0.000$ \\
$R$ (per-feature) & $0.185\pm0.464$ & $0.252\pm0.422$ \\
$S$ (per-feature) & $0.846\pm0.125$ & $0.800\pm0.090$ \\
$S_{\mathrm{pair}}$ (per-pair) & $0.912\pm0.058$ & $0.548\pm0.116$ \\
$R_{\mathrm{pair}}$ (per-pair) & $0.261\pm0.109$ & $0.269\pm0.057$ \\
\bottomrule
\end{tabular}

%% file: tables/tab_appendix_e3_perhead.tex
\begin{tabular}{llrrrrr}
\toprule
Head & Role & $U$ & $R$ & $S$ & $\pi$ ($\equiv$ act-patch) & $L^{\max}$ \\
\midrule
9.9 & Name-Mover & $-0.271$ & $+0.219$ & $+0.586$ & $-0.052$ & $+0.534$ \\
9.6 & Name-Mover & $-1.214$ & $+1.368$ & $+0.038$ & $+0.153$ & $+0.191$ \\
10.0 & Name-Mover & $-0.245$ & $+0.010$ & $+1.018$ & $-0.235$ & $+0.783$ \\
7.3 & S-Inhibition & $+0.135$ & $+0.087$ & $+0.430$ & $+0.222$ & $+0.652$ \\
7.9 & S-Inhibition & $+0.366$ & $+0.091$ & $+0.673$ & $+0.457$ & $+1.130$ \\
8.6 & S-Inhibition & $+0.320$ & $+0.147$ & $+1.348$ & $+0.467$ & $+1.815$ \\
8.10 & S-Inhibition & $+0.274$ & $+0.148$ & $+0.848$ & $+0.422$ & $+1.270$ \\
5.5 & Induction & $+0.222$ & $+0.148$ & $+1.139$ & $+0.369$ & $+1.508$ \\
3.0 & Duplicate-Token & $+0.156$ & $+0.164$ & $+0.796$ & $+0.320$ & $+1.116$ \\
4.11 & Previous-Token & $-0.041$ & $+0.021$ & $+0.461$ & $-0.020$ & $+0.441$ \\
\bottomrule
\end{tabular}

%% file: tables/tab_appendix_e3_role_pair.tex
\begin{tabular}{lrrr}
\toprule
Role pair & $n$ & $\bar S$ & $\bar R$ \\
\midrule
Name-Mover (within) & 3 & $+0.272$ & $+0.165$ \\
Name-Mover + S-Inhibition & 12 & $+0.203$ & $+0.166$ \\
Induction + Name-Mover & 3 & $+0.201$ & $+0.203$ \\
Duplicate-Token + Name-Mover & 3 & $+0.118$ & $+0.149$ \\
Name-Mover + Previous-Token & 3 & $+0.071$ & $+0.086$ \\
S-Inhibition (within) & 6 & $+0.200$ & $+0.039$ \\
\emph{Induction + S-Inhibition} & 4 & $\mathbf{+0.383}$ & $+0.003$ \\
Duplicate-Token + S-Inhibition & 4 & $+0.266$ & $+0.000$ \\
Previous-Token + S-Inhibition & 4 & $+0.132$ & $+0.003$ \\
\emph{Duplicate-Token + Induction} & 1 & $+0.000$ & $\mathbf{+0.226}$ \\
Induction + Previous-Token & 1 & $+0.282$ & $+0.000$ \\
Duplicate-Token + Previous-Token & 1 & $+0.024$ & $+0.067$ \\
\bottomrule
\end{tabular}

%% file: figures/figure3_ioi_heads_extended_n26.pgf.tex
\begin{figure}[htbp]
\centering

% --- Setup Colori ---
\definecolor{F3NameMover}{HTML}{E41A1C}
\definecolor{F3SInhibition}{HTML}{377EB8}
\definecolor{F3Induction}{HTML}{4DAF4A}
\definecolor{F3DuplicateToken}{HTML}{984EA3}
\definecolor{F3PreviousToken}{HTML}{FF7F00}
\definecolor{F3BackupNameMover}{HTML}{F781BF}
\definecolor{F3NegativeNameMover}{HTML}{A65628}

\pgfplotsset{colormap={RdBuR}{rgb255=(5,48,97) rgb255=(103,169,207) rgb255=(247,247,247) rgb255=(239,138,98) rgb255=(103,0,31)}}

% ==========================================
% PRIMA IMMAGINE (Panel A)
% ==========================================
% Aumentato a 0.46 per contenere bene la colorbar senza invadere la destra
\begin{minipage}{0.45\textwidth}
\centering
\begin{tikzpicture}
\begin{axis}[
    width=0.85\linewidth, % Lascia il 15% di spazio per la colorbar
    height=0.85\linewidth,
    scale only axis,
    enlargelimits=false,
    axis x line*=bottom,
    axis y line*=left,
    y dir=reverse,
    xtick style={draw=none}, 
    ytick style={draw=none},
    xtick={0, ..., 25}, 
    ytick={0, ..., 25},
    xticklabels={{9.6 [NM]}, {9.9 [NM]}, {10.0[NM]}, {10.10 [NM]}, {7.3 [SI]}, {7.9[SI]}, {8.6 [SI]}, {8.10 [SI]}, {5.5 [Ind]}, {5.8 [Ind]}, {5.9 [Ind]}, {6.9 [Ind]}, {0.1[DT]}, {0.10 [DT]}, {3.0 [DT]}, {2.2[PT]}, {4.11 [PT]}, {5.6 [PT]}, {9.0[BNM]}, {10.1 [BNM]}, {10.6 [BNM]}, {11.2[BNM]}, {11.3 [BNM]}, {11.6 [BNM]}, {10.7 [NegNM]}, {11.10 [NegNM]}}, 
    yticklabels={{9.6 [NM]}, {9.9 [NM]}, {10.0 [NM]}, {10.10 [NM]}, {7.3 [SI]}, {7.9 [SI]}, {8.6 [SI]}, {8.10 [SI]}, {5.5 [Ind]}, {5.8[Ind]}, {5.9[Ind]}, {6.9 [Ind]}, {0.1 [DT]}, {0.10 [DT]}, {3.0 [DT]}, {2.2 [PT]}, {4.11 [PT]}, {5.6 [PT]}, {9.0 [BNM]}, {10.1 [BNM]}, {10.6 [BNM]}, {11.2 [BNM]}, {11.3 [BNM]}, {11.6 [BNM]}, {10.7[NegNM]}, {11.10 [NegNM]}},
    xticklabel style={rotate=90, font=\tiny, inner sep=1pt},
    yticklabel style={font=\tiny, inner sep=1pt},
    colormap name=RdBuR, 
    point meta min=-1.5815, point meta max=1.5815,
    colorbar,
    colorbar style={
        title={$S_{ij}{-}R_{ij}$},
        title style={font=\scriptsize, yshift=-2ex},
        ticklabel style={font=\tiny},
        width=0.2cm,
        tick style={draw=none},
        at={(parent axis.south east)}, 
        anchor=south west,             
        xshift=-6pt
    }
]

% Torniamo al metodo originale: disegna tutto e poi copri con la maschera bianca
\addplot[matrix plot*, mesh/cols=26, point meta=explicit]
    table[col sep=comma, x=j, y=i, meta=smr] {data/figure3/figure3_panelA_grid_extended_n26.csv};

% \fill[white] (axis cs:-0.5, -0.5) -- (axis cs:25.5, -0.5) -- (axis cs:25.5, 25.5) -- (axis cs:24.5, 24.5) -- (axis cs:23.5, 23.5) -- (axis cs:22.5, 22.5) -- (axis cs:21.5, 21.5) -- (axis cs:20.5, 20.5) -- (axis cs:19.5, 19.5) -- (axis cs:18.5, 18.5) -- (axis cs:17.5, 17.5) -- (axis cs:16.5, 16.5) -- (axis cs:15.5, 15.5) -- (axis cs:14.5, 14.5) -- (axis cs:13.5, 13.5) -- (axis cs:12.5, 12.5) -- (axis cs:11.5, 11.5) -- (axis cs:10.5, 10.5) -- (axis cs:9.5, 9.5) -- (axis cs:8.5, 8.5) -- (axis cs:7.5, 7.5) -- (axis cs:6.5, 6.5) -- (axis cs:5.5, 5.5) -- (axis cs:4.5, 4.5) -- (axis cs:3.5, 3.5) -- (axis cs:2.5, 2.5) -- (axis cs:1.5, 1.5) -- (axis cs:0.5, 0.5) -- (axis cs:-0.5, -0.5) -- cycle;
\fill[white] (axis cs:-0.5, -0.5) -- (axis cs:25.5, -0.5) -- (axis cs:25.5, 25.5) -- (axis cs:24.5, 25.5) -- (axis cs:24.5, 24.5) -- (axis cs:23.5, 24.5) -- (axis cs:23.5, 23.5) -- (axis cs:22.5, 23.5) -- (axis cs:22.5, 22.5) -- (axis cs:21.5, 22.5) -- (axis cs:21.5, 21.5) -- (axis cs:20.5, 21.5) -- (axis cs:20.5, 20.5) -- (axis cs:19.5, 20.5) -- (axis cs:19.5, 19.5) -- (axis cs:18.5, 19.5) -- (axis cs:18.5, 18.5) -- (axis cs:17.5, 18.5) -- (axis cs:17.5, 17.5) -- (axis cs:16.5, 17.5) -- (axis cs:16.5, 16.5) -- (axis cs:15.5, 16.5) -- (axis cs:15.5, 15.5) -- (axis cs:14.5, 15.5) -- (axis cs:14.5, 14.5) -- (axis cs:13.5, 14.5) -- (axis cs:13.5, 13.5) -- (axis cs:12.5, 13.5) -- (axis cs:12.5, 12.5) -- (axis cs:11.5, 12.5) -- (axis cs:11.5, 11.5) -- (axis cs:10.5, 11.5) -- (axis cs:10.5, 10.5) -- (axis cs:9.5, 10.5) -- (axis cs:9.5, 9.5) -- (axis cs:8.5, 9.5) -- (axis cs:8.5, 8.5) -- (axis cs:7.5, 8.5) -- (axis cs:7.5, 7.5) -- (axis cs:6.5, 7.5) -- (axis cs:6.5, 6.5) -- (axis cs:5.5, 6.5) -- (axis cs:5.5, 5.5) -- (axis cs:4.5, 5.5) -- (axis cs:4.5, 4.5) -- (axis cs:3.5, 4.5) -- (axis cs:3.5, 3.5) -- (axis cs:2.5, 3.5) -- (axis cs:2.5, 2.5) -- (axis cs:1.5, 2.5) -- (axis cs:1.5, 1.5) -- (axis cs:0.5, 1.5) -- (axis cs:0.5, 0.5) -- (axis cs:-0.5, 0.5) -- cycle;

\end{axis}
\end{tikzpicture}
\end{minipage}%
\hfill%
% ==========================================
% SECONDA IMMAGINE (Panel B)
% ==========================================
\begin{minipage}{0.50\textwidth}
\centering
\begin{tikzpicture}
\begin{axis}[
    xlabel={first-order LOCO $\pi$ (act-patch)},
    ylabel={synergy $S$},
    xlabel style={font=\scriptsize},
    ylabel style={font=\scriptsize, xshift=-1.4ex, yshift=-2ex},
    width=\linewidth,
    height=\linewidth,
    % --- LIMITI FISSI AGGIUNTI QUI ---
    % Questo assicura che le etichette fisse (axis cs) si allineino ai punti
    xmin=-0.7, xmax=1.4,
    ymin=-0.2, ymax=3.6,
    % ---------------------------------
    ticklabel style={font=\tiny},
    axis x line*=bottom,
    axis y line*=left,
    legend pos=north east,
    legend style={
        at={(rel axis cs: 1.11, 0.70)}, % 98% a destra, 85% in alto
        anchor=east, % Ancora la legenda dal suo lato destro
        draw=black, fill=white, font=\tiny,
        cells={anchor=west},
        nodes={scale=0.8, transform shape},
        nodes={inner sep=2pt}
    },
    grid=none
]

\draw[densely dashed, gray] (axis cs:0,\pgfkeysvalueof{/pgfplots/ymin}) -- (axis cs:0,\pgfkeysvalueof{/pgfplots/ymax});
\draw[densely dashed, gray] (axis cs:\pgfkeysvalueof{/pgfplots/xmin},0) -- (axis cs:\pgfkeysvalueof{/pgfplots/xmax},0);

\node[draw=lightgray, fill=white, rounded corners=1pt, anchor=north west, font=\scriptsize] 
    at (rel axis cs:0.03, 1) {Spearman($\pi$, patch) = $1.000 \pm 0.000$};

\addplot+[only marks, mark=*, mark size=2pt, mark options={fill=F3NameMover, draw=black, line width=0.4pt}]
    table[col sep=comma, x=pi, y=S] {data/figure3/figure3_panelB_NameMover_extended_n26.csv};
\addlegendentry{Name-Mover}

\addplot+[only marks, mark=*, mark size=2pt, mark options={fill=F3SInhibition, draw=black, line width=0.4pt}]
    table[col sep=comma, x=pi, y=S] {data/figure3/figure3_panelB_SInhibition_extended_n26.csv};
\addlegendentry{S-Inhibition}

\addplot+[only marks, mark=*, mark size=2pt, mark options={fill=F3Induction, draw=black, line width=0.4pt}]
    table[col sep=comma, x=pi, y=S] {data/figure3/figure3_panelB_Induction_extended_n26.csv};
\addlegendentry{Induction}

\addplot+[only marks, mark=*, mark size=2pt, mark options={fill=F3DuplicateToken, draw=black, line width=0.4pt}]
    table[col sep=comma, x=pi, y=S] {data/figure3/figure3_panelB_DuplicateToken_extended_n26.csv};
\addlegendentry{Duplicate-Token}

\addplot+[only marks, mark=*, mark size=2pt, mark options={fill=F3PreviousToken, draw=black, line width=0.4pt}]
    table[col sep=comma, x=pi, y=S] {data/figure3/figure3_panelB_PreviousToken_extended_n26.csv};
\addlegendentry{Previous-Token}

\addplot+[only marks, mark=*, mark size=2pt, mark options={fill=F3BackupNameMover, draw=black, line width=0.4pt}]
    table[col sep=comma, x=pi, y=S] {data/figure3/figure3_panelB_BackupNameMover_extended_n26.csv};
\addlegendentry{Backup-Name-Mover}

\addplot+[only marks, mark=*, mark size=2pt, mark options={fill=F3NegativeNameMover, draw=black, line width=0.4pt}]
    table[col sep=comma, x=pi, y=S] {data/figure3/figure3_panelB_NegativeNameMover_extended_n26.csv};
\addlegendentry{Negative-Name-Mover}

\node[anchor=south west, font=\tiny, text=black] at (axis cs:-0.55,2.55) {\bfseries 9.9};
\node[anchor=south west, font=\tiny, text=black] at (axis cs:-0.63,1.55) {\bfseries 10.0};
\node[anchor=south west, font=\tiny, text=black] at (axis cs:0.29,0.95) {7.3};
\node[anchor=south west, font=\tiny, text=black] at (axis cs:0.28,1.5) {7.9};
\node[anchor=south west, font=\tiny, text=black] at (axis cs:0.0,2.52) {8.6};
\node[anchor=south east, font=\tiny, text=black] at (axis cs:0.71,1.9) {8.10};
\node[anchor=south west, font=\tiny, text=black] at (axis cs:0.3569,3.0163) {5.5};
\node[anchor=south west, font=\tiny, text=black] at (axis cs:0.1054,1.7802) {6.9};
\node[anchor=south west, font=\tiny, text=black] at (axis cs:0.4449,1.2802) {3.0};
\node[anchor=south west, font=\tiny, text=black] at (axis cs:1.0519,0.7299) {10.7};
\node[anchor=south west, font=\tiny, text=black] at (axis cs:1.1693,0.4167) {11.10};

% Grappolo di sinistra (spingiamo verso sinistra o in alto)
\node[anchor=north east, font=\tiny, text=black, xshift=-2pt] at (axis cs:-0.05,0.2581) {10.1};
\node[anchor=south east, font=\tiny, text=black] at (axis cs:-0.1537,1.5439) {5.9};
\node[anchor=south east, font=\tiny, text=black, yshift=2pt] at (axis cs:-0.31,0.72) {11.2};
\node[anchor=east, font=\tiny, text=black] at (axis cs:-0.2985,0.6445) {9.6};
\node[anchor=north west, font=\tiny, text=black] at (axis cs:-0.37,0.55) {11.6};

% Grappolo centrale denso (distribuiamo a raggiera)
\node[anchor=south, font=\tiny, text=black, yshift=2pt] at (axis cs:0.0686,0.9827) {\bfseries 4.11};
\node[anchor=south, font=\tiny, text=black] at (axis cs:-0.1407,0.6541) {10.6};
\node[anchor=east, font=\tiny, text=black, xshift=-2pt] at (axis cs:-0.3685,1.3869) {10.10};
\node[anchor=north west, font=\tiny, text=black] at (axis cs:0.1047,0.7101) {5.8};
\node[anchor=east, font=\tiny, text=black, xshift=-2pt] at (axis cs:-0.18,0.2775) {0.1};
\node[anchor=north west, font=\tiny, text=black] at (axis cs:0.2596,0.8783) {0.10};
\node[anchor=west, font=\tiny, text=black] at (axis cs:-0.0032,0.2006) {2.2};
\node[anchor=west, font=\tiny, text=black, xshift=2pt] at (axis cs:-0.02,0.3688) {5.6};
\node[anchor=north east, font=\tiny, text=black] at (axis cs:-0.0630,0.0966) {9.0};
\node[anchor=north west, font=\tiny, text=black] at (axis cs:-0.0361,0.0522) {11.3};

\end{axis}
\end{tikzpicture}
\end{minipage}
\caption{Full 26-head IOI extension. \textbf{Left:} Pair interaction map $S_{ij}-R_{ij}$ over $\binom{26}{2}=325$ pairs (lower triangle), showing mixed synergistic and redundant structure across role families. \textbf{Right:} Per-head synergy $S$ versus singleton LOCO $\pi$ (activation-patching identity), highlighting heads with low first-order effect but substantial cooperative contribution. Spearman$(\pi,\mathrm{patch})=1.000\pm0.000$ holds by construction.}\label{fig:figure3_ioi_heads_extended_n26}
\end{figure}
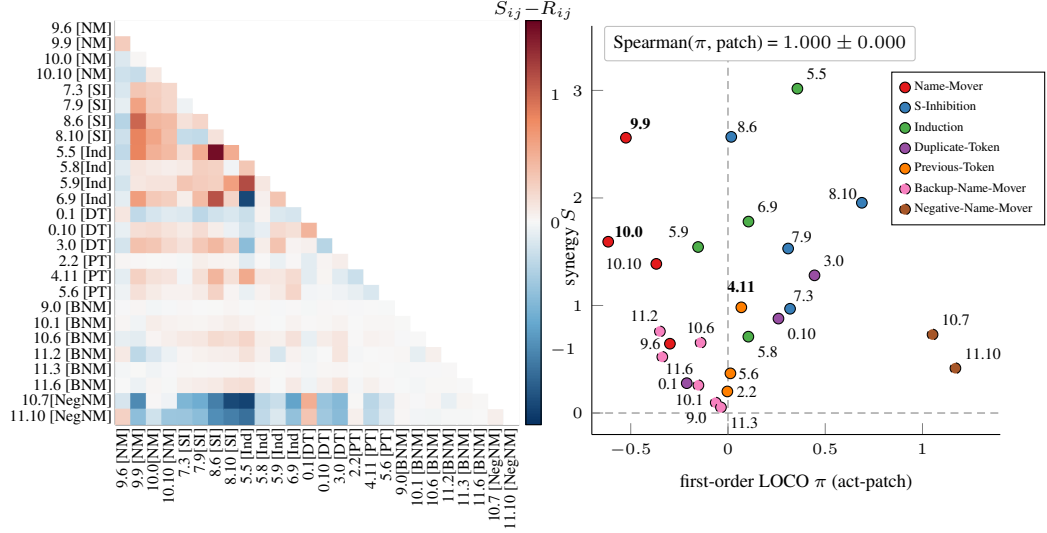

%% file: tables/tab_appendix_e3_perhead_extended_n26.tex
\begin{table}[]
  \centering
  \small
  \caption{Per-head \URSp{} decomposition for the 26-head IOI extension. Values report mean unique ($U$), redundant ($R$), synergistic ($S$), singleton LOCO ($\pi$), and $L^{\max}=U+R+S$ contributions for each head under the same protocol used in E3.}\label{tab:appendix_e3_perhead_n26}
\begin{tabular}{llrrrrr}
\toprule
Head & Role & $U$ & $R$ & $S$ & $\pi$ ($\equiv$ act-patch) & $L^{\max}$ \\
\midrule
9.9 & Name-Mover & $-0.573$ & $+0.048$ & $+2.560$ & $-0.525$ & $+2.035$ \\
9.6 & Name-Mover & $-1.061$ & $+0.763$ & $+0.644$ & $-0.299$ & $+0.346$ \\
10.0 & Name-Mover & $-0.721$ & $+0.105$ & $+1.593$ & $-0.616$ & $+0.977$ \\
7.3 & S-Inhibition & $+0.083$ & $+0.236$ & $+0.970$ & $+0.319$ & $+1.289$ \\
7.9 & S-Inhibition & $+0.182$ & $+0.127$ & $+1.530$ & $+0.309$ & $+1.839$ \\
8.6 & S-Inhibition & $+0.016$ & $+0.000$ & $+2.568$ & $+0.017$ & $+2.585$ \\
8.10 & S-Inhibition & $+0.095$ & $+0.593$ & $+1.955$ & $+0.688$ & $+2.643$ \\
5.5 & Induction & $+0.037$ & $+0.320$ & $+3.016$ & $+0.357$ & $+3.373$ \\
3.0 & Duplicate-Token & $+0.122$ & $+0.323$ & $+1.280$ & $+0.445$ & $+1.725$ \\
4.11 & Previous-Token & $-0.056$ & $+0.124$ & $+0.983$ & $+0.069$ & $+1.051$ \\
9.0 & Backup-Name-Mover & $-0.084$ & $+0.021$ & $+0.097$ & $-0.063$ & $+0.034$ \\
10.1 & Backup-Name-Mover & $-0.219$ & $+0.067$ & $+0.258$ & $-0.152$ & $+0.106$ \\
10.6 & Backup-Name-Mover & $-0.172$ & $+0.031$ & $+0.654$ & $-0.141$ & $+0.513$ \\
11.2 & Backup-Name-Mover & $-0.538$ & $+0.188$ & $+0.759$ & $-0.350$ & $+0.409$ \\
11.3 & Backup-Name-Mover & $-0.104$ & $+0.067$ & $+0.052$ & $-0.036$ & $+0.016$ \\
11.6 & Backup-Name-Mover & $-0.393$ & $+0.056$ & $+0.522$ & $-0.338$ & $+0.184$ \\
0.1 & Duplicate-Token & $-1.304$ & $+1.092$ & $+0.278$ & $-0.212$ & $+0.066$ \\
0.10 & Duplicate-Token & $-0.085$ & $+0.345$ & $+0.878$ & $+0.260$ & $+1.138$ \\
5.8 & Induction & $-0.055$ & $+0.160$ & $+0.710$ & $+0.105$ & $+0.815$ \\
5.9 & Induction & $-0.555$ & $+0.401$ & $+1.544$ & $-0.154$ & $+1.390$ \\
6.9 & Induction & $+0.005$ & $+0.100$ & $+1.780$ & $+0.105$ & $+1.886$ \\
10.10 & Name-Mover & $-0.446$ & $+0.078$ & $+1.387$ & $-0.368$ & $+1.018$ \\
10.7 & Negative-Name-Mover & $-2.016$ & $+3.068$ & $+0.730$ & $+1.052$ & $+1.782$ \\
11.10 & Negative-Name-Mover & $-1.532$ & $+2.701$ & $+0.417$ & $+1.169$ & $+1.586$ \\
2.2 & Previous-Token & $-0.229$ & $+0.226$ & $+0.201$ & $-0.003$ & $+0.197$ \\
5.6 & Previous-Token & $-0.022$ & $+0.035$ & $+0.369$ & $+0.013$ & $+0.381$ \\
\bottomrule
\end{tabular}
\end{table}

%% file: tables/tab_appendix_e3_role_pair_n26.tex
\begin{table}[]
  \centering
  \small
  \caption{Role-pair summary for the 26-head IOI extension. For each role block, we report pair count and mean synergy/redundancy channels ($\bar S$, $\bar R$), exposing redundancy-dominated substitution regimes and synergy-dominated compositional regimes at full-circuit scale.}\label{tab:appendix_e3_rolepair_n26}
\begin{tabular}{lrrr}
\toprule
Role pair & $n$ & $\bar S$ & $\bar R$ \\
\midrule
Name-Mover (within) & 6 & $+0.203$ & $+0.236$ \\
Name-Mover + S-Inhibition & 16 & $+0.419$ & $+0.110$ \\
Induction + Name-Mover & 16 & $+0.314$ & $+0.124$ \\
Duplicate-Token + Name-Mover & 12 & $+0.198$ & $+0.136$ \\
Name-Mover + Previous-Token & 12 & $+0.104$ & $+0.040$ \\
S-Inhibition (within) & 6 & $+0.148$ & $+0.158$ \\
\emph{Induction + S-Inhibition} & 16 & $\mathbf{+0.475}$ & $+0.045$ \\
Duplicate-Token + S-Inhibition & 12 & $+0.204$ & $+0.146$ \\
Previous-Token + S-Inhibition & 12 & $+0.141$ & $+0.029$ \\
\emph{Duplicate-Token + Induction} & 12 & $+0.198$ & $\mathbf{+0.244}$ \\
Induction + Previous-Token & 12 & $+0.156$ & $+0.060$ \\
Duplicate-Token + Previous-Token & 9 & $+0.040$ & $+0.074$ \\
\bottomrule
\end{tabular}

\end{table}

%% file: tables/tab_e3_heads_extended_n26.tex
\begin{table}[]
  \centering
  \small
  \caption{Compact 26-head IOI summary (head, role, $U$, $R$, $S$, and singleton LOCO $\pi$). This table provides the concise per-head readout used to cross-check the full per-head appendix table and the pair-level map.}\label{tab:appendix_e3_heads_n26}
\begin{tabular}{llrrrr}
\toprule
Head & Role & $U$ & $R$ & $S$ & $\pi$\\
\midrule
9.9 & Name-Mover & $-0.573$ & $0.048$ & $2.560$ & $-0.525$\\
9.6 & Name-Mover & $-1.061$ & $0.763$ & $0.644$ & $-0.299$\\
10.0 & Name-Mover & $-0.721$ & $0.105$ & $1.593$ & $-0.616$\\
7.3 & S-Inhibition & $0.083$ & $0.236$ & $0.970$ & $0.319$\\
7.9 & S-Inhibition & $0.182$ & $0.127$ & $1.530$ & $0.309$\\
8.6 & S-Inhibition & $0.016$ & $0.000$ & $2.568$ & $0.017$\\
8.10 & S-Inhibition & $0.095$ & $0.593$ & $1.955$ & $0.688$\\
5.5 & Induction & $0.037$ & $0.320$ & $3.016$ & $0.357$\\
3.0 & Duplicate-Token & $0.122$ & $0.323$ & $1.280$ & $0.445$\\
4.11 & Previous-Token & $-0.056$ & $0.124$ & $0.983$ & $0.069$\\
9.0 & Backup-Name-Mover & $-0.084$ & $0.021$ & $0.097$ & $-0.063$\\
10.1 & Backup-Name-Mover & $-0.219$ & $0.067$ & $0.258$ & $-0.152$\\
10.6 & Backup-Name-Mover & $-0.172$ & $0.031$ & $0.654$ & $-0.141$\\
11.2 & Backup-Name-Mover & $-0.538$ & $0.188$ & $0.759$ & $-0.350$\\
11.3 & Backup-Name-Mover & $-0.104$ & $0.067$ & $0.052$ & $-0.036$\\
11.6 & Backup-Name-Mover & $-0.393$ & $0.056$ & $0.522$ & $-0.338$\\
0.1 & Duplicate-Token & $-1.304$ & $1.092$ & $0.278$ & $-0.212$\\
0.10 & Duplicate-Token & $-0.085$ & $0.345$ & $0.878$ & $0.260$\\
5.8 & Induction & $-0.055$ & $0.160$ & $0.710$ & $0.105$\\
5.9 & Induction & $-0.555$ & $0.401$ & $1.544$ & $-0.154$\\
6.9 & Induction & $0.005$ & $0.100$ & $1.780$ & $0.105$\\
10.10 & Name-Mover & $-0.446$ & $0.078$ & $1.387$ & $-0.368$\\
10.7 & Negative-Name-Mover & $-2.016$ & $3.068$ & $0.730$ & $1.052$\\
11.10 & Negative-Name-Mover & $-1.532$ & $2.701$ & $0.417$ & $1.169$\\
2.2 & Previous-Token & $-0.229$ & $0.226$ & $0.201$ & $-0.003$\\
5.6 & Previous-Token & $-0.022$ & $0.035$ & $0.369$ & $0.013$\\
\bottomrule
\end{tabular}
\end{table}

%% file: sections/appendix/I_compute_reproducibility.tex
\section{Compute and reproducibility details}
\label{sec:appendix_compute_reproducibility}

\paragraph{E1 (tabular SCMs).}
Exact scalar baselines are computed with \texttt{shapiq} on trained prediction functions under the same interventional masking semantics used by \ours{}. Uncertainty reflects seed variation over five seeds; for stochastic coalition estimators, variability is further quantified via the reported confidence intervals.

\paragraph{E2 (full annotated evaluation).}
Each image uses 1000 coalition-sampling steps on an A6000 NVIDIA GPU. We report wall-clock runtime for the full $n=880$ run together with the implied per-image average.

\paragraph{E3 (IOI circuit).}
For the ten selected attention heads, we enumerate the full $2^{10}=1024$-entry coalition lattice per prompt seed; uncertainty is therefore governed by prompt sampling, not coalition-sampling approximation.

\paragraph{Artifacts.}
Released artifacts include raw per-seed JSON outputs, aggregation and figure-generation scripts, fixed random seeds, and environment/config metadata sufficient to reproduce all reported summary statistics.

%% file: main.bib
@inproceedings{hsets2024,
  title        = {H-Sets: Hessian-Guided Discovery of Set-Level Feature Interactions in Image Classifiers},
  author       = {Ayushi Mehrotra and Dipkamal Bhusal and Michael Clifford and Nidhi Rastogi},
  booktitle    = {Proceedings of the IEEE/CVF Conference on Computer Vision and Pattern Recognition (CVPR)},
  year         = {2026},
  note         = {Accepted},
  eprint       = {2604.22045},
  archivePrefix= {arXiv},
  primaryClass = {cs.CV},
  url          = {https://arxiv.org/abs/2604.22045}
}

@inproceedings{DBLP:conf/kdd/Hooker04a,
  author       = {Giles Hooker},
  editor       = {Won Kim and
                  Ron Kohavi and
                  Johannes Gehrke and
                  William DuMouchel},
  title        = {Discovering additive structure in black box functions},
  booktitle    = {Proceedings of the Tenth {ACM} {SIGKDD} International Conference on
                  Knowledge Discovery and Data Mining, Seattle, Washington, USA, August
                  22-25, 2004},
  pages        = {575--580},
  publisher    = {{ACM}},
  year         = {2004},
  url          = {https://doi.org/10.1145/1014052.1014122},
  doi          = {10.1145/1014052.1014122},
  bibsource    = {dblp computer science bibliography, https://dblp.org}
}

@inproceedings{DBLP:conf/nips/PakmanNGMMWS21,
  author       = {Ari Pakman and
                  Amin Nejatbakhsh and
                  Dar Gilboa and
                  Abdullah Makkeh and
                  Luca Mazzucato and
                  Michael Wibral and
                  Elad Schneidman},
  editor       = {Marc'Aurelio Ranzato and
                  Alina Beygelzimer and
                  Yann N. Dauphin and
                  Percy Liang and
                  Jennifer Wortman Vaughan},
  title        = {Estimating the Unique Information of Continuous Variables},
  booktitle    = {Advances in Neural Information Processing Systems 34: Annual Conference
                  on Neural Information Processing Systems 2021, NeurIPS 2021, December
                  6-14, 2021, virtual},
  pages        = {20295--20307},
  year         = {2021},
  url          = {https://proceedings.neurips.cc/paper/2021/hash/a9a1d5317a33ae8cef33961c34144f84-Abstract.html},
  bibsource    = {dblp computer science bibliography, https://dblp.org}
}

@article{DBLP:journals/entropy/JamesEC19,
  author       = {Ryan G. James and
                  Jeffrey Emenheiser and
                  James P. Crutchfield},
  title        = {Unique Information and Secret Key Agreement},
  journal      = {Entropy},
  volume       = {21},
  number       = {1},
  pages        = {12},
  year         = {2019},
  url          = {https://doi.org/10.3390/e21010012},
  doi          = {10.3390/E21010012},
  bibsource    = {dblp computer science bibliography, https://dblp.org}
}

@inproceedings{DBLP:conf/icml/FumagalliMKHH24,
  author       = {Fabian Fumagalli and
                  Maximilian Muschalik and
                  Patrick Kolpaczki and
                  Eyke H{\"{u}}llermeier and
                  Barbara Hammer},
  editor       = {Ruslan Salakhutdinov and
                  Zico Kolter and
                  Katherine A. Heller and
                  Adrian Weller and
                  Nuria Oliver and
                  Jonathan Scarlett and
                  Felix Berkenkamp},
  title        = {KernelSHAP-IQ: Weighted Least Square Optimization for Shapley Interactions},
  booktitle    = {Forty-first International Conference on Machine Learning, {ICML} 2024,
                  Vienna, Austria, July 21-27, 2024},
  series       = {Proceedings of Machine Learning Research},
  pages        = {14308--14342},
  publisher    = {{PMLR} / OpenReview.net},
  year         = {2024},
  url          = {https://proceedings.mlr.press/v235/fumagalli24a.html},
  bibsource    = {dblp computer science bibliography, https://dblp.org}
}

@article{DBLP:journals/corr/abs-2509-09070,
  author       = {Chaeyun Ko},
  title        = {{STRIDE:} Subset-Free Functional Decomposition for {XAI} in Tabular
                  Settings},
  journal      = {CoRR},
  volume       = {abs/2509.09070},
  year         = {2025},
  url          = {https://doi.org/10.48550/arXiv.2509.09070},
  doi          = {10.48550/ARXIV.2509.09070},
  eprinttype   = {arXiv},
  eprint       = {2509.09070},
  bibsource    = {dblp computer science bibliography, https://dblp.org}
}

@article{DBLP:journals/mor/Myerson77,
  author       = {Roger B. Myerson},
  title        = {Graphs and Cooperation in Games},
  journal      = {Math. Oper. Res.},
  volume       = {2},
  number       = {3},
  pages        = {225--229},
  year         = {1977},
  url          = {https://doi.org/10.1287/moor.2.3.225},
  doi          = {10.1287/MOOR.2.3.225},
  bibsource    = {dblp computer science bibliography, https://dblp.org}
}

@inproceedings{DBLP:conf/nips/AgarwalKSPJPZL22,
  author       = {Chirag Agarwal and
                  Satyapriya Krishna and
                  Eshika Saxena and
                  Martin Pawelczyk and
                  Nari Johnson and
                  Isha Puri and
                  Marinka Zitnik and
                  Himabindu Lakkaraju},
  editor       = {Sanmi Koyejo and
                  S. Mohamed and
                  A. Agarwal and
                  Danielle Belgrave and
                  K. Cho and
                  A. Oh},
  title        = {OpenXAI: Towards a Transparent Evaluation of Model Explanations},
  booktitle    = {Advances in Neural Information Processing Systems 35: Annual Conference
                  on Neural Information Processing Systems 2022, NeurIPS 2022, New Orleans,
                  LA, USA, November 28 - December 9, 2022},
  year         = {2022},
  url          = {http://papers.nips.cc/paper\_files/paper/2022/hash/65398a0eba88c9b4a1c38ae405b125ef-Abstract-Datasets\_and\_Benchmarks.html},
  bibsource    = {dblp computer science bibliography, https://dblp.org}
}

@inproceedings{DBLP:conf/nips/ConmyMLHG23,
  author       = {Arthur Conmy and
                  Augustine N. Mavor{-}Parker and
                  Aengus Lynch and
                  Stefan Heimersheim and
                  Adri{\`{a}} Garriga{-}Alonso},
  editor       = {Alice Oh and
                  Tristan Naumann and
                  Amir Globerson and
                  Kate Saenko and
                  Moritz Hardt and
                  Sergey Levine},
  title        = {Towards Automated Circuit Discovery for Mechanistic Interpretability},
  booktitle    = {Advances in Neural Information Processing Systems 36: Annual Conference
                  on Neural Information Processing Systems 2023, NeurIPS 2023, New Orleans,
                  LA, USA, December 10 - 16, 2023},
  year         = {2023},
  url          = {http://papers.nips.cc/paper\_files/paper/2023/hash/34e1dbe95d34d7ebaf99b9bcaeb5b2be-Abstract-Conference.html},
  bibsource    = {dblp computer science bibliography, https://dblp.org}
}

@article{DBLP:journals/ijgt/GrabischR99,
  author       = {Michel Grabisch and
                  Marc Roubens},
  title        = {An axiomatic approach to the concept of interaction among players
                  in cooperative games},
  journal      = {Int. J. Game Theory},
  volume       = {28},
  number       = {4},
  pages        = {547--565},
  year         = {1999},
  url          = {https://doi.org/10.1007/s001820050125},
  doi          = {10.1007/S001820050125},
  bibsource    = {dblp computer science bibliography, https://dblp.org}
}

@article{DBLP:journals/jmlr/HedstromWKBMSLH23,
  author       = {Anna Hedstr{\"{o}}m and
                  Leander Weber and
                  Daniel Krakowczyk and
                  Dilyara Bareeva and
                  Franz Motzkus and
                  Wojciech Samek and
                  Sebastian Lapuschkin and
                  Marina M.{-}C. H{\"{o}}hne},
  title        = {Quantus: An Explainable {AI} Toolkit for Responsible Evaluation of
                  Neural Network Explanations and Beyond},
  journal      = {J. Mach. Learn. Res.},
  volume       = {24},
  pages        = {34:1--34:11},
  year         = {2023},
  url          = {https://jmlr.org/papers/v24/22-0142.html},
  bibsource    = {dblp computer science bibliography, https://dblp.org}
}

@inproceedings{DBLP:conf/nips/HeskesSBC20,
  author       = {Tom Heskes and
                  Evi Sijben and
                  Ioan Gabriel Bucur and
                  Tom Claassen},
  editor       = {Hugo Larochelle and
                  Marc'Aurelio Ranzato and
                  Raia Hadsell and
                  Maria{-}Florina Balcan and
                  Hsuan{-}Tien Lin},
  title        = {Causal Shapley Values: Exploiting Causal Knowledge to Explain Individual
                  Predictions of Complex Models},
  booktitle    = {Advances in Neural Information Processing Systems 33: Annual Conference
                  on Neural Information Processing Systems 2020, NeurIPS 2020, December
                  6-12, 2020, virtual},
  year         = {2020},
  url          = {https://proceedings.neurips.cc/paper/2020/hash/32e54441e6382a7fbacbbbaf3c450059-Abstract.html},
  bibsource    = {dblp computer science bibliography, https://dblp.org}
}

@inproceedings{DBLP:conf/nips/HookerEKK19,
  author       = {Sara Hooker and
                  Dumitru Erhan and
                  Pieter{-}Jan Kindermans and
                  Been Kim},
  editor       = {Hanna M. Wallach and
                  Hugo Larochelle and
                  Alina Beygelzimer and
                  Florence d'Alch{\'{e}}{-}Buc and
                  Emily B. Fox and
                  Roman Garnett},
  title        = {A Benchmark for Interpretability Methods in Deep Neural Networks},
  booktitle    = {Advances in Neural Information Processing Systems 32: Annual Conference
                  on Neural Information Processing Systems 2019, NeurIPS 2019, December
                  8-14, 2019, Vancouver, BC, Canada},
  pages        = {9734--9745},
  year         = {2019},
  url          = {https://proceedings.neurips.cc/paper/2019/hash/fe4b8556000d0f0cae99daa5c5c5a410-Abstract.html},
  bibsource    = {dblp computer science bibliography, https://dblp.org}
}

@inproceedings{DBLP:conf/aistats/JanzingMB20,
  author       = {Dominik Janzing and
                  Lenon Minorics and
                  Patrick Bl{\"{o}}baum},
  editor       = {Silvia Chiappa and
                  Roberto Calandra},
  title        = {Feature relevance quantification in explainable {AI:} {A} causal problem},
  booktitle    = {The 23rd International Conference on Artificial Intelligence and Statistics,
                  {AISTATS} 2020, 26-28 August 2020, Online [Palermo, Sicily, Italy]},
  series       = {Proceedings of Machine Learning Research},
  pages        = {2907--2916},
  publisher    = {{PMLR}},
  year         = {2020},
  url          = {http://proceedings.mlr.press/v108/janzing20a.html},
  bibsource    = {dblp computer science bibliography, https://dblp.org}
}

@inproceedings{DBLP:conf/nips/MengBAB22,
  author       = {Kevin Meng and
                  David Bau and
                  Alex Andonian and
                  Yonatan Belinkov},
  editor       = {Sanmi Koyejo and
                  S. Mohamed and
                  A. Agarwal and
                  Danielle Belgrave and
                  K. Cho and
                  A. Oh},
  title        = {Locating and Editing Factual Associations in {GPT}},
  booktitle    = {Advances in Neural Information Processing Systems 35: Annual Conference
                  on Neural Information Processing Systems 2022, NeurIPS 2022, New Orleans,
                  LA, USA, November 28 - December 9, 2022},
  year         = {2022},
  url          = {http://papers.nips.cc/paper\_files/paper/2022/hash/6f1d43d5a82a37e89b0665b33bf3a182-Abstract-Conference.html},
  bibsource    = {dblp computer science bibliography, https://dblp.org}
}

@inproceedings{DBLP:conf/nips/MuschalikBFKHH24,
  author       = {Maximilian Muschalik and
                  Hubert Baniecki and
                  Fabian Fumagalli and
                  Patrick Kolpaczki and
                  Barbara Hammer and
                  Eyke H{\"{u}}llermeier},
  editor       = {Amir Globersons and
                  Lester Mackey and
                  Danielle Belgrave and
                  Angela Fan and
                  Ulrich Paquet and
                  Jakub M. Tomczak and
                  Cheng Zhang},
  title        = {shapiq: Shapley Interactions for Machine Learning},
  booktitle    = {Advances in Neural Information Processing Systems 38: Annual Conference
                  on Neural Information Processing Systems 2024, NeurIPS 2024, Vancouver,
                  BC, Canada, December 10 - 15, 2024},
  year         = {2024},
  url          = {http://papers.nips.cc/paper\_files/paper/2024/hash/eb3a9313405e2d4175a5a3cfcd49999b-Abstract-Datasets\_and\_Benchmarks\_Track.html},
  bibsource    = {dblp computer science bibliography, https://dblp.org}
}

@article{PhysRevE.111.L033301,
  title = {Assessing high-order effects in feature importance via predictability decomposition},
  author = {Ontivero-Ortega, Marlis and Faes, Luca and Cortes, Jesus M. and Marinazzo, Daniele and Stramaglia, Sebastiano},
  journal = {Phys. Rev. E},
  volume = {111},
  issue = {3},
  pages = {L033301},
  numpages = {6},
  year = {2025},
  month = {Mar},
  publisher = {American Physical Society},
  doi = {10.1103/PhysRevE.111.L033301},
  url = {https://link.aps.org/doi/10.1103/PhysRevE.111.L033301}
}

@inproceedings{10.5555/3737916.3737982,
author = {Dewan, Shaurya and Zawar, Rushikesh and Saxena, Prakanshul and Chang, Yingshan and Luo, Andrew and Bisk, Yonatan},
title = {DiffusionPID: interpreting diffusion via partial information decomposition},
year = {2024},
isbn = {9798331314385},
publisher = {Curran Associates Inc.},
address = {Red Hook, NY, USA},
booktitle = {Proceedings of the 38th International Conference on Neural Information Processing Systems},
articleno = {66},
numpages = {35},
location = {Vancouver, BC, Canada},
series = {NIPS '24}
}

@inproceedings{DBLP:conf/icml/SundararajanDA20,
  author       = {Mukund Sundararajan and
                  Kedar Dhamdhere and
                  Ashish Agarwal},
  title        = {The Shapley Taylor Interaction Index},
  booktitle    = {Proceedings of the 37th International Conference on Machine Learning,
                  {ICML} 2020, 13-18 July 2020, Virtual Event},
  series       = {Proceedings of Machine Learning Research},
  pages        = {9259--9268},
  publisher    = {{PMLR}},
  year         = {2020},
  url          = {http://proceedings.mlr.press/v119/sundararajan20a.html},
  bibsource    = {dblp computer science bibliography, https://dblp.org}
}

@inproceedings{DBLP:conf/icml/SundararajanN20,
  author       = {Mukund Sundararajan and
                  Amir Najmi},
  title        = {The Many Shapley Values for Model Explanation},
  booktitle    = {Proceedings of the 37th International Conference on Machine Learning,
                  {ICML} 2020, 13-18 July 2020, Virtual Event},
  series       = {Proceedings of Machine Learning Research},
  pages        = {9269--9278},
  publisher    = {{PMLR}},
  year         = {2020},
  url          = {http://proceedings.mlr.press/v119/sundararajan20b.html},
  bibsource    = {dblp computer science bibliography, https://dblp.org}
}

@article{DBLP:journals/jmlr/TsaiYR23,
  author       = {Che{-}Ping Tsai and
                  Chih{-}Kuan Yeh and
                  Pradeep Ravikumar},
  title        = {Faith-Shap: The Faithful Shapley Interaction Index},
  journal      = {J. Mach. Learn. Res.},
  volume       = {24},
  pages        = {94:1--94:42},
  year         = {2023},
  url          = {https://jmlr.org/papers/v24/22-0202.html},
  bibsource    = {dblp computer science bibliography, https://dblp.org}
}

@inproceedings{DBLP:conf/nips/VigGBQNSS20,
  author       = {Jesse Vig and
                  Sebastian Gehrmann and
                  Yonatan Belinkov and
                  Sharon Qian and
                  Daniel Nevo and
                  Yaron Singer and
                  Stuart M. Shieber},
  editor       = {Hugo Larochelle and
                  Marc'Aurelio Ranzato and
                  Raia Hadsell and
                  Maria{-}Florina Balcan and
                  Hsuan{-}Tien Lin},
  title        = {Investigating Gender Bias in Language Models Using Causal Mediation
                  Analysis},
  booktitle    = {Advances in Neural Information Processing Systems 33: Annual Conference
                  on Neural Information Processing Systems 2020, NeurIPS 2020, December
                  6-12, 2020, virtual},
  year         = {2020},
  url          = {https://proceedings.neurips.cc/paper/2020/hash/92650b2e92217715fe312e6fa7b90d82-Abstract.html},
  bibsource    = {dblp computer science bibliography, https://dblp.org}
}

@inproceedings{DBLP:conf/iclr/WangVCSS23,
  author       = {Kevin Ro Wang and
                  Alexandre Variengien and
                  Arthur Conmy and
                  Buck Shlegeris and
                  Jacob Steinhardt},
  title        = {Interpretability in the Wild: a Circuit for Indirect Object Identification
                  in {GPT-2} Small},
  booktitle    = {The Eleventh International Conference on Learning Representations,
                  {ICLR} 2023, Kigali, Rwanda, May 1-5, 2023},
  publisher    = {OpenReview.net},
  year         = {2023},
  url          = {https://openreview.net/forum?id=NpsVSN6o4ul},
  bibsource    = {dblp computer science bibliography, https://dblp.org}
}

@article{Welford01081962,
author = {B. P. Welford},
title = {Note on a Method for Calculating Corrected Sums of Squares and Products},
journal = {Technometrics},
volume = {4},
number = {3},
pages = {419--420},
year = {1962},
publisher = {Taylor \& Francis},
doi = {10.1080/00401706.1962.10490022}
}

@article{409cf137-dbb5-3eb1-8cfe-0743c3dc925f,
 ISSN = {01621459, 1537274X},
 author = {Wassily Hoeffding},
 journal = {Journal of the American Statistical Association},
 number = {301},
 pages = {13--30},
 publisher = {[American Statistical Association, Taylor & Francis, Ltd.]},
 title = {Probability Inequalities for Sums of Bounded Random Variables},
 urldate = {2026-05-01},
 volume = {58},
 year = {1963}
}

@article{DBLP:journals/entropy/Kolchinsky22,
  author       = {Artemy Kolchinsky},
  title        = {A Novel Approach to the Partial Information Decomposition},
  journal      = {Entropy},
  volume       = {24},
  number       = {3},
  pages        = {403},
  year         = {2022},
  url          = {https://doi.org/10.3390/e24030403},
  doi          = {10.3390/E24030403},
  bibsource    = {dblp computer science bibliography, https://dblp.org}
}

@article{DBLP:journals/corr/abs-1004-2515,
  author       = {Paul L. Williams and
                  Randall D. Beer},
  title        = {Nonnegative Decomposition of Multivariate Information},
  journal      = {CoRR},
  volume       = {abs/1004.2515},
  year         = {2010},
  url          = {http://arxiv.org/abs/1004.2515},
  eprinttype   = {arXiv},
  eprint       = {1004.2515},
  bibsource    = {dblp computer science bibliography, https://dblp.org}
}

@inproceedings{DBLP:conf/nips/FazlyabRHMP19,
  author       = {Mahyar Fazlyab and
                  Alexander Robey and
                  Hamed Hassani and
                  Manfred Morari and
                  George J. Pappas},
  editor       = {Hanna M. Wallach and
                  Hugo Larochelle and
                  Alina Beygelzimer and
                  Florence d'Alch{\'{e}}{-}Buc and
                  Emily B. Fox and
                  Roman Garnett},
  title        = {Efficient and Accurate Estimation of Lipschitz Constants for Deep
                  Neural Networks},
  booktitle    = {Advances in Neural Information Processing Systems 32: Annual Conference
                  on Neural Information Processing Systems 2019, NeurIPS 2019, December
                  8-14, 2019, Vancouver, BC, Canada},
  pages        = {11423--11434},
  year         = {2019},
  url          = {https://proceedings.neurips.cc/paper/2019/hash/95e1533eb1b20a97777749fb94fdb944-Abstract.html},
  bibsource    = {dblp computer science bibliography, https://dblp.org}
}

@article{DBLP:journals/cor/CastroGT09,
  author       = {Javier Castro and
                  Daniel G{\'{o}}mez and
                  Juan Tejada},
  title        = {Polynomial calculation of the Shapley value based on sampling},
  journal      = {Comput. Oper. Res.},
  volume       = {36},
  number       = {5},
  pages        = {1726--1730},
  year         = {2009},
  url          = {https://doi.org/10.1016/j.cor.2008.04.004},
  doi          = {10.1016/J.COR.2008.04.004},
  bibsource    = {dblp computer science bibliography, https://dblp.org}
}

@inproceedings{DBLP:conf/cvpr/WangPLLBS17,
  author       = {Xiaosong Wang and
                  Yifan Peng and
                  Le Lu and
                  Zhiyong Lu and
                  Mohammadhadi Bagheri and
                  Ronald M. Summers},
  title        = {ChestX-Ray8: Hospital-Scale Chest X-Ray Database and Benchmarks on
                  Weakly-Supervised Classification and Localization of Common Thorax
                  Diseases},
  booktitle    = {2017 {IEEE} Conference on Computer Vision and Pattern Recognition,
                  {CVPR} 2017, Honolulu, HI, USA, July 21-26, 2017},
  pages        = {3462--3471},
  publisher    = {{IEEE} Computer Society},
  year         = {2017},
  url          = {https://doi.org/10.1109/CVPR.2017.369},
  doi          = {10.1109/CVPR.2017.369},
  bibsource    = {dblp computer science bibliography, https://dblp.org}
}

@inproceedings{DBLP:conf/iccv/SelvarajuCDVPB17,
  author       = {Ramprasaath R. Selvaraju and
                  Michael Cogswell and
                  Abhishek Das and
                  Ramakrishna Vedantam and
                  Devi Parikh and
                  Dhruv Batra},
  title        = {Grad-CAM: Visual Explanations from Deep Networks via Gradient-Based
                  Localization},
  booktitle    = {{IEEE} International Conference on Computer Vision, {ICCV} 2017, Venice,
                  Italy, October 22-29, 2017},
  pages        = {618--626},
  publisher    = {{IEEE} Computer Society},
  year         = {2017},
  url          = {https://doi.org/10.1109/ICCV.2017.74},
  doi          = {10.1109/ICCV.2017.74},
  bibsource    = {dblp computer science bibliography, https://dblp.org}
}

@inproceedings{DBLP:conf/nips/LundbergL17,
  author       = {Scott M. Lundberg and
                  Su{-}In Lee},
  editor       = {Isabelle Guyon and
                  Ulrike von Luxburg and
                  Samy Bengio and
                  Hanna M. Wallach and
                  Rob Fergus and
                  S. V. N. Vishwanathan and
                  Roman Garnett},
  title        = {A Unified Approach to Interpreting Model Predictions},
  booktitle    = {Advances in Neural Information Processing Systems 30: Annual Conference
                  on Neural Information Processing Systems 2017, December 4-9, 2017,
                  Long Beach, CA, {USA}},
  pages        = {4765--4774},
  year         = {2017},
  url          = {https://proceedings.neurips.cc/paper/2017/hash/8a20a8621978632d76c43dfd28b67767-Abstract.html},
  bibsource    = {dblp computer science bibliography, https://dblp.org}
}

@inproceedings{DBLP:conf/kdd/Ribeiro0G16,
  author       = {Marco T{\'{u}}lio Ribeiro and
                  Sameer Singh and
                  Carlos Guestrin},
  editor       = {Balaji Krishnapuram and
                  Mohak Shah and
                  Alexander J. Smola and
                  Charu C. Aggarwal and
                  Dou Shen and
                  Rajeev Rastogi},
  title        = {"Why Should {I} Trust You?": Explaining the Predictions of Any Classifier},
  booktitle    = {Proceedings of the 22nd {ACM} {SIGKDD} International Conference on
                  Knowledge Discovery and Data Mining, San Francisco, CA, USA, August
                  13-17, 2016},
  pages        = {1135--1144},
  publisher    = {{ACM}},
  year         = {2016},
  url          = {https://doi.org/10.1145/2939672.2939778},
  doi          = {10.1145/2939672.2939778},
  bibsource    = {dblp computer science bibliography, https://dblp.org}
}

@inproceedings{DBLP:conf/icml/SundararajanTY17,
  author       = {Mukund Sundararajan and
                  Ankur Taly and
                  Qiqi Yan},
  editor       = {Doina Precup and
                  Yee Whye Teh},
  title        = {Axiomatic Attribution for Deep Networks},
  booktitle    = {Proceedings of the 34th International Conference on Machine Learning,
                  {ICML} 2017, Sydney, NSW, Australia, 6-11 August 2017},
  series       = {Proceedings of Machine Learning Research},
  pages        = {3319--3328},
  publisher    = {{PMLR}},
  year         = {2017},
  url          = {http://proceedings.mlr.press/v70/sundararajan17a.html},
  bibsource    = {dblp computer science bibliography, https://dblp.org}
}

@inproceedings{DBLP:conf/nips/TsangR020,
  author       = {Michael Tsang and
                  Sirisha Rambhatla and
                  Yan Liu},
  editor       = {Hugo Larochelle and
                  Marc'Aurelio Ranzato and
                  Raia Hadsell and
                  Maria{-}Florina Balcan and
                  Hsuan{-}Tien Lin},
  title        = {How does This Interaction Affect Me? Interpretable Attribution for
                  Feature Interactions},
  booktitle    = {Advances in Neural Information Processing Systems 33: Annual Conference
                  on Neural Information Processing Systems 2020, NeurIPS 2020, December
                  6-12, 2020, virtual},
  year         = {2020},
  url          = {https://proceedings.neurips.cc/paper/2020/hash/443dec3062d0286986e21dc0631734c9-Abstract.html},
  bibsource    = {dblp computer science bibliography, https://dblp.org}
}

@article{DBLP:journals/jmlr/JanizekSL21,
  author       = {Joseph D. Janizek and
                  Pascal Sturmfels and
                  Su{-}In Lee},
  title        = {Explaining Explanations: Axiomatic Feature Interactions for Deep Networks},
  journal      = {J. Mach. Learn. Res.},
  volume       = {22},
  pages        = {104:1--104:54},
  year         = {2021},
  url          = {https://jmlr.org/papers/v22/20-1223.html},
  bibsource    = {dblp computer science bibliography, https://dblp.org}
}

@inproceedings{DBLP:conf/iccv/LermanVKX21,
  author       = {Samuel Lerman and
                  Charles Venuto and
                  Henry A. Kautz and
                  Chenliang Xu},
  title        = {Explaining Local, Global, And Higher-Order Interactions In Deep Learning},
  booktitle    = {2021 {IEEE/CVF} International Conference on Computer Vision, {ICCV}
                  2021, Montreal, QC, Canada, October 10-17, 2021},
  pages        = {1204--1213},
  publisher    = {{IEEE}},
  year         = {2021},
  url          = {https://doi.org/10.1109/ICCV48922.2021.00126},
  doi          = {10.1109/ICCV48922.2021.00126},
  bibsource    = {dblp computer science bibliography, https://dblp.org}
}

@article{DBLP:journals/entropy/BertschingerROJA14,
  author       = {Nils Bertschinger and
                  Johannes Rauh and
                  Eckehard Olbrich and
                  J{\"{u}}rgen Jost and
                  Nihat Ay},
  title        = {Quantifying Unique Information},
  journal      = {Entropy},
  volume       = {16},
  number       = {4},
  pages        = {2161--2183},
  year         = {2014},
  url          = {https://doi.org/10.3390/e16042161},
  doi          = {10.3390/E16042161},
  bibsource    = {dblp computer science bibliography, https://dblp.org}
}

@inproceedings{DBLP:conf/eccv/ZeilerF14,
  author       = {Matthew D. Zeiler and
                  Rob Fergus},
  editor       = {David J. Fleet and
                  Tom{\'{a}}s Pajdla and
                  Bernt Schiele and
                  Tinne Tuytelaars},
  title        = {Visualizing and Understanding Convolutional Networks},
  booktitle    = {Computer Vision - {ECCV} 2014 - 13th European Conference, Zurich,
                  Switzerland, September 6-12, 2014, Proceedings, Part {I}},
  series       = {Lecture Notes in Computer Science},
  pages        = {818--833},
  publisher    = {Springer},
  year         = {2014},
  url          = {https://doi.org/10.1007/978-3-319-10590-1\_53},
  doi          = {10.1007/978-3-319-10590-1\_53},
  bibsource    = {dblp computer science bibliography, https://dblp.org}
}

@inproceedings{DBLP:conf/bmvc/PetsiukDS18,
  author       = {Vitali Petsiuk and
                  Abir Das and
                  Kate Saenko},
  title        = {{RISE:} Randomized Input Sampling for Explanation of Black-box Models},
  booktitle    = {British Machine Vision Conference 2018, {BMVC} 2018, Newcastle, UK,
                  September 3-6, 2018},
  pages        = {151},
  publisher    = {{BMVA} Press},
  year         = {2018},
  url          = {http://bmvc2018.org/contents/papers/1064.pdf},
  bibsource    = {dblp computer science bibliography, https://dblp.org}
}

@inproceedings{DBLP:conf/iccv/FongV17,
  author       = {Ruth C. Fong and
                  Andrea Vedaldi},
  title        = {Interpretable Explanations of Black Boxes by Meaningful Perturbation},
  booktitle    = {{IEEE} International Conference on Computer Vision, {ICCV} 2017, Venice,
                  Italy, October 22-29, 2017},
  pages        = {3449--3457},
  publisher    = {{IEEE} Computer Society},
  year         = {2017},
  url          = {https://doi.org/10.1109/ICCV.2017.371},
  doi          = {10.1109/ICCV.2017.371},
  bibsource    = {dblp computer science bibliography, https://dblp.org}
}

@article{SOBOL2001271,
title = {Global sensitivity indices for nonlinear mathematical models and their Monte Carlo estimates},
journal = {Mathematics and Computers in Simulation},
volume = {55},
number = {1},
pages = {271-280},
year = {2001},
note = {The Second IMACS Seminar on Monte Carlo Methods},
issn = {0378-4754},
doi = {https://doi.org/10.1016/S0378-4754(00)00270-6},
author = {I.M Sobol'}
}

@inproceedings{
marzouk2026shap,
title={{SHAP} Meets Tensor Networks: Provably Tractable Explanations with Parallelism},
author={Reda Marzouk and Shahaf Bassan and Guy Katz},
booktitle={The Thirty-ninth Annual Conference on Neural Information Processing Systems},
year={2026},
url={https://openreview.net/forum?id=FfccSikDfZ}
}

@inproceedings{DBLP:conf/icml/BuiNNY24,
  author       = {Ngoc Bui and
                  Hieu Trung Nguyen and
                  Viet Anh Nguyen and
                  Rex Ying},
  editor       = {Ruslan Salakhutdinov and
                  Zico Kolter and
                  Katherine A. Heller and
                  Adrian Weller and
                  Nuria Oliver and
                  Jonathan Scarlett and
                  Felix Berkenkamp},
  title        = {Explaining Graph Neural Networks via Structure-aware Interaction Index},
  booktitle    = {Forty-first International Conference on Machine Learning, {ICML} 2024,
                  Vienna, Austria, July 21-27, 2024},
  series       = {Proceedings of Machine Learning Research},
  pages        = {4854--4883},
  publisher    = {{PMLR} / OpenReview.net},
  year         = {2024},
  url          = {https://proceedings.mlr.press/v235/bui24b.html},
  bibsource    = {dblp computer science bibliography, https://dblp.org}
}

@article{Hoeffding1948,
 ISSN = {00034851},
 URL = {http://www.jstor.org/stable/2235637},
 author = {Wassily Hoeffding},
 journal = {The Annals of Mathematical Statistics},
 number = {3},
 pages = {293--325},
 publisher = {Institute of Mathematical Statistics},
 title = {A Class of Statistics with Asymptotically Normal Distribution},
 urldate = {2026-05-02},
 volume = {19},
 year = {1948}
}

@article{Lei03072018,
author = {Jing Lei and Max G’Sell and Alessandro Rinaldo and Ryan J. Tibshirani and Larry Wasserman},
title = {Distribution-Free Predictive Inference for Regression},
journal = {Journal of the American Statistical Association},
volume = {113},
number = {523},
pages = {1094--1111},
year = {2018},
publisher = {Taylor \& Francis},
doi = {10.1080/01621459.2017.1307116},
URL = {https://doi.org/10.1080/01621459.2017.1307116}
}

@inproceedings{DBLP:conf/nips/CovertLL20,
  author       = {Ian Covert and
                  Scott M. Lundberg and
                  Su{-}In Lee},
  editor       = {Hugo Larochelle and
                  Marc'Aurelio Ranzato and
                  Raia Hadsell and
                  Maria{-}Florina Balcan and
                  Hsuan{-}Tien Lin},
  title        = {Understanding Global Feature Contributions With Additive Importance
                  Measures},
  booktitle    = {Advances in Neural Information Processing Systems 33: Annual Conference
                  on Neural Information Processing Systems 2020, NeurIPS 2020, December
                  6-12, 2020, virtual},
  year         = {2020},
  url          = {https://proceedings.neurips.cc/paper/2020/hash/c7bf0b7c1a86d5eb3be2c722cf2cf746-Abstract.html},
  bibsource    = {dblp computer science bibliography, https://dblp.org}
}

@article{DBLP:journals/natmi/LundbergECDPNKH20,
  author       = {Scott M. Lundberg and
                  Gabriel G. Erion and
                  Hugh Chen and
                  Alex J. DeGrave and
                  Jordan M. Prutkin and
                  Bala Nair and
                  Ronit Katz and
                  Jonathan Himmelfarb and
                  Nisha Bansal and
                  Su{-}In Lee},
  title        = {From local explanations to global understanding with explainable {AI}
                  for trees},
  journal      = {Nat. Mach. Intell.},
  volume       = {2},
  number       = {1},
  pages        = {56--67},
  year         = {2020},
  url          = {https://doi.org/10.1038/s42256-019-0138-9},
  doi          = {10.1038/S42256-019-0138-9},
  bibsource    = {dblp computer science bibliography, https://dblp.org}
}

@inproceedings{DBLP:conf/cvpr/HuangLMW17,
  author       = {Gao Huang and
                  Zhuang Liu and
                  Laurens van der Maaten and
                  Kilian Q. Weinberger},
  title        = {Densely Connected Convolutional Networks},
  booktitle    = {2017 {IEEE} Conference on Computer Vision and Pattern Recognition,
                  {CVPR} 2017, Honolulu, HI, USA, July 21-26, 2017},
  pages        = {2261--2269},
  publisher    = {{IEEE} Computer Society},
  year         = {2017},
  url          = {https://doi.org/10.1109/CVPR.2017.243},
  doi          = {10.1109/CVPR.2017.243},
  bibsource    = {dblp computer science bibliography, https://dblp.org}
}
